%% file: main.tex
\documentclass{article}
\usepackage{silence}
\usepackage[margin=1in]{geometry}
\usepackage{mathpazo}

\usepackage{enumitem}
\usepackage[utf8]{inputenc}
\usepackage[T1]{fontenc}
\usepackage{amsmath,amssymb,amsthm}
\usepackage{mathtools,thmtools}

\let\counterwithin\relax
\usepackage{chngcntr}
\usepackage{bbm}
\usepackage{hyperref}
\usepackage{cleveref}
\usepackage{graphicx}
\usepackage{color}
\usepackage{booktabs}
\usepackage{multirow}
\usepackage{subcaption}
\usepackage{makecell}
\usepackage{mdframed}
\usepackage{adjustbox}
\usepackage{mleftright}
\usepackage{bm}
\usepackage{amssymb}
\usepackage{algorithm}
\usepackage{caption}
\usepackage[noend]{algpseudocode}

\usepackage[toc,page,header]{appendix}
\usepackage{minitoc}

\usepackage{caption}
\usepackage{comment}
\usepackage[backend=biber,
            style=alphabetic,
            natbib=true,
            maxbibnames=99,
            minalphanames=3,
            maxcitenames=2,
            sorting=ynt,
            sortcites=true]{biblatex}
\usepackage{tikz}
\usetikzlibrary{decorations.pathreplacing,calligraphy,shapes,calc,graphs,shapes.misc,math}

\usepackage{listings}
\definecolor{mydarkblue}{rgb}{0,0.08,0.85}
\definecolor{mylightblue}{rgb}{0.06,0.56,1.0}
\definecolor{mylightorange}{rgb}{1.0,0.62,0.12}
\definecolor{mylightred}{rgb}{0.99,0.00,0.04}
\hypersetup{colorlinks=true,
linkcolor=mydarkblue,
citecolor=mydarkblue,
filecolor=mydarkblue,
urlcolor=mydarkblue,
bookmarksopen=true,
linktoc=page}

\newcommand\supp{\textsc{Support}}

\newcommand{\lr}[1]{\left({#1}\right)}

\makeatletter
\newcommand*{\E}{%
  \def\E@sub{}%
  \def\E@sup{}%
  \E@scripts
}
\newcommand*{\E@scripts}{%
  \@ifnextchar_\E@subscript{%
    \@ifnextchar^\E@supscript\E@finish
  }%
}
\def\E@subscript_#1{%
  \ifx\E@sub\@empty
    \def\E@sub{#1}%
  \else
    \errmessage{E already has a subscript}%
  \fi
  \E@scripts
}
\def\E@supscript^#1{%
  \ifx\E@sup\@empty
    \def\E@sup{#1}%
  \else
    \errmessage{E already has a superscript}%
  \fi
  \E@scripts
}
\newcommand*{\E@finish}[1]{%
  \mathbb{E}%
  \ifx\E@sub\@empty\else _{\E@sub}\fi
  \ifx\E@sup\@empty\else ^{\E@sup}\fi
  \mleft(#1\mright)%
}
\makeatother

\makeatletter
\renewcommand*{\P}{%
  \def\P@sub{}%
  \def\P@sup{}%
  \P@scripts
}
\newcommand*{\P@scripts}{%
  \@ifnextchar_\P@subscript{%
    \@ifnextchar^\P@supscript\P@finish
  }%
}
\def\P@subscript_#1{%
  \ifx\P@sub\@empty
    \def\P@sub{#1}%
  \else
    \errmessage{P already has a subscript}%
  \fi
  \P@scripts
}
\def\P@supscript^#1{%
  \ifx\P@sup\@empty
    \def\P@sup{#1}%
  \else
    \errmessage{P already has a superscript}%
  \fi
  \P@scripts
}
\newcommand*{\P@finish}[1]{%
  \mathbb{P}%
  \ifx\P@sub\@empty\else _{\P@sub}\fi
  \ifx\P@sup\@empty\else ^{\P@sup}\fi
  \mleft(#1\mright)%
}
\makeatother

\renewcommand{\hat}[1]{\widehat{#1}}
\newcommand{\ind}[1]{\bm{1}\{{#1}\}}
\newcommand{\modeleval}[2]{f_{\mathcal{A}}({#1};{#2})}

\newmdtheoremenv{defn}{Definition}
\newmdtheoremenv{thm}{Theorem}
\newmdtheoremenv{rem}{Remark}
\newmdtheoremenv{appmode}{Use Case}

\newtheorem{remark}{Remark}

\title{Datamodels: Predicting Predictions from Training Data}
\newcommand\AND{
    \end{tabular}\hfil\linebreak[4]\hfil%
    \begin{tabular}[t]{c}\ignorespaces%
}
\author{
    Andrew Ilyas\footnote{Equal contribution.} \\
    \texttt{ailyas@mit.edu} \\
    MIT
    \and
    Sung Min Park\footnotemark[1] \\
    \texttt{sp765@mit.edu} \\
    MIT
    \and
    Logan Engstrom\footnotemark[1] \\
    \texttt{engstrom@mit.edu} \\
    MIT
    \AND
    Guillaume Leclerc \\
    \texttt{leclerc@mit.edu} \\
    MIT
    \and
    Aleksander M\k{a}dry \\
    \texttt{madry@mit.edu} \\
    MIT
}
\date{}

\usepackage{xspace}

\newcommand{\fmow}{FMoW}

\begin{document}
\setcounter{tocdepth}{2}
\doparttoc %
\renewcommand\ptctitle{}
\faketableofcontents %

\maketitle

\begin{abstract}
  \input{sections/abstract}

\end{abstract}

\newcommand{\mask}[1]{\bm{1}_{#1}}
\newcommand{\parmap}[1]{g_{\theta}(#1)}

\section{Introduction}
\input{sections/intro}

\section{Constructing (linear) datamodels}
\label{sec:constructing}
\input{sections/instantiating_datamodels}

\section{Accurately predicting outputs with datamodels}
\label{sec:specializing}
\input{sections/specializing}

\section{Leveraging datamodels}
\label{sec:applications}
\input{sections/leveraging}

\section{Discussion: The role of the subsampling fraction \texorpdfstring{$\alpha$}{alpha}}
\label{sec:alpha_choice}
\input{sections/role_of_alpha}

\section{Related work}
\label{sec:related}
\input{sections/related}

\section{Future work}
\label{sec:future_work}
\input{sections/future_work}

\section{Conclusion}
We present datamodeling, a framework for viewing the output of
model training as a simple function of the presence of each training data point.
We show that a simple linear instantiation of datamodeling enables us to predict
model outputs accurately, and facilitates a variety of applications.

\section*{Acknowledgements}
We thank Chiyuan Zhang and Vitaly Feldman for providing a set of 5,000 models
with which we began our investigation. We also thank Hadi Salman for valuable
discussions.

Work supported in part by the NSF grants CCF-1553428 and CNS-1815221, and Open
Philanthropy. This material is based upon work supported by the Defense Advanced
Research Projects Agency (DARPA) under Contract No. HR001120C0015.

\clearpage
\begin{refcontext}[sorting=nyt]
\printbibliography
\end{refcontext}

\clearpage
\appendix
\addcontentsline{toc}{section}{Appendix} %
\renewcommand\ptctitle{Appendices}
\part{}
\parttoc
\clearpage

\counterwithin{figure}{section}
\counterwithin{table}{section}
\counterwithin{algorithm}{section}

\input{appendices/datamodel_pseudocode}
\clearpage
\input{appendices/numerical_sim}
\clearpage
\input{appendices/choosing_margins}

\clearpage
\input{appendices/experimental_setup}

\clearpage
\input{appendices/regression}
\clearpage
\input{appendices/causal_view_extras}

\clearpage
\input{appendices/nearest}
\clearpage
\input{appendices/ttl}

\clearpage

\input{appendices/spectral_app}
\clearpage
\input{appendices/pca_app}
\clearpage
\input{appendices/influence_proof}
\clearpage
\input{appendices/taylor}

\end{document}

%% file: sections/abstract.tex
We present a conceptual framework, \emph{datamodeling}, for analyzing the
behavior of a model class in terms of the training data.
For any fixed ``target'' example $x$, training set $S$, and
learning algorithm, a {\em datamodel} is a parameterized function
$2^S \to \mathbb{R}$ that for any subset of $S' \subset S$---using only
information about which examples of $S$ are contained in $S'$---predicts the
outcome of training a model on $S'$ and evaluating on $x$.
Despite the potential complexity of the underlying process being approximated (e.g.,
end-to-end training and evaluation of deep neural networks), we
show that even simple {\em linear} datamodels can successfully predict model outputs.
We then demonstrate that datamodels give rise to a variety of applications, such as:
accurately predicting the effect of dataset counterfactuals; identifying
brittle predictions; finding semantically similar examples; quantifying
train-test leakage; and embedding data into a well-behaved and feature-rich
{\em representation space}.\footnote{Data for this paper (including pre-computed
datamodels as well as raw predictions from four million trained deep neural
networks) is available at \url{https://github.com/MadryLab/datamodels-data}.}

%% file: sections/intro.tex
\label{sec:intro}
\begin{center}
{\em What kinds of biases does my (machine learning) system exhibit?
What correlations does it exploit?
On what subpopulations does it perform well (or poorly)?}
\end{center}
A recent body of work in machine learning suggests that the answers to these
questions lie within both the learning algorithm and the training data used
\citep{carlini2019secret,gu2017badnets,ilyas2019adversarial, hooker2021moving,
jain2021co}. 
However, it is often difficult to understand {\em how} algorithms and data
combine to yield model predictions. In this work, we present {\em
datamodeling}---a framework for tackling this question by forming an
explicit model for predictions in terms of the training data.

\newcommand{\trainex}{\ensuremath{z}}
\paragraph{Setting.} Consider a typical machine learning setup, starting
with a training set $S$ comprising $d$ input-label pairs. 
The focal point of this setup is
a {\em learning algorithm} $\mathcal{A}$ that takes in such a training
set of input-label pairs, and outputs a trained model. (Note that this learning
algorithm does not have to be deterministic---for example, $\mathcal{A}$ might encode
the process of training a deep neural network from random initialization using
stochastic gradient descent.)

Now, consider a {\em fixed} input $x$
(e.g., a photo from the test set of a computer vision benchmark) and define
\begin{align}
    \label{eq:modeleval}
    \modeleval{x}{S} \coloneqq \text{the outcome of training a model on $S$ using $\mathcal{A}$ and evaluating it on the input $x$},
\end{align}
where we leave ``outcome'' intentionally broad to capture a variety of use cases.
For example, $\modeleval{x}{S}$ may be the
cross-entropy loss of a classifier on $x$, or the squared-error of a regression
model on $x$. The potentially stochastic nature of $\mathcal{A}$
means that $\modeleval{x}{S}$ is a random variable.

\paragraph{Goal.} Broadly, we aim to understand how the training examples
in $S$ combine through the learning algorithm $\mathcal{A}$ to yield
$\modeleval{x}{S}$ (again, for the {\em specifically} chosen input $x$).
To this end, we leverage a classic technique for
studying black-box functions: {\em surrogate modeling}
\citep{sacks1989design}. In surrogate modeling, one replaces a complex black-box
function with an inexact but significantly easier-to-analyze approximation,
then uses the latter to shed light on the behavior of the original function.

In our context, the complex black-box function is $\modeleval{x}{\cdot}$.
We thus aim to find a simple {\em surrogate} function $g(S')$ whose output
roughly matches $\modeleval{x}{S'}$ for a variety of training sets $S'$
(but again, for a {\em fixed} input $x$). Achieving this goal would reduce the
challenge of scrutinizing $\modeleval{x}{\cdot}$---and more generally, the map
from training data to predictions through learning algorithm
$\mathcal{A}$---to the (hopefully easier) task of analyzing $g$.

\paragraph{Datamodeling.}  By parameterizing the surrogate function $g$ (e.g.,
as $g_\theta$, for a parameter vector $\theta$), we transform the
challenge of constructing a surrogate function into a {\em supervised learning}
problem.
In this problem, the ``training examples'' are subsets $S' \subset S$ of the
original task's training set $S$, and the
corresponding ``labels'' are given by $\modeleval{x}{S'}$ (which we can compute
by simply training a new model on $S'$ with algorithm $\mathcal{A}$, and
evaluating on $x$).
Our goal is then to fit a parametric function $g_\theta$ mapping the former to
the latter.

We now formalize this idea as {\em datamodeling}---a framework that forms the
basis of our work. In this framework, we
first fix a distribution over subsets that we will use to collect
``training data'' for $g_\theta$,
\begin{align}
    \mathcal{D}_S \coloneqq
    \text{a fixed distribution over subsets of $S$ (i.e., }
    \text{support}(\mathcal{D}_S) \subseteq 2^S \text{)},
\end{align}
and then use $\mathcal{D}_S$ to collect a {\em datamodel training set}, or a collection of pairs
$$\left\{(S_1,\, \modeleval{x}{S_1}), \ldots, (S_m,\, \modeleval{x}{S_m})\right\},$$
where $S_i \sim \mathcal{D}_S$, and again $\modeleval{x}{S_i}$ is the result of
training a model on $S_i$ and evaluating on $x$ (cf. \eqref{eq:modeleval}).

We next focus on how to parameterize our surrogate function $g_\theta$.
In theory, $g_\theta$ can be any map that takes as input
subsets of the training set, and returns estimates of $\modeleval{x}{\cdot}$.
However, to simplify $g_\theta$ we ignore the actual {\em
contents} of the subsets $S_i$, and instead focus solely on the
{\em presence} of each training example of $S$ within $S_i$.
In particular, we consider the {\em characteristic vector} corresponding to each $S_i$,
\begin{align}
    \label{eq:char_vec}
    \mask{S_i} \in \{0,1\}^d
        \qquad
        \text{such that}
        \qquad
    (\mask{S_i})_j = \begin{cases}
        1 &\text{if } z_j \in S_i \\
        0 &\text{otherwise,}
    \end{cases}
\end{align}
a vector that indicates which elements of the original training set $S$
belong to a given subset $S_i$. We then define
a \underline{\em datamodel} for a given input $x$ as a function
\vspace*{-0.75em}
\begin{align}
    \label{eq:datamodeling_redef}
    g_\theta: \{0,1\}^d \to \mathbb{R}, \qquad \text{where}
    \qquad
    \theta =
    \arg\min_{w}\ \frac{1}{m}\sum_{i=1}^m \mathcal{L} \lr{
        g_w(\mask{S_i}),\ \modeleval{x}{S_i}
    },
\end{align}
and $\mathcal{L}(\cdot, \cdot)$ is a fixed loss function (e.g., squared-error).
This setup \eqref{eq:datamodeling_redef} places datamodels squarely
within the realm of supervised learning: e.g., we can easily test a
given datamodel by sampling new subset-output pairs $\{(S_i,
\modeleval{x}{S_i})\}$ and computing average loss.
For completeness, we restate the entire datamodeling
framework below: 
\begin{defn}[Datamodeling]
    \label{def:datamodeling}
    \vspace{-0.7em}
    Consider a fixed training set $S$, a learning algorithm $\mathcal{A}$,
    a target example $x$, and a distribution $\mathcal{D}_S$ over subsets of $S$.
    For any set $S' \subset S$, let $\modeleval{x}{S'}$ be
    the (stochastic) output of training a model on
    $S'$ using $\mathcal{A}$, and evaluating on $x$.
    A \underline{datamodel} for $x$ is a parametric function $g_\theta$ optimized to
    predict $\modeleval{x}{S_i}$ from training subsets $S_i \sim \mathcal{D}_S$,
    i.e.,
    \begin{align*}
        g_\theta: \{0,1\}^{|S|} \to \mathbb{R}, \qquad\text{ where }\qquad
        \theta =
        \arg\min_{w}\ \widehat{\mathbb{E}}^{(m)}_{S_i \sim \mathcal{D}_S}
        \left[ \mathcal{L} \lr{
            g_w(\mask{S_i}),\ \modeleval{x}{S_i}
        } \right],
    \end{align*}
    $\mask{S_i} \in \{0,1\}^{|S|}$ is the characteristic vector of
    $S_i$ in $S$ (see \eqref{eq:char_vec}),
    $\mathcal{L}(\cdot, \cdot)$ is a loss function, and
    $\widehat{\mathbb{E}}^{(m)}$ is an $m$-sample empirical estimate of the expectation.
\end{defn}
The pseudocode for computing datamodels is in Appendix \ref{app:pseudocode}.
Before proceeding further, we highlight two critical (yet somewhat
subtle) properties of the datamodeling framework:
\begin{itemize}
    \item {\bf Datamodeling studies model {\em classes}, not specific models}:
    Datamodeling focuses on the entire distribution of models
    induced by the algorithm $\mathcal{A}$, rather than a specific model. Recent
    work suggests this distinction is particularly significant for modern learning
    algorithms (e.g., neural networks), where models can exhibit drastically
    different behavior depending on only the choice of random seed during
    training \citep{nakkiran2020distributional, jiang2021assessing,
    damour2020underspecification, zhong2021larger}---we
    discuss this further in Section \ref{sec:related}.
    \item {\bf Datamodels are target example-specific}: A datamodel $g_\theta$
    predicts model outputs on a specific but arbitrary
    target example $x$. This $x$ might be an example from the test set,
    a synthetically generated example, or even (as we will see in
    Section \ref{subsec:impl_details}) an example from the training set
    $S$ itself.
    We will often work with
    {\em collections} of datamodels corresponding to a set of target
    examples (e.g., we might consider a test set $\{x_1,\ldots x_n\}$ with
    corresponding datamodels $\{g_{\theta_1},\ldots g_{\theta_n}\}$). In Section
    \ref{sec:specializing} we show that as long
    as the learning algorithm $\mathcal{A}$ and the training set $S$ are fixed,
    computing a collection of datamodels simultaneously is not much harder than
    computing a single one.
\end{itemize}

\subsection{Roadmap and contributions}
The key contribution of our work is the {\em datamodeling framework}
described above, which allows us to analyze the behavior of a machine learning
algorithm $\mathcal{A}$ in terms of the training data. In the remainder of this
work, we show how to instantiate, implement, and apply this framework.

We begin in Section \ref{sec:constructing} by considering a
concrete instantiation of datamodeling in which the map $g_\theta$
is a {\em linear} function. Then, in Section \ref{sec:specializing} we develop the
remaining machinery required to apply this instantiation to deep neural networks
trained on standard image datasets. In the rest of the paper, we find
that:
\begin{itemize}
    \item {\bf Datamodels successfully predict model outputs
    (\S\,\ref{sec:estimation_results}, Figure \ref{fig:ondist_headline})}:
    despite their simplicity,
    datamodels yield predictions that match expected model outputs
    on new sets $S$ drawn from the same distribution $\mathcal{D}_S$.
    (For example, the Pearson correlation between predicted and ground-truth
    outputs is $r > 0.99$.)
    \item {\bf Datamodels successfully predict counterfactuals
    (\S\,\ref{sec:causal_view}, Figure \ref{fig:causal_headline})}: predictions correlate with model outputs
    even on out-of-distribution training subsets (Figures
    \ref{fig:cifar_causal}, \ref{fig:pca_causality_scatter} and Appendix
    \ref{app:causal_details}) allowing us to estimate the
    {\em causal effect} of removing training images on a given test prediction.
    Leveraging this ability, we find that {\em for
    50\% of CIFAR-10 \citep{krizhevsky2009learning} test images, models can be made incorrect by
    removing less than 200 target-specific training
    points (i.e., 0.4\% of the total training set size).} If one
    mislabels the training examples instead of only removing them, {\em 35
    label-specific points} suffice.
    \begin{figure}[!h]
        \centering
        \begin{minipage}{0.3\textwidth}
            \includegraphics[width=\textwidth,trim={0 0.0cm 0 -0.20cm},clip]{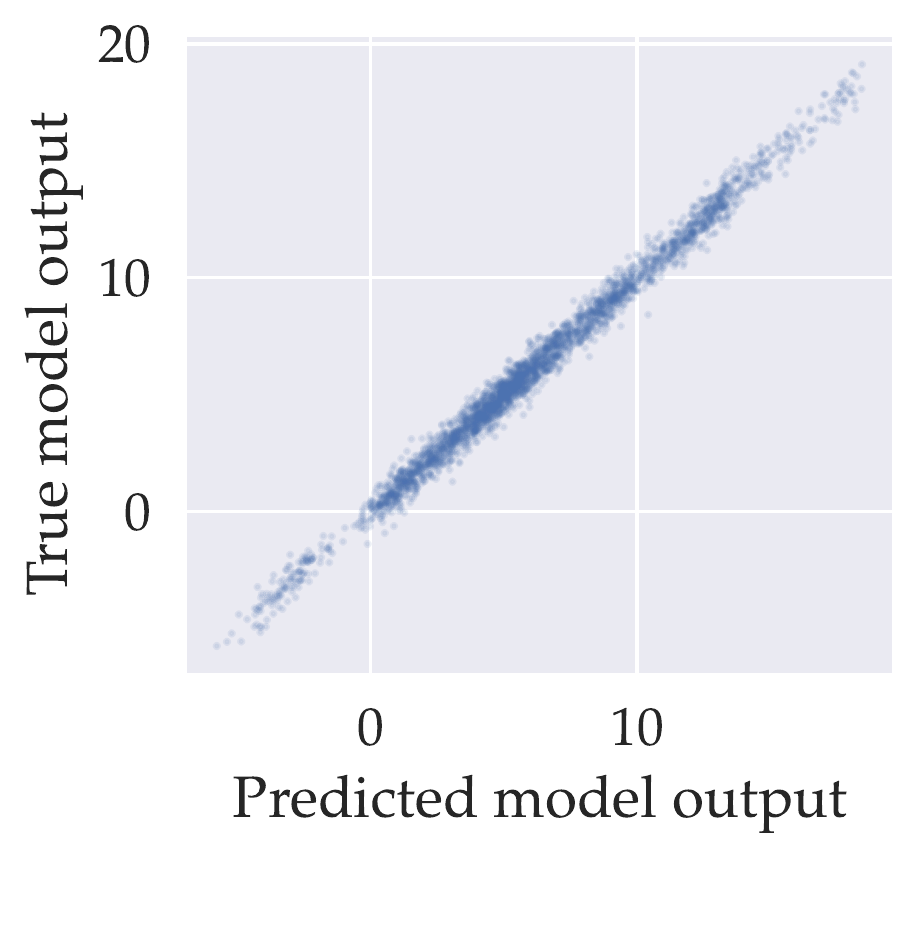}
            \caption{Datamodels predict
            ($g_\theta(S')$, x-axis)
            the outcome of training models on
            subsets $S'$ of the training set $S$ sampled from $\mathcal{D}_S$ and evaluating on $x$
            ($\mathbb{E}[\modeleval{x}{S'}]$, y-axis)}
            \label{fig:ondist_headline}
        \end{minipage}
        \hspace{1em}
        \begin{minipage}{0.65\textwidth}
            \includegraphics[width=\textwidth,trim={0 0 0 0cm},clip]{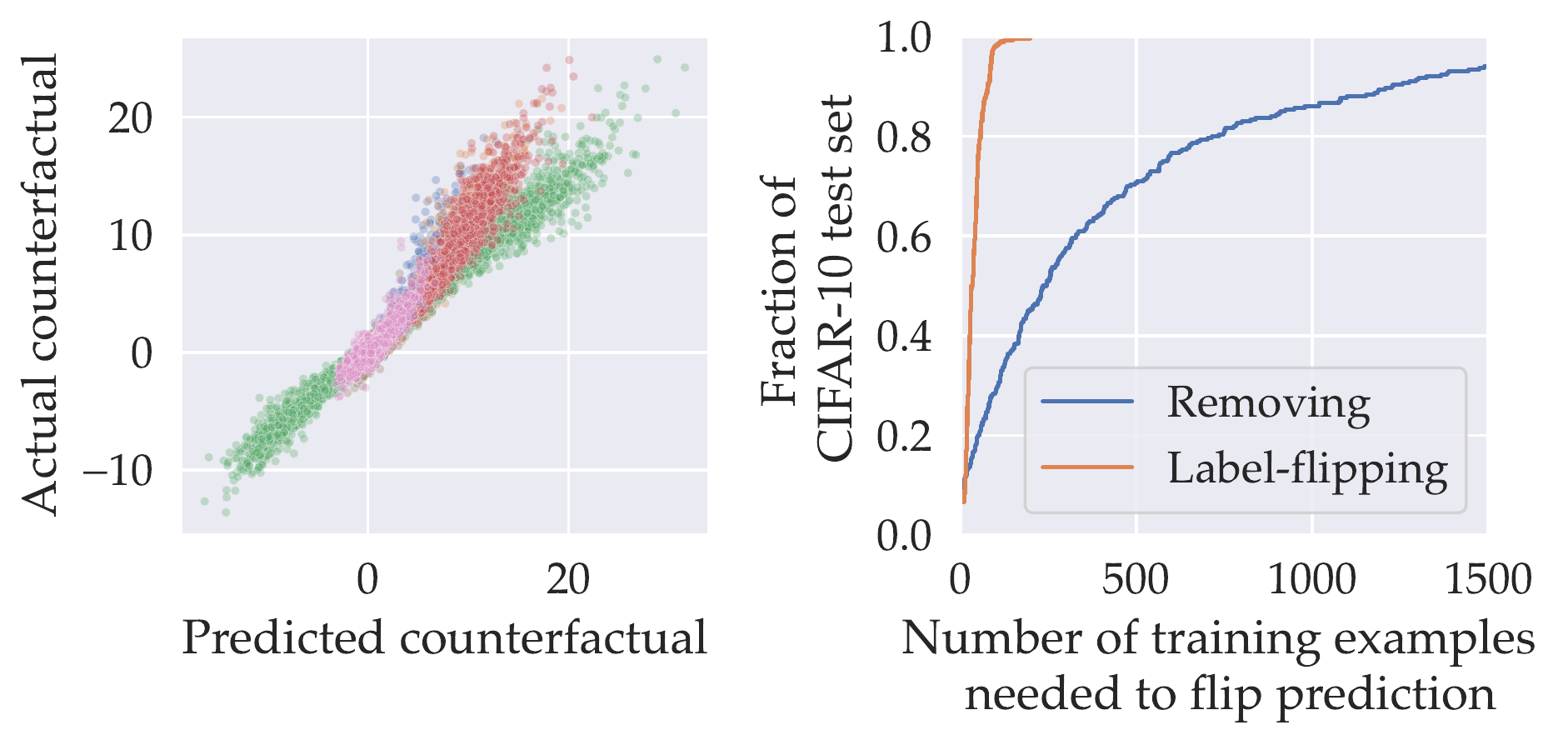}
            \caption{Datamodels predict out-of-distribution dataset
            counterfactuals (left) and identify brittle predictions (right). As
            seen on the right, approximately 50\% of predictions on the CIFAR-10 test
            set can be flipped by removing less than $200$ (target-specific)
            training images. If we flip the labels of chosen training images instead of removing them, just 35 images suffices.}
            \label{fig:causal_headline}
        \end{minipage}
    \end{figure}
    \item {\bf Datamodel weights encode similarity
    (\S\,\ref{sec:distance_metric}, Figure \ref{fig:headline_ttl})}:
    the most positive (resp., negative) datamodel weights tend to
    correspond to similar training images from the same
    (resp., different) class as the target example $x$.
    We use this property to identify significant train-test
    leakage across both datasets we study (CIFAR-10 and Functional Map of the
    World \citep{koh2020wilds,christie2018functional}).
    \begin{figure}[!h]
        \begin{subfigure}{\textwidth}
            \centering
            \includegraphics[width=\textwidth,trim={0 2.5cm 0 0.35cm},clip]{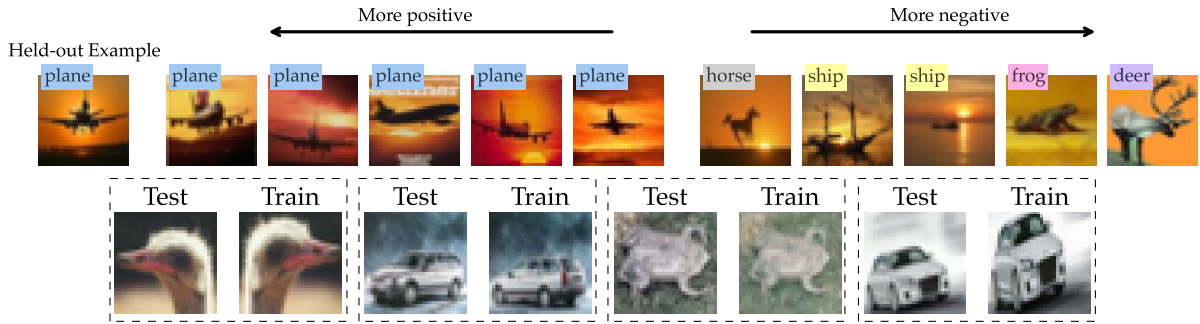}
        \end{subfigure}
        \begin{subfigure}{\textwidth}
            \centering
            \input{paper_figs/ttl_headline_tikz}
        \end{subfigure}
        \caption{High-magnitude datamodel weights identify semantically similar
        training examples (top), which we can use to find train-test leakage in
        benchmark computer vision datasets (bottom).}
        \label{fig:headline_ttl}
    \end{figure}
    \item {\bf Datamodels yield a well-behaved embedding
    (\S\,\ref{sec:rep_view}, Figure \ref{fig:headline_pca})}:
    viewing datamodel weights as a {\em feature embedding} of each image into 
    $\mathbb{R}^d$ (where $d$ is the training set size), we discover a
    well-behaved representation space. In particular, we find that
    so-called  {\em datamodel embeddings}:
    \begin{itemize}
        \item[(a)] enable (qualitatively) high-quality clustering;
        \item[(b)] allow us to identify model-relevant {\em subpopulations} that
        we can causally verify in a natural sense;
        \item[(c)] have a number of advantages over representations derived from, e.g.,
        the penultimate layer of a fixed pre-trained network, such as higher
        effective dimensionality and ({\em a priori}) human-meaningful coordinates.
    \end{itemize} 
    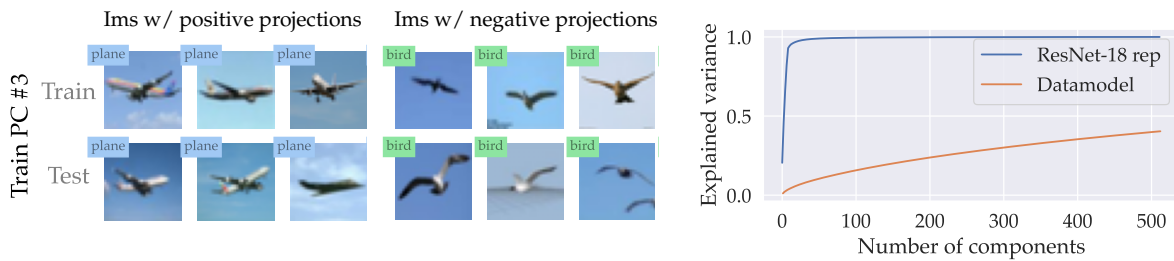
\begin{figure}[h]
        \input{paper_figs/pca_headline_tikz}
        \caption{Datamodels yield a well-behaved (left), relatively dense
        (right) embedding of any given input into $\mathbb{R}^d$, where $d$ is
        the training set size. Applying natural manipulations to these
        embeddings enables a variety of applications which we discuss more
        thoroughly in \S\,\ref{sec:rep_view}.}
        \label{fig:headline_pca}
    \end{figure}
\end{itemize}

More broadly, datamodels turn out to be a versatile tool for understanding how
learning algorithms leverage their training data. In Section \ref{sec:related},
we contextualize datamodeling with respect to several ongoing lines of work in
machine learning and statistics. We conclude, in Section \ref{sec:future_work},
by outlining a variety of directions for future work on both improving and
applying datamodels.

%% file: paper_figs/ttl_headline_tikz.tex
\begin{tikzpicture}
    \node [fill=white] (label) at (0, 0) 
        {\includegraphics[width=0.8\textwidth]{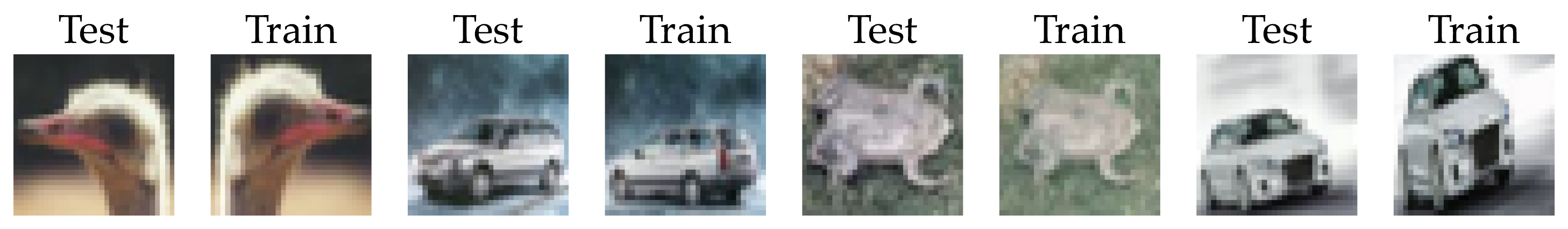}};

    \draw[draw=black,dashed] (-6.55, -0.95) rectangle ++(3.15,1.9);
    \draw[draw=black,dashed] (-3.24, -0.95) rectangle ++(3.15,1.9);
    \draw[draw=black,dashed] (0.07, -0.95) rectangle ++(3.15,1.9);
    \draw[draw=black,dashed] (3.4, -0.95) rectangle ++(3.15,1.9);
\end{tikzpicture}

%% file: paper_figs/pca_headline_tikz.tex
\begin{tikzpicture}
    \node [fill=white] (label) at (0, 0) 
        {\includegraphics[width=0.3\textwidth,trim={0 5cm 20.175cm 12.2cm},clip]{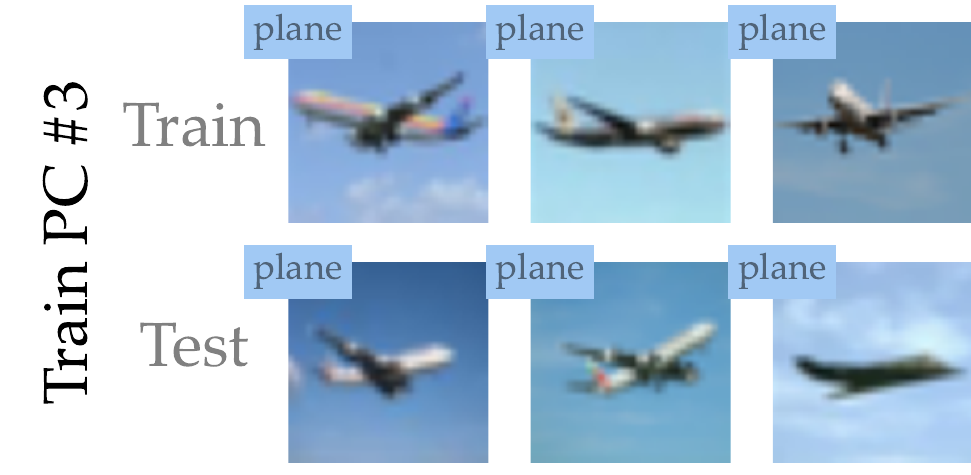}};
    \node [fill=white] (label) at (4.5, 0) 
        {\includegraphics[width=0.225\textwidth,trim={17.25cm 5cm 5.3cm 12.2cm},clip]{paper_figs/pca_new_train.pdf}};
    \node [fill=white] (label) at (10.1, -0.1) 
        {\includegraphics[width=0.4\textwidth]{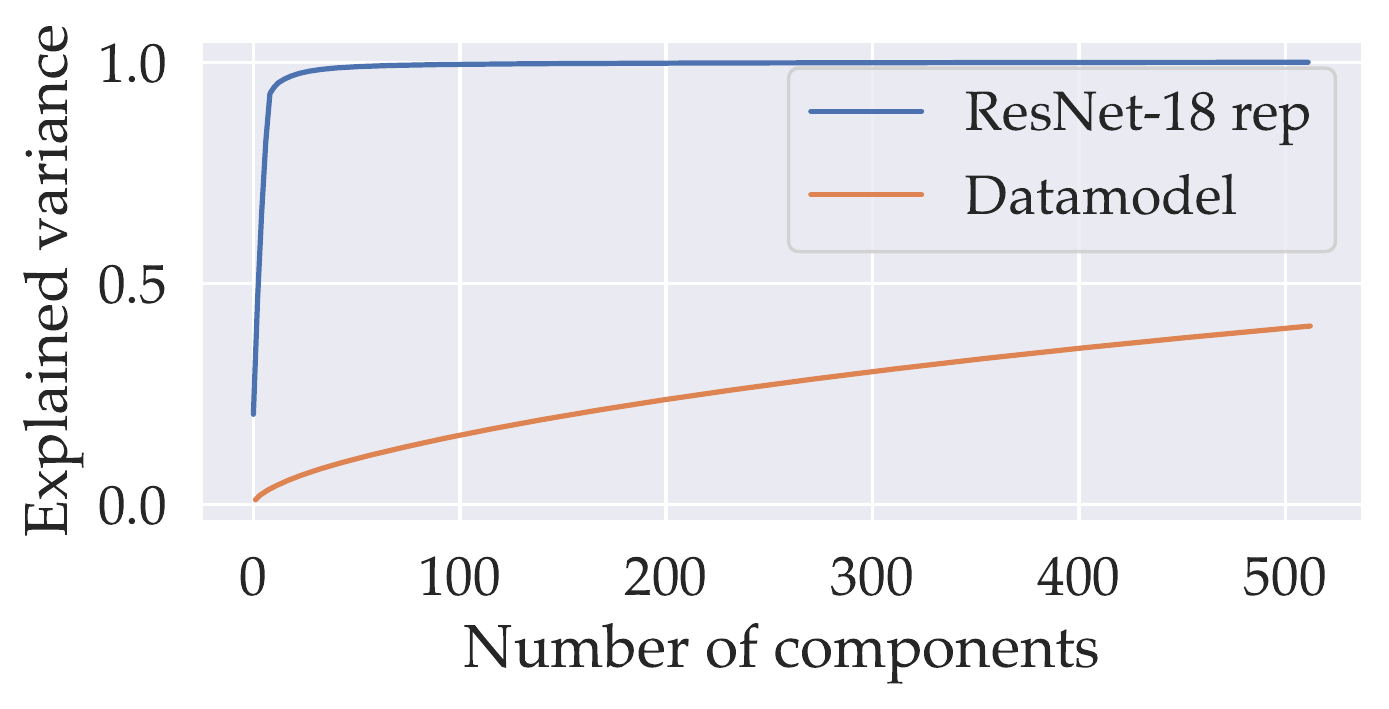}};

    \node[fill=white,align=center] at (0.7, 1.5) {\footnotesize Ims w/ positive projections};
    \node[fill=white,align=center] at (4.6, 1.5) {\footnotesize Ims w/ negative projections};
\end{tikzpicture}

%% file: sections/instantiating_datamodels.tex
\newcommand{\nummodels}{m}
\newcommand{\numtest}{n}
\newcommand{\numtrain}{d}

As described in Section \ref{sec:intro}, building datamodels comprises
the following steps:
\begin{itemize}
    \item[(a)] pick a parameterized class of functions $g_\theta$;
    \item[(b)] sample a collection of subsets $S_i \subset S$
    from a fixed training set according to a distribution $\mathcal{D}_S$;
    \item[(c)] for each subset $S_i$, train a model using algorithm $\mathcal{A}$,
    evaluate the model on target input $x$ using the relevant metric (e.g.,
    loss); collect the resulting pair $(\mask{S_i}, \modeleval{x}{S_i})$;
    \item[(d)] split the collected dataset of subset-output pairs into a
    datamodel training set of size $m$, a datamodel validation set of size
    $m_{val}$, and a datamodel test set of size $m_{test}$;
    \item[(e)] estimate parameters $\theta$ by fitting $g_\theta$ on subset-output pairs, i.e.,
    by minimizing
    \[
        \frac{1}{m} \sum_{i=1}^m \mathcal{L}\lr{
            \parmap{\mask{S_i}},\ \modeleval{x}{S_i}
        }
    \]
    over the collected datamodel training set, and use the validation set to perform
    model selection.
\end{itemize}

We now explicitly instantiate this framework, with the goal of understanding the
predictions of (deep) {\em classification} models.
To this end, we revisit steps
(a)-(e) above, and consider each relevant aspect---the sampling distribution
$\mathcal{D}_S$, the
output function $\modeleval{x}{S}$, the parameterized family $g_\theta$, and the
loss function $\mathcal{L}(\cdot, \cdot)$---separately:

\paragraph{(a) What surrogate function $g_\theta$ should we use?}
The first design choice to make is which family of parameterized surrogate
functions $g_\theta$ to optimize over. At first, one might be inclined
to use a complex family of functions in the hope of reducing potential misspecification
error. After all, $g_\theta$ is meant to be a surrogate for the end-to-end
training of a deep classifier. In this work, however, we will instantiate
datamodeling by taking $\parmap{\cdot}$ to be a simple {\em linear} mapping
\begin{align}
    \parmap{\mask{S_i}} \coloneqq \theta^\top \mask{S_i} + \theta_0,
\end{align}
where we recall that $\mask{S_i}$ is the size-$d$ {\em characteristic
vector} of $S_i$ within $S$ (cf. \eqref{eq:char_vec}).
\begin{remark}
    While we will allow $\parmap{\cdot}$ to fit a \underline{bias}
    term as above, for notational convenience we omit $\theta_0$ throughout this
    work and will simply write $\theta^\top \mask{S_i}$ to represent a datamodel
    prediction for the set $S_i$.
\end{remark}

\paragraph{(b) What distribution $\mathcal{D}_S$ over training subsets do we use?}
In step (a) of the estimation process above, we collect a ``datamodel training
set'' by sampling subsets $S_i \subset S$ from a distribution $\mathcal{D}_S$.
A simple first choice for $\mathcal{D}_S$---and indeed, the one we consider
for the remainder of this work---is the distribution of random $\alpha$-fraction
subsets of the training set. Formally, we set
\begin{align}
    \label{eq:distribution}
    \mathcal{D}_S = \text{Uniform}\lr{
        \{S' \subset S: |S'| = \alpha d\}.
    }
\end{align}
This design choice reduces the choice of $\mathcal{D}_S$ to a choice of
{\em subsampling fraction} $\alpha \in (0, 1)$, a decision whose
impact we explore in Section \ref{sec:alpha_choice}.
In practice, we estimate datamodels for {\em several} choices of $\alpha$,
as it turns out that the value of $\alpha$ corresponding to the most useful datamodels
can vary by setting.

\paragraph{(c) What outputs $\modeleval{x}{S'}$ should we track?}
Recall that for any subset $S' \subset S$ of the training set $S$,
$\modeleval{x}{S'}$ is intended to be a specific (potentially stochastic)
function representing the output of a model trained on $S'$ and evaluated on a
target example $x$.
There are, however, several candidates for $\modeleval{x}{S'}$ based on
which model output we opt to track.

In the context of understanding {\em classifiers}, perhaps the simplest such
candidate is the {\em correctness} function (i.e., a stochastic function that is
$1$ if the model trained on $S'$ is correct on $x$, and $0$ otherwise).
However, while the correctness function may be a natural choice for $\modeleval{x}{S'}$,
it turns out to be suboptimal in two ways. First, fitting to the
correctness function ignores potentially valuable information about the model's
confidence in a given decision. Second, recall that our
procedure fits model outputs using a least-squares linear model, which is not
designed to properly handle discrete (binary) dependent variables.

A natural way to improve over our initial candidate would thus be to use
continuous output function, such as cross-entropy loss or correct-label
confidence.
But which exact function should we choose?
In Appendix \ref{app:output_fn}, we describe a heuristic
that we use to guide our choice of the {\em correct-class margin}:
\begin{align}
    \label{eq:margin}
    \modeleval{x}{S'} \coloneqq \text{(logit for correct class)} - \text{(highest incorrect logit)}.
\end{align}

\paragraph{(e) What loss function $\mathcal{L}$ should we minimize?}
In step (d) above, we are free to pick any estimation algorithm for $\theta$.
This freedom of choice allows us to incorporate {\em priors} into the
datamodeling process.
In particular, one might expect that predictions on a given target example
will not depend on every training example. We can incorporate a corresponding
{\em sparsity prior} by adding $\ell_1$ regularization, i.e., setting
\begin{align}
    \label{eq:testset_datamodels}
    \theta = \min_{w \in \mathbb{R}^{\numtrain}}
    \frac{1}{m} \sum_{i=1}^m \lr{w^\top \mask{S_i} - \modeleval{x}{S_i}}^2 + \lambda \|w\|_1,
\end{align}
where we recall that $\numtrain$ is the size of the original training set $S$.
We can use cross-validation to select the regularization parameter $\lambda$ for
each specific target example $x$.

%% file: sections/specializing.tex
We now demonstrate how datamodels can be applied in the context of deep neural
networks---specifically, we consider deep image classifiers trained
on two standard datasets: CIFAR-10 \citep{krizhevsky2009learning} and Functional
Map of the World (\fmow{}) \citep{koh2020wilds} (see Appendix \ref{app:datasets}
for more information on each dataset).

\paragraph{Goal.} As discussed in \Cref{sec:intro}, our goal is to construct a
{\em collection} of datamodels for each dataset, with each datamodel predicting
the model-training outcomes for a {\em specific} target example.
Thus, for both CIFAR and \fmow{}, we fix a deep learning algorithm
(architecture, random initialization, optimizer, etc.),
and aim to estimate a datamodel for each test set example
{\em and} training set example. As a result, we will obtain
$n = 10,000$ ``test set datamodels'' and $n = 50,000$ ``training set
datamodels'' for CIFAR (each being a linear model $g_\theta$ parameterized by a vector
$\theta \in \mathbb{R}^{d}$, for $d = 50,000$); as well as
$n = 3,138$ test set datamodels and $n = 21,404$ training set datamodels for
FMoW (again, parameterized by $\theta \in \mathbb{R}^{d}$ where $d = 21,404$).

\subsection{Implementation details}
\label{subsec:impl_details}
Before applying datamodels to our two tasks of interest, we address a few
remaining technical aspects of datamodel estimation:

\paragraph{Simultaneously estimating datamodels for a collection of target examples.}
Rather than repeat the entire datamodel estimation process for each target
example $x$ of interest separately, we can estimate datamodels for an entire
{\em set} of target examples simultaneously through model reuse. Specifically,
we train a large pool of models on subsets $S_i \subset S$ sampled from the
distribution $\mathcal{D}_S$, and use the {\em same} models to compute outputs
$\modeleval{x}{S_i}$ for each target example $x$.

\paragraph{Collecting a (sufficiently large) datamodel training set.}
The cost of obtaining a single subset-output pair can be non-trivial---in our case, it
involves training a ResNet from scratch on CIFAR-10.
It turns out, however, that recent advances in fast neural network training
\citep{page2018cifar,leclerc2022ffcv} allow
us to train a wealth of models on different $\alpha$-subsets of each
dataset {\em very} efficiently. (For example, for $\alpha = 50\%$ we can use \citep{leclerc2022ffcv} to train
40,000 models/day on an $8 \times \text{A}100$ GPU machine; see Appendix \ref{app:hyperparams} for
details.)
We train $m = 300,000$ CIFAR models and $m = 150,000$ \fmow{} models on
$\alpha = 50\%$ subsets of each dataset. We also train $m$ models for
each subsampling fraction $\alpha \in \{10\%, 20\%, 75\%\}$,
using $\alpha$ to scale $m$. See \Cref{fig:estimation} for a summary of the
models trained.
\begin{table}[h]
    \centering
    \begin{tabular}[h]{@{}lcc@{}}
        \toprule
        & \multicolumn{2}{c}{Models trained} \\
        Subset size ($\alpha$) & {CIFAR-10} & {\fmow{}} \\
        \cmidrule{2-3}
        0.1 & 1,500,000 & -- \\
        0.2 & 750,000 & 375,000 \\
        0.5 & 300,000 & 150,000 \\
        0.75 & 600,000 & 300,000 \\ \bottomrule
    \end{tabular}
    \caption{The number of models (ResNet-9 for CIFAR and ResNet-18 for
    \fmow{}) used to estimate datamodels for each dataset. All
    models are trained from scratch using optimized code \citep{leclerc2022ffcv}
    (e.g., each $\alpha=0.5$ model on CIFAR-10 takes 17s to train (on a single
    A100 GPU) to 90\% accuracy; see \Cref{app:hyperparams} for
    details).}
    \label{fig:estimation}
\end{table}

\paragraph{Computing datamodels when the target example is a training input.}
Recall that the target example $x$ for which we estimate a datamodel can be
arbitrary. In particular, $x$ could itself be a training example---indeed, as
we mention above, our goal is to estimate a datamodel for every image in the
\fmow{} and CIFAR-10 test {\em and} training sets.
When $x$ is in the training set, however, we slightly
alter the datamodel estimation objective \eqref{eq:testset_datamodels} to
exclude training sets $S_i$ containing the target example:
\begin{equation}
    \label{eq:trainset_datamodel}
    \theta = \min_{w \in \mathbb{R}^{\numtrain}}
    \frac{1}{m} \sum_{i=1}^m \mathbb{1}\{x \not\in S_i\} \cdot \lr{w^\top \mask{S_i} - \modeleval{x}{S_i}}^2 + \lambda \|w\|_1.
\end{equation}

\paragraph{Running LASSO regression at scale.}
After training the models, we record the correct-class margin for
all the (train and test) images as well as the training subsets.
Our task now is to estimate,
for each example in the train and test set, a datamodel $g_\theta$ mapping subsets
$\mask{S_i}$ to observed margins. Recall that we compute datamodels via
$\ell_1$-regularized least-squares regression (cf.
\eqref{eq:testset_datamodels}), where we set the regularization parameter
$\lambda$ for each test image independently via a fixed validation set of
trained models.

However, most readily available LASSO solvers require too much memory or are
prohibitively slow for our values of $n$ (the number of datamodels to estimate),
$m$ (the number of models trained and thus the size of the datamodel training
set of subset-output pairs),
and $d$ (the size of the original tasks training set and thus the input
dimensionality of the regression problem in \eqref{eq:testset_datamodels}).
We therefore built a custom solver leveraging the works of
\citep{wong2021leveraging} and
\citep{leclerc2022ffcv}---details of
our implementation are in Appendix \ref{app:ffcv_regression}.

\subsection{Results: linear datamodels can predict deep network training}
\label{sec:estimation_results}
For both datasets considered (CIFAR-10 and \fmow{}), we minimize objectives
\eqref{eq:testset_datamodels} (respectively, \eqref{eq:trainset_datamodel}) yielding a
datamodel $g_{\theta_i}$ for each example $x_i$ in the test set (respectively, training set).
We now assess the quality of these datamodels in terms of how well they
predict model outputs on
{\em unseen} subsets (i.e., fresh samples from $\mathcal{D}_S$).
We refer to this process as {\em on-distribution} evaluation because we are
interested in subsets $S_i$, sampled from the same distribution $\mathcal{D}_S$ as the
datamodel training set, but {\em not} the exact ones used for estimation.
(In fact, recall that we explicitly held out $m_{test}$ subset-output pairs for
evaluation in Section \ref{sec:constructing}.)

We focus here on the collection of datamodels corresponding to
the CIFAR-10 test set, i.e., a set of linear datamodel parameters
$\{\theta_1,\ldots,\theta_n\}$ corresponding to examples $\{x_1,\ldots,
x_n\}$ for $n = 10,000$ (analogous results for \fmow{} are in
\Cref{app:regression_extra}).
In Figure \ref{fig:ondistribution}, aggregating over both datamodels
$\smash{\{g_{\theta_j}\}_{j=1}^n}$ and heldout subsets $\smash{\{S_i\}_{i=1}^m}$, we compare datamodel predictions
$\smash{\theta_j^\top \mask{S_i}}$ to {\em expected} true model outputs
$\smash{\mathbb{E}[\modeleval{x_j}{S_i}]}$ (which we estimate by training 100 models on
the same subset $S_i$ and averaging their output on $x_j$).
The results show a near-perfect correspondence between datamodel
predictions and ground truth. Thus, for a given target example $x$, we can
accurately predict the outcome of ``training a neural network on a random
($\alpha$-)training subset and computing correct-class margin on $x$'' (a process that
involves hundreds of SGD steps on a non-convex objective)
as a simple {\em linear} function of the characteristic vector of the subset.

\begin{figure}[h]
    \begin{minipage}{0.65\textwidth}
    \centering
    \includegraphics[width=\linewidth]{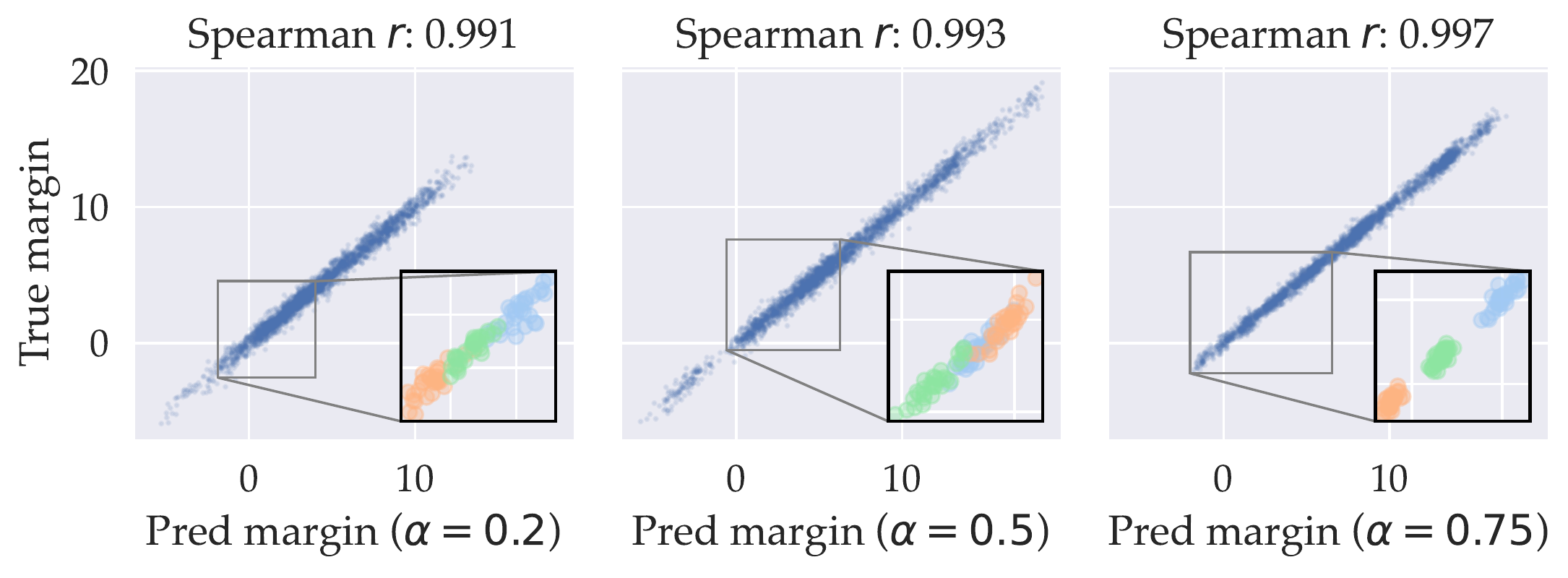}
    \caption{
        {\bf {\em Linear} datamodels accurately predict margins.}
        Each point in the graphs above corresponds to a specific target example
        $x_j$ {\em and} a specific held-out training set $S_i$ from CIFAR-10.
        The $y$-coordinate represents the ground-truth margin $\modeleval{x_j}{S_i}$, averaged across $T$=100 models trained on $S_i$.
        The $x$-coordinate represents the {\em
        datamodel-predicted} value of the same quantity.
        We observe a strong linear correlation (as seen in the main blue line)
        that persists even at the level
        of individual examples (the bottom-right panel shows data for
        three random target examples $x_j$ color-coded by example).
        Corresponding plots for $\alpha = 10\%$ and \fmow{} are in Figures
        \ref{fig:fmow_ondistribution} and \ref{fig:cifar_ondistribution_extra}.
    }
    \label{fig:ondistribution}
    \end{minipage}
    \hspace{1em}
    \begin{minipage}{0.32\textwidth}
        \centering
        \includegraphics[width=\textwidth,trim={0 0 16.85cm 0},clip]{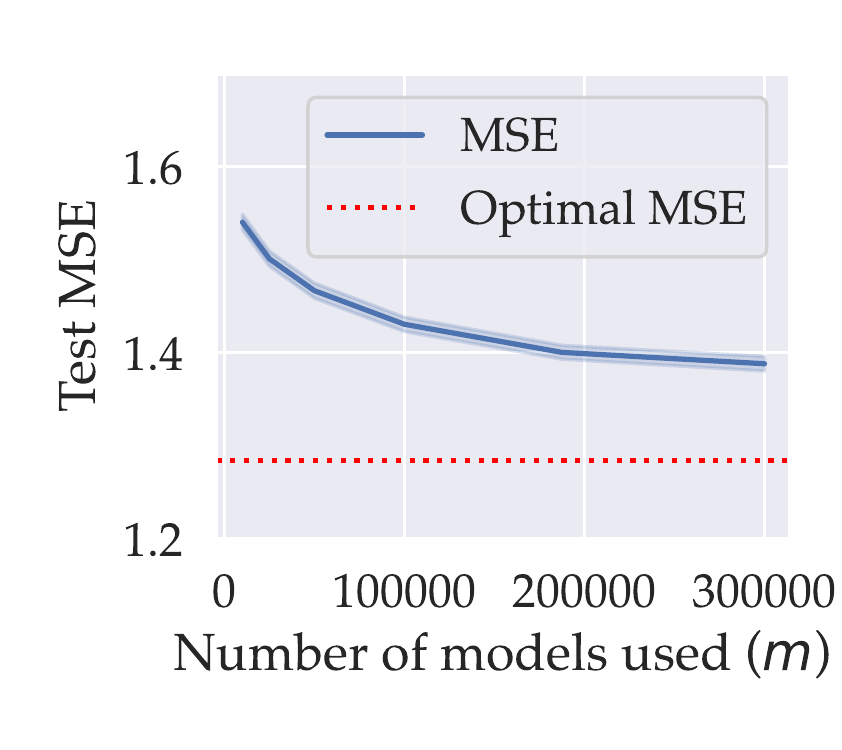}
        \caption{Average mean-squared error (Eqn. \eqref{eq:mse}) for CIFAR-10 test set
        datamodels ($\alpha = 0.5$) as a function of the size of the datamodel
        training set $m$. The
        red line denotes optimal error (Eqn. \eqref{eq:optimal_mse}) based on
        inherent noise in training.}
        \label{fig:sample_complexity}
    \end{minipage}
\end{figure}

\paragraph{Sample complexity.}
We next study the dependence of datamodel estimation on the size of the
datamodel training set $m$. Specifically, we can measure the {\em
on-distribution} average mean-squared error (MSE) as
\begin{align}
   \label{eq:mse}
    \text{MSE}(\{\theta_1,\ldots, \theta_n\}) = \frac{1}{2n} \sum_{j=1}^n \lr{
        \mathbb{E}_{S_i \sim \mathcal{D_S}}\left[
            \lr{\theta_j^\top \mask{S_i} - \modeleval{x_j}{S_i}}^2
        \right]
    }.
\end{align}
To evaluate \eqref{eq:mse}, we replace the inner expectation with an
empirical average, again
using a heldout set of
samples that was not used for estimation.

In Figure \ref{fig:sample_complexity}, we plot average MSE as a function of the
number of trained models $m$.
To put the results into context, we introduce
the {\em optimal mean-squared error loss} (OPT),
which is the MSE \eqref{eq:mse} with datamodel
predictors $\theta_j^\top \mask{S_i}$ replaced by the optimal deterministic
predictors $\mathbb{E}[\modeleval{x_j}{S_i}]$:
\begin{equation}
    \label{eq:optimal_mse}
    \text{OPT} = \frac{1}{2n} \sum_{j=1}^n \lr{
        \mathbb{E}\left[
            \lr{\mathbb{E}\left[\modeleval{x_j}{S_i}\right] - \modeleval{x_j}{S_i}}^2
        \right]
    }.
\end{equation}
Note that OPT is independent of the estimator $g_\theta$ and
measures only the inherent variance in the prediction problem, i.e.,
loss that will necessarily be incurred due only to inherent noise in deep
network training.

\paragraph{The role of regularization.} Finally, in Appendix
\ref{app:regression_extra}, we study the effect of the regularization parameter
$\lambda$ (cf. \eqref{eq:testset_datamodels} and \eqref{eq:trainset_datamodel})
on datamodel performance. In particular, in Figure
\ref{fig:cifar_lambda_effect} we plot the variation in average MSE, on both {\em
on-distribution} subsets (i.e., the exact subsets that we used to optimize
\eqref{eq:testset_datamodels}) and unseen subsets, as we vary the regularization
parameter $\lambda$ in \eqref{eq:testset_datamodels}. We find that---as
predicted by classical learning theory---setting $\lambda = 0$ leads to {\em
overfit} datamodels, i.e., estimators $g_\theta$ that perform well on the exact
subsets that were used to
estimate them, but are poor output predictors on {\em new} subsets $S_i$ sampled
from $\mathcal{D}_S$. (In fact, using $m=300,000$ trained models with $\lambda =
0$ results in higher MSE than using only $m=10,000$ with optimal
$\lambda$, i.e., the left-most datapoint in Figure \ref{fig:estimation}).

%% file: sections/leveraging.tex
Now that we have introduced (Section \ref{sec:intro}), instantiated (Section
\ref{sec:constructing}), and implemented (Section \ref{sec:specializing}) the
datamodeling framework, we turn to some of its applications.
Specifically, we will now show how to apply datamodels within
three different contexts:

\paragraph{Counterfactual prediction.}
We originally constructed
datamodels to predict the outcome of training a model on {\em random} ($\alpha$-)subsets
of the training set. However, it turns out that we can also use datamodels to predict
model outputs on {\em arbitrary} subsets (i.e., subsets that are
``off-distribution'' from the perspective of the datamodel prediction task).
To illustrate the utility of this capability, we will use datamodels to (a)
identify predictions that are {\em brittle} to removal of relatively few
training points, and (b) estimate {\em data counterfactuals}, i.e., the
causal effects of removing groups of training examples.

\paragraph{Train-test similarity.}
We demonstrate that datamodels can identify, for any given target example $x$, a
set of visually similar examples in the training data.
Leveraging this ability, we will identify instances of {\em train-test
leakage}, i.e., when test examples are duplicated (or nearly duplicated) within
the training set.

\paragraph{Data embeddings.}
We show that for a given target example $x$ with corresponding
datamodel $g_\theta$, we can use the parameters $\theta$ of the datamodel as an
{\em embedding} of the target example $x$ into $\mathbb{R}^d$ (where $d$ is the
training set size).
By applying two standard data exploration techniques to such embeddings, we
demonstrate that they capture {\em latent structure} within the data, and enable us to
find model-relevant {\em data subpopulations}.

\input{sections/causal_view}

\subsection{Using datamodels to find similar training examples}
\label{sec:distance_metric}
\input{sections/distance_view}

\subsection{Using datamodels as a feature embedding}
\label{sec:rep_view}
\input{sections/rep_view}

%% file: sections/causal_view.tex
\subsection{Counterfactual prediction}
\label{sec:causal_view}
So far, we have computed and evaluated datamodels entirely within
a supervised learning framework. In particular, we constructed datamodels with
the goal of predicting the outcome of training on {\em random} subsets
of the training set (sampled from a distribution $\mathcal{D}_S$
\eqref{eq:distribution}) and evaluating on a fixed target example $x$.
Accordingly, for each target example $x$, we evaluated its datamodel $g_\theta$
by (a) sampling new random subsets $S_i$ (from the same distribution);
(b) training (a neural network) on each one of these subsets; (c) 
measuring correct-class margin on the target example
$x$; and (d) comparing the results to the datamodel's predictions (namely,
$g_\theta(S_i)$) in terms of {\em expected} mean-squared error (see
\eqref{eq:mse}) over the distribution of subsets.

We will now go beyond this framework, and use datamodels to predict
the outcome of training on {\em arbitrary} subsets of the training set.
In particular, consider a fixed target example $x$ with corresponding datamodel
$g_\theta$. For any subset $S'$ of the training set $S$, we will use the
{\em datamodel-predicted} outcome of training on $S'$ and evaluating on $x$,
i.e., \(
    g_\theta(\mask{S'}),
\)
in place of the {\em ground-truth} outcome
\(
    \modeleval{x}{S'}.
\)
Since $S'$ is an {\em arbitrary} subset of the training set, it is ``out-of-distribution''
with respect to the distribution of fixed-size subsets $\mathcal{D}_S$ that we
designed the datamodel to operate on. As such, using datamodel predictions
in place of end-to-end-model training in this manner is not a priori
guaranteed to work.
Nevertheless, we will demonstrate through two applications that datamodels
{\em can} in fact be effective proxies for end-to-end model training, even for
such out-of-distribution subsets.

\begin{appmode}[Proxy for end-to-end training]
    \vspace{-0.5em}
    We can use datamodel predictions as an efficient, closed-form
    proxy for end-to-end model training. That is, for a test example $x$ with
    datamodel $g_\theta$, and an \underline{arbitrary} subset $S'$ of the
    training set $S$, we can leverage the approximation
    \begin{align*}
        \modeleval{x}{S'} &\approx g_\theta(\mask{S'})
    \end{align*}
\end{appmode}

\subsubsection{Measuring brittleness of individual predictions to training data removal}
\label{subsec:brittleness}
We first illustrate the utility of datamodels as a proxy for model training by
using them to answer the question: {\em how brittle are model predictions to
removing training data?}
While all useful learning algorithms are data-dependent, cases
where model behavior is sensitive to just a few data points are often of
particular interest or concern
\citep{broderick2021automatic,dwork2006our}.
To quantify such sensitivity, we define the {\em data support} $\supp(x)$ of a
target example $x$ as
\begin{align}
    \label{eq:decision_support}
    \supp(x) = \parbox[t]{11cm}{\raggedright
    the smallest training
    subset $R \subset S$ such that classifiers trained on $S \setminus R$ misclassify
    $x$ on average.\footnotemark}
\end{align}
\footnotetext{We define misclassification here as having expected margin
(Eq. \eqref{eq:margin}) less than $0$.}

Intuitively, examples with a small data support are the examples for
which removing a small subset of the training data significantly changes model
behavior, i.e., they are ``brittle'' examples by our criterion of interest. By
computing $\supp(x)$ for every image in the test set, we can thus get an idea of
how brittle model predictions are to removing training data.

\paragraph{Computing data support.}
One way to compute $\supp(x)$ for a given target example $x$
would be to train several models on every possible subset of the training set
$S$, then report the largest subset for which the example was misclassified on
average---the complement of this set would be {\em exactly}
$\supp(x)$. However, exhaustively computing data support in this manner is
simply intractable.

Using datamodels as a proxy for end-to-end model
training provides an (efficient) alternative approach. Specifically, rather
than training models on every possible subset of the training set, we can use
datamodel-predicted outputs $g_\theta(S')$ to perform a {\em guided search},
and only train on subsets for which predicted margin on the target example is
small. This strategy (described in detail in Algorithm
\ref{alg:guided_search} and in Appendix \ref{app:brittleness})
allows us to compute estimates of the data
support while training only a handful of models per target example.

\paragraph{Results.} We apply our algorithm to estimate $\supp(x)$ for 300
random target examples in the CIFAR-10 test set. For over 90\% of these 300
examples, we are able to {\em certify} that our estimated data
support is {\em strictly larger than} the true data support $\supp(x)$
(i.e., that we are not over-estimating brittleness) by training several models
after excluding the estimated data support and checking that the target example is indeed
misclassified on average.

We plot the distribution of estimated data support sizes in Figure
\ref{fig:cifar_brittleness}. Around {\em half} of the CIFAR-10 test images have
a datamodel-estimated data support comprising 250 images or less, meaning that
removing a specific 0.4\% of the CIFAR-10 training set induces
misclassification. Similarly, 20\% of the images had an estimated data support
of less than {\em 40 training images} (which corresponds to {\em 0.08\% of the
training set}).

\begin{figure}[h]
    \centering
    \begin{minipage}{0.47\textwidth}
    \includegraphics[width=\linewidth,trim={0 0 0 0},clip]{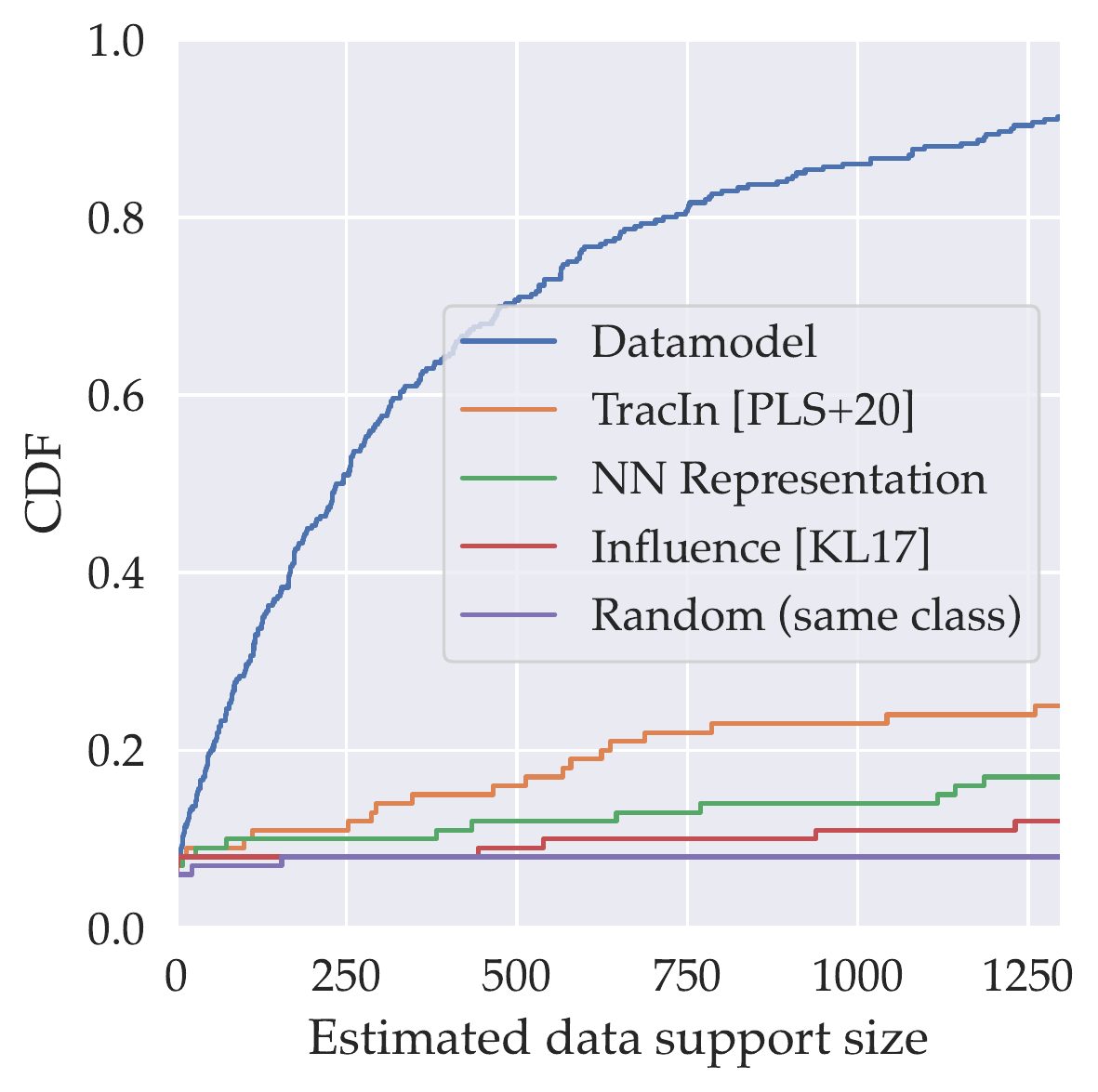}
    \caption{{\bf Characterizing brittleness.}
    We use datamodels to estimate {\em data support}
    (i.e., the minimal set of training examples whose removal causes misclassification)
    for 300 random CIFAR-10 test examples, and plot the cumulative
    distribution of estimated sizes.
    Over 25\% of examples can be
    misclassified by removing \emph{less than 100} (example-specific) training images.
    Also, datamodels yield substantially better upper bounds on support size than
    baselines.}
    \label{fig:cifar_brittleness}
    \end{minipage}
    \hspace{1em}
    \begin{minipage}{0.49\textwidth}
        \captionof{algorithm}{The (intractable) exhaustive search
        algorithm (top) and (efficient) datamodel-guided algorithm (bottom) for
        estimating data support.
        Note that in our linear datamodel setting,
        $G$ (the quantity on line 4 of the second algorithm) 
        actually has a closed-form solution: it is the subset of the
        training set corresponding to the largest $k$ indices of the datamodel
        parameter $\theta$.
        See \Cref{app:brittleness} for implementation and setup details.}
        \label{alg:guided_search}
        \input{paper_figs/brittleness_algorithm}
    \end{minipage}
\end{figure}

\paragraph{Baselines.}
To contextualize these findings, we compare our estimates of data
support to a few natural baselines.
We provide the exact comparison setup in Appendix \ref{app:brittleness_baselines}: in
summary, each baseline technique can be cast as a swap-in alternative to
datamodels for guiding the data support search described above.

It turns out that every baseline we tested provides much
looser estimates of data support (Figure~\ref{fig:cifar_brittleness}).
For example, even the best-performing baseline predicts that one would need to
remove over 600 training images per test image to force misclassification on
20\% of the test set\footnote{Moreover, the data support estimates derived from
the baselines are only ``certifiable'' in the above-described sense (see the
beginning of the ``Results'' paragraph) for 60\% of
the 300 test examples we study (as opposed to 90\% for datamodel-derived estimates).}.
In contrast, our datamodel-guided estimates
indicate that removing 40 train examples is sufficient for misclassifying 20\% of test examples.

\paragraph{Removing versus mislabeling.} Note that the
brittleness we consider in this section (i.e., brittleness to {\em removing}
training examples) is substantively different than brittleness to {\em
mislabeling} examples (as in {\em label-flipping attacks}
\citep{koh2017understanding,xiao2012adversarial,rosenfeld2020certified}).
In particular, brittleness to removal
indicates that there exists a small set of training images whose presence is
{\em necessary} for correct classification of the target example (thus motivating
the term ``data support''). Meanwhile, label-flipping attacks can succeed even
when the target example has a large data support,
as (consistently) mislabeling a set of training examples provide a much stronger signal than simply removing them.
Nevertheless, we can easily adapt the above experiment to test brittleness
to mislabeling---we do so in
Appendix \ref{app:brittleness_flipping}. As one might
expect, test predictions are even {\em more} brittle to
data mislabeling than removal---for 50\% of the CIFAR-10 test set, mislabeling
{\em 35 target-specific training examples} suffices to flip the corresponding
prediction (see Figure~\ref{fig:cifar_flippy} for a CDF).

\subsubsection{Predicting data counterfactuals}
\label{subsec:counterfactuals}
As we have already seen, a simple application of datamodels as a proxy for model
training (on {\em arbitrary} subsets of the training set)
enabled us to identify brittle predictions.
We now demonstrate another, more intricate application of
datamodels as a proxy for end-to-end training: predicting
{\em data counterfactuals}.

For a fixed target example $x$, and a specific subset of the training set
$R(x) \subset S$, a data counterfactual is the causal effect of
removing the set of examples $R(x)$
on model outputs for $x$. In terms of our notation, this effect is precisely
\[
    \mathbb{E}\left[
        \modeleval{x}{S} - \modeleval{x}{S\setminus R(x)}
    \right].
\]
Such data counterfactuals can be helpful tools for finding brittle
predictions (as in the previous subsection), estimating {\em group influence}
(as done by \citep{koh2019accuracy} for linear models), and more broadly for
understanding how training examples combine (through the lens of the
model class) to produce test-time predictions.

\paragraph{Estimating data counterfactuals.} Just as in the last section,
we again use datamodels beyond the supervised learning regime in which they
were developed. In particular, we predict the outcome of a data
counterfactual as
\[
    g_\theta(\mask{S}) - g_\theta(\mask{S \setminus R(x)}),
\]
where again $g_\theta$ is the datamodel for a given target example of interest.
Since $g_\theta$ is a linear function in our case, the above {\em predicted data
counterfactual} actually simplifies to
\[
    \theta^\top \mask{S} - \theta^\top \mask{S \setminus R(x)} = \theta^\top \mask{R(x)}.
\]
Our goal now is to demonstrate that datamodels are useful
predictors of data counterfactuals across a variety of removed sets $R(x)$.
To accomplish this, we use a large set of target examples. Specifically, for each such target
example, we consider different subset sizes $k$; for each such $k$, we
use a variety of heuristics to select a set $R(x)$ comprising $k$ ``examples of
interest.'' These heuristics are:
\begin{itemize}
    \item[(a)] setting $R(x)$ to be the nearest $k$ training examples to the target
    example $x$ in terms of {\em influence score} \citep{koh2017understanding},
    {\em TracIn score} \citep{pruthi2020estimating}, or {\em distance in
    pre-trained representation space}
    \citep{bengio2013representation}\footnote{Note that these methods are
    precisely the ones used as baselines in the previous section.};
    \item[(b)] setting $R(x)$ to be the {\em maximizer} of the
    datamodel-predicted counterfactual, i.e.,
    $$R(x) = \arg\max_{|R|=k} g_\theta(S) - g_\theta(S \setminus R) = \arg\max_{|R|=k}
    \theta^\top \mask{R}.$$
    (Note that since our datamodels are linear, this
    simplifies to excluding the training examples corresponding to the top
    $k$ coordinates of the datamodel parameter $\theta$.)
    \item[(c)] setting $R(x)$ to be the training images corresponding to the {\em
    bottom} (i.e., most negative) $k$ coordinates of the datamodel weight $\theta$.
\end{itemize}
We consider six values of $k$ (the size of the removed subset) ranging from $10$
to $1280$ examples (i.e., $0.02\%-2.6\%$ of the training set).
Thus, the outcome of our procedure is, for each target example,
both {\em true} and {\em datamodel-predicted} data counterfactuals for {\em 30
different training subsets $R(x)$} (six values of $k$ and five different
heuristics).

\paragraph{Results.} In Figure~\ref{fig:cifar_causal}, we plot
datamodel-predicted data counterfactuals against true data counterfactuals,
aggregating across all target examples $x$, values of $k$, and selection
heuristics for $R(x)$.
We find a strong correlation between these two quantities. In particular, across
all factors of variation, predicted and true data counterfactuals have Spearman
correlation $\rho=0.98$ and $\rho=0.94$ for CIFAR-10 and \fmow{} respectively.
In fact, the two quantities are correlated roughly {\em linearly}: we obtain
(Pearson) correlations of $r=0.96$ (CIFAR-10) and $r=0.90$
(\fmow{}) between counterfactuals and their estimates on aggregate.
Correlations are even more pronounced when restricting to any single class of
removed sets (i.e., any single hue in Figure \ref{fig:cifar_causal}).

\begin{figure}[!htbp]
    \centering
    \includegraphics[width=.8\textwidth]{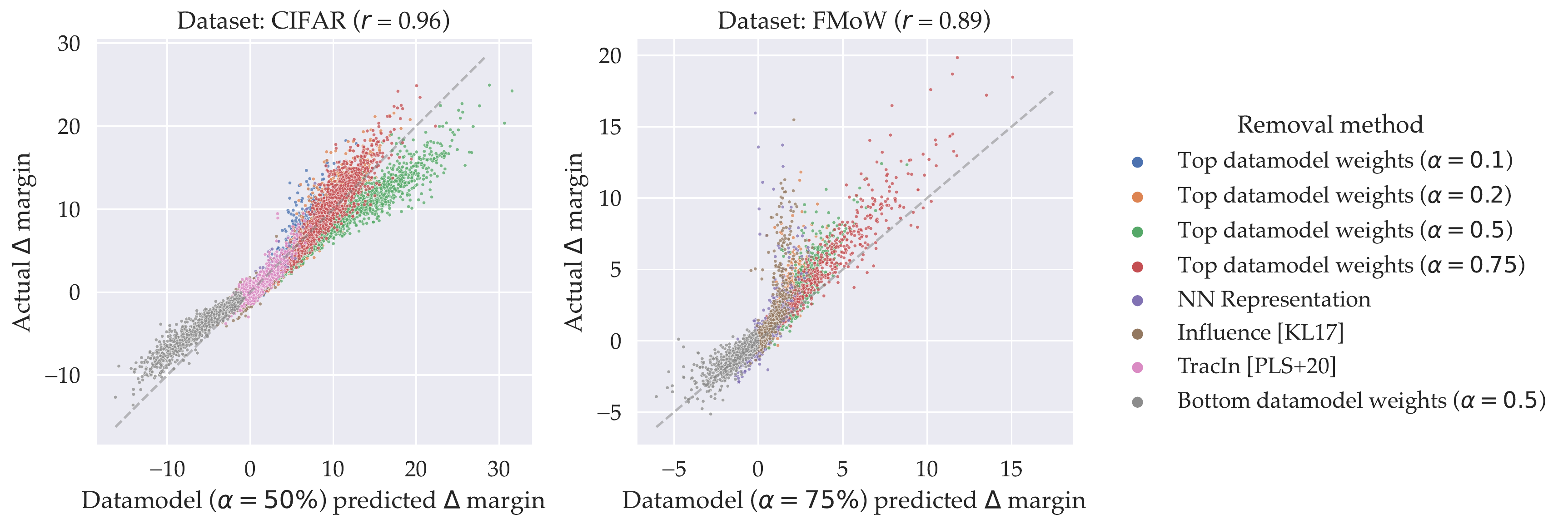}
   \caption{
    {\bf Datamodels predict data counterfactuals.}
    Each point in the graphs above corresponds to a test example and a
    subset $R(x)$ of the original training set $S$, identified by a
    (color-coded) heuristic.
    The $y$-coordinate of each point represents the {\em true}
    difference, in terms of model output on $x$, between training on $S$, and
    training on $S \setminus R(x)$. The $x$-coordinate of each point represents the {\em
    datamodel-predicted} value of this quantity. We plot results for
    \textbf{(right)} CIFAR-10 and
    \textbf{(left)} \fmow{}.
    Datamodel predictions are predictive of the underlying counterfactuals,
    with Pearson coefficients $r$ being 0.96/0.90 for CIFAR/\fmow{} respectively.
    Predictions are computed with datamodels estimated with $\alpha=0.5$  for CIFAR-10 and $\alpha=0.75$ for \fmow{} (cf. \Cref{app:causal_more_alpha} other values of $\alpha$).
    See Appendix~\ref{app:causal} for more experimental details and results.}
    \label{fig:cifar_causal}
\end{figure}

\paragraph{Limits of datamodel predictions.}
We have seen that datamodels accurately predict the
outcome of many natural data counterfactuals,
despite only being constructed to predict outcomes for random
subsets of a fixed size ($\alpha \cdot d$ for $\alpha \in (0, 1)$ and $d$
the training set size).
Of course, due to both estimation error (i.e., we might not have trained
enough models to identify {\em optimal} linear datamodels)
and misspecification error (i.e., the optimal datamodel might not be
linear), we don't expect a perfect correspondence between datamodel-predicted outputs
$g_\theta(\mask{S'})$ and true outputs $f(x;S)$ for {\em all} $2^d$ possible
subsets of the training set.
Indeed, this is part of the reason why we estimated datamodels for several
values of $\alpha$, only one of which is shown in
Figure~\ref{fig:cifar_causal}. As
shown in Appendix \ref{app:causal_more_alpha},
datamodels estimated for other values of $\alpha$ 
still display strong correlation between true and predicted model outputs, but
behave qualitatively differently than the ones shown above (i.e., each value of
$\alpha$ is better or worse at predicting the outcomes of certain types of
counterfactuals).

%% file: paper_figs/brittleness_algorithm.tex
\makeatletter
\algnewcommand{\LineComment}[1]{\Statex \hskip\ALG@thistlm #1}
\makeatother
\begin{algorithmic}[1]
    \Statex\hrulefill
    \Procedure{Exhaustive}{target ex. $x$, trainset $S$}
    \For{$k$ in $\{1,\ldots,d\}$}
        \LineComment{{\em Try every subset of size $k$:}}
        \For{$G$ in $2^S$ with $|G| = k$}
            \If{$\mathbb{E}[\modeleval{x}{S \setminus G}] < 0$}
            \State\Return $G$
            \EndIf
        \EndFor
    \EndFor
    \EndProcedure
    \smallskip
    \setcounter{ALG@line}{0}
    \Procedure{Guided}{$x$, $S$, datamodel $g_\theta$}
    \State $A \gets []$
    \For{$k \in \{10, 20, 40, \ldots \}$}
        \LineComment{{\em Find subset $G_k$ of lowest predicted margin:}}
        \State $G_k \gets \arg\min_{|G|=k} g_\theta(S \setminus G)$
        \State Estimate $\mathbb{E}[\modeleval{x}{S \setminus G_k}$ by re-training
        \State Append $(k, \mathbb{E}[\modeleval{x}{S \setminus G_k}])$ to $A$
    \EndFor
    \LineComment{{\em Piecewise-linear interp. mapping $k$ to margin:}}
    \State $h(\cdot) \gets \Call{Interpolate}{A}$
    \State $\hat{k} \gets k$ for which $h(k) = 0$
    \LineComment{{\em Conservative estimate of data support:}}
    \State\Return $\Call{Top-K}{\theta, \hat{k} \times 1.2}$
    \Statex\hrulefill
    \EndProcedure
\end{algorithmic}

%% file: sections/distance_view.tex
We now turn to another application of datamodels: identifying training
examples that are similar to a given test example.
One can use this primitive to identify issues in
datasets such as duplicated training examples \citep{lee2021deduplicating}
or train-test leakage \citep{barz2020train} (test examples that have near-duplicates
in the training set).

Recall that in our instantiation of the framework,
datamodels predict model output (for a fixed target example) as a
{\em linear} function of the presence of each training 
example in the training set. That is, we predict
the output of training on a subset $S'$ of the training set $S$ as
\[
     g_\theta(\mask{S'}) = \theta^\top \mask{S'}.
\]
A benefit of parameterizing datamodels as simple linear functions is that we can
use the magnitude of the coordinates of $\theta$ to ascertain {\em feature
importance} \citep{guyon2003introduction}. In particular, since in our case each
feature coordinate (i.e., each coordinate of $\mask{S'}$)
actually represents the presence of a particular training example, we can
interpret the highest-magnitude coordinates of $\theta$ as the indices of the
training examples whose presence (or absence) is most predictive of model
behavior (again, on the fixed target example in context).

We now show that these high-magnitude training examples
(a) they visually resemble the target image, yielding a
method for finding similar training examples to a given target; and (b) as a
result, datamodels can automatically detect train-test leakage.
\begin{appmode}[Train-test similarity]
     \vspace*{-0.5em}
     For a test example $x$ with a linear datamodel $g_\theta$, we can interpret the training
     examples corresponding to the highest-magnitude coordinates of $\theta$ as
     the ``nearest neighbors'' of $x$.
\end{appmode}

\subsubsection{Finding similar training examples}
Motivated by the feature importance perspective described above,
we visualize (in Figures~\ref{fig:largest_mag} and \ref{appfig:posneg}) a random
set of target examples from the CIFAR-10 test set
together with the CIFAR-10 training images that correspond to the
highest-magnitude datamodel coordinates for each test image.

\vspace*{-0.5em}
\paragraph{Results.} We find that for a given target example, the highest-magnitude
datamodel coordinates---both positive and negative---consistently correspond to
visually similar training examples.

\begin{figure}[hb]
     \centering
     \includegraphics[width=\textwidth]{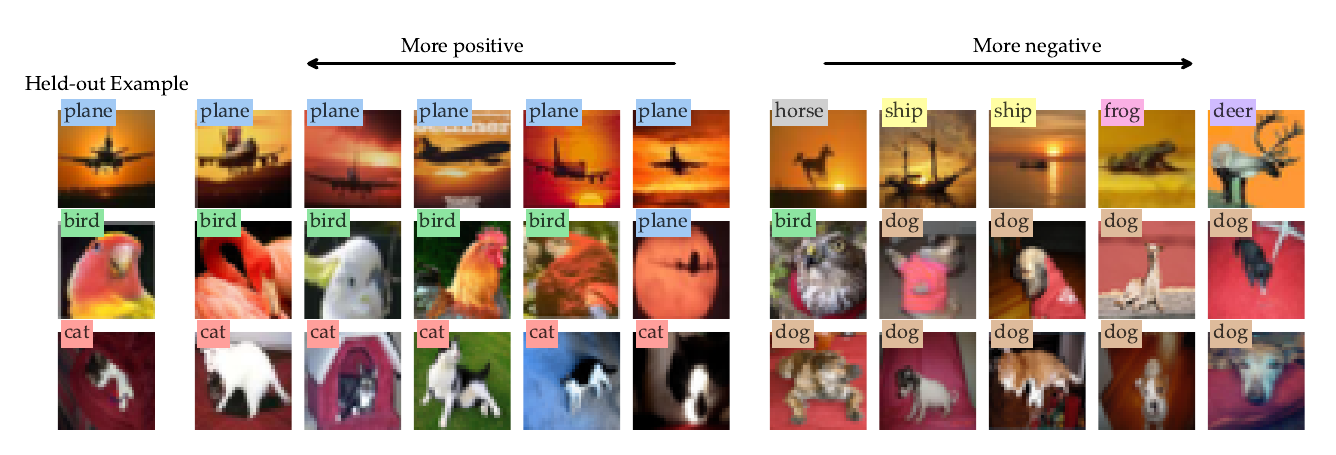}
     \caption{{\bf Large datamodel weights correspond to similar images.}
     Randomly choosing test examples and visualizing their most negative- and
     positive-weight examples for $\alpha=50\%$, we find that large magnitude
     train examples share similarities with their test examples. Top
     negative weights generally correspond to visually similar images from
     other classes. See Appendix~\ref{app:nns} for more examples.}
     \label{fig:largest_mag}
\end{figure}

Furthermore, the exact training images that are surfaced by looking at
high-magnitude weights differ depending on the subsampling parameter
$\alpha$ that we use while constructing the datamodels.
(Recall from Section \ref{sec:constructing} that $\alpha$ controls the
size of the random subsets used to collect the datamodel training set---a
datamodel estimated with parameter $\alpha$ is constructed to predict
outcomes of training on random training subsets of size $\alpha \cdot d$, where
$d$ is the training set size.)
In Figure~\ref{fig:diff_alphas} (and \ref{appfig:more_alphas}), we consider a
pair of target examples from the CIFAR-10 test set, and, for each target example,
compare the top training images from two different datamodels: one estimated
using $\alpha=10\%$, and the other using $\alpha = 50\%$. 
We find that in some cases (e.g., Figure \ref{fig:diff_alphas} left),
the $\alpha=10\%$ datamodel identifies training images that are highly similar
to the target example
but do not correspond to the highest-magnitude coordinates for the
$\alpha=50\%$ datamodel (in other cases, the reverse is true).
Our hypothesis here---which we expand upon in
Section~\ref{sec:alpha_choice}---is that datamodels estimated with lower
$\alpha$ (i.e., based on smaller random training subsets) find
train-test relationships driven by larger groups of examples (and vice-versa).

\begin{figure}[h]
     \centering
     \includegraphics[width=\textwidth]{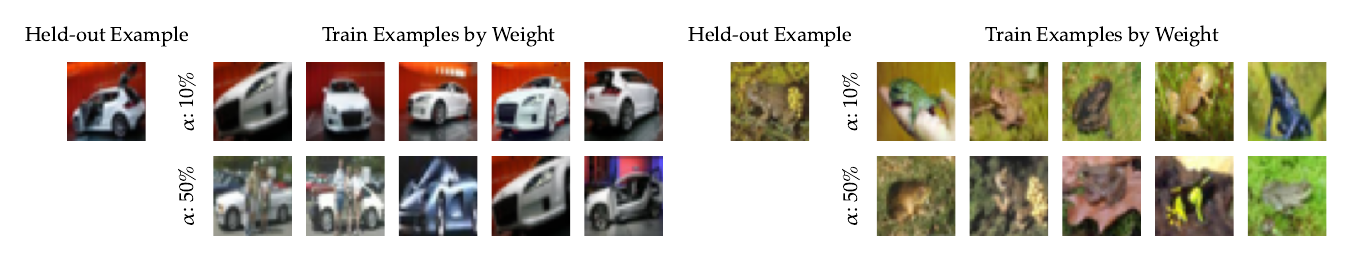}
     \caption{{\bf Datamodels corresponding to different $\alpha$ surface qualitatively
     different images}. For each target example (taken from the
     CIFAR-10 test set), we consider two different datamodels: one estimated
     with $\alpha = 10\%$ (i.e., constructed to predict model outputs on the
     target example after training on random $10\%$ subsets of the CIFAR-10
     training set), and the other estimated with $\alpha = 50\%$.
     For each datamodel, we visualize the training examples corresponding to the largest
     coordinates of the parameter vector $\theta$. On the left we see an example
     where the datamodel estimated with $\alpha=10\%$ (top row) detects a set of
     near-duplicates of the target example that the $\alpha=50\%$ datamodel
     (bottom row) does
     not identify. See Appendix~\ref{app:nns} for more examples.}
     \label{fig:diff_alphas}
\end{figure}

\paragraph{Influence functions.} Another method for finding similar training
images is {\em influence functions}, which aim to estimate the effect of
removing a single training image on the loss (or correctness) for a given test
image. A standard technique from robust statistics \citep{hampel2011robust} (applied
to deep networks by \citet{koh2017understanding}) uses first-order approximation to estimate
influence of each training example. We find (cf. Appendix Figure \ref{appfig:nn_baseline}),
that the high-influence and low-influence examples yielded by this approximation
(and similar methods) often fail to find similar training examples for a given
test example (also see \citep{basu2021influence,hanawa2021evaluation}).

Another approach based on {\em empirical} influence approximation was used by
\citet{feldman2020neural}, who (successfully) use their estimates to identify similar
train-test pairs in image datasets as we do above. We discuss empirical
influence approximation and its connection with datamodeling in Section
\ref{sec:infl_connection}.

\subsubsection{Identifying train-test leakage}
\label{sec:ttl}
We now leverage datamodels' ability to surface training examples similar to a
given target in order to identify {\em same-scene} train-test leakage:
cases where test examples are near-duplicates of,
or clearly come from the same scene as, training examples.
Below,
we use datamodels to uncover evidence of train-test leakage on both CIFAR
and \fmow{}, and show that datamodels
outperform a natural baseline for this task.

\paragraph{Train-test leakage in CIFAR.}
To find train-test leakage in CIFAR-10, we collect ten candidate training
examples for each image in the CIFAR-10 test set---the ones corresponding to the ten
largest coordinates of the test example's datamodel parameter.
We then show crowd annotators (using Amazon Mechanical Turk) tasks that consist
of a random CIFAR-10 test example accompanied by its candidate training
examples. We ask the annotators to label any of the candidate training images
that constitute instances of same-scene leakage (as defined above).
We show each task (i.e., each test example) to multiple annotators, and compute the
``annotation score'' for each of the test example's candidate training examples
as the fraction of annotators who marked it as an instance of leakage.
Finally, we compute the ``leakage score'' for each test
example as the highest annotation score (over all of its candidate train
images). We use the leakage score as a proxy for whether or not the given image
constitutes train-test leakage.

In Figure~\ref{fig:cifar_ttl}, we plot the distribution of leakage scores over
the CIFAR-10 test set, along with random train-test pairs stratified
by their annotation score. As the annotation score
increases, pairs (qualitatively) appear more likely to correspond to leakage
(see Appendix~\ref{app:ttl} for more pairs).
Furthermore,
{\em roughly 10\% of test set images were labeled as train-test leakage
by over half of the annotators that reviewed them}.

\begin{figure}[!h]
     \centering
     \includegraphics[width=\textwidth]{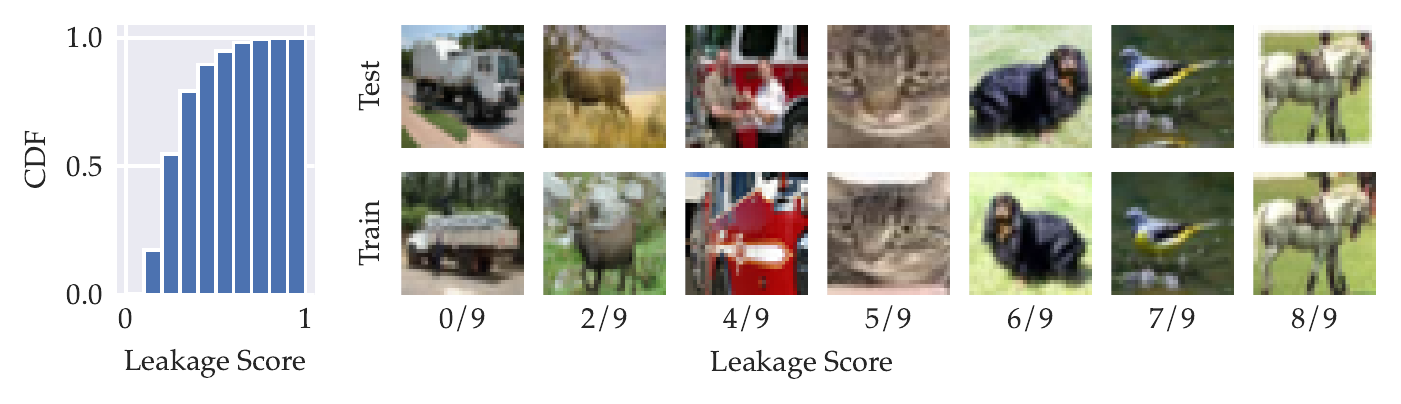}
     \caption{{\bf Finding CIFAR train-test leakage candidates with datamodels.}
     Nine MTurk annotators view each test image alongside the
     train images with largest datamodel weight. The annotators
     then select the train images judged as belonging to the same
     scene as the test image. We measure the {\em annotation score}
     for a given (train, test) pair as the frequency with which annotators
     selected the pair as clearly coming from the same scene. The {\em leakage score} for a test image is defined as the maximum annotation score over all of its candidate train images. (See Appendix
     Section~\ref{app:ttl} for a more detailed setup.)
     {\bf (Left)} Histogram of the leakage score for each image of
     the CIFAR test set.
     {\bf (Right)} Train-test pairs stratified by their leakage score.
     A majority of annotators (annotation score of more than $\frac{1}{2}$)
     consider 10\% of the test set as train-test leakage.
     Many of these pairs are near-duplicate; see more examples in
     Appendix~\ref{app:ttl}.}
     \label{fig:cifar_ttl}
\end{figure}

\begin{figure}[!h]
     \centering
     \includegraphics[width=.9\textwidth]{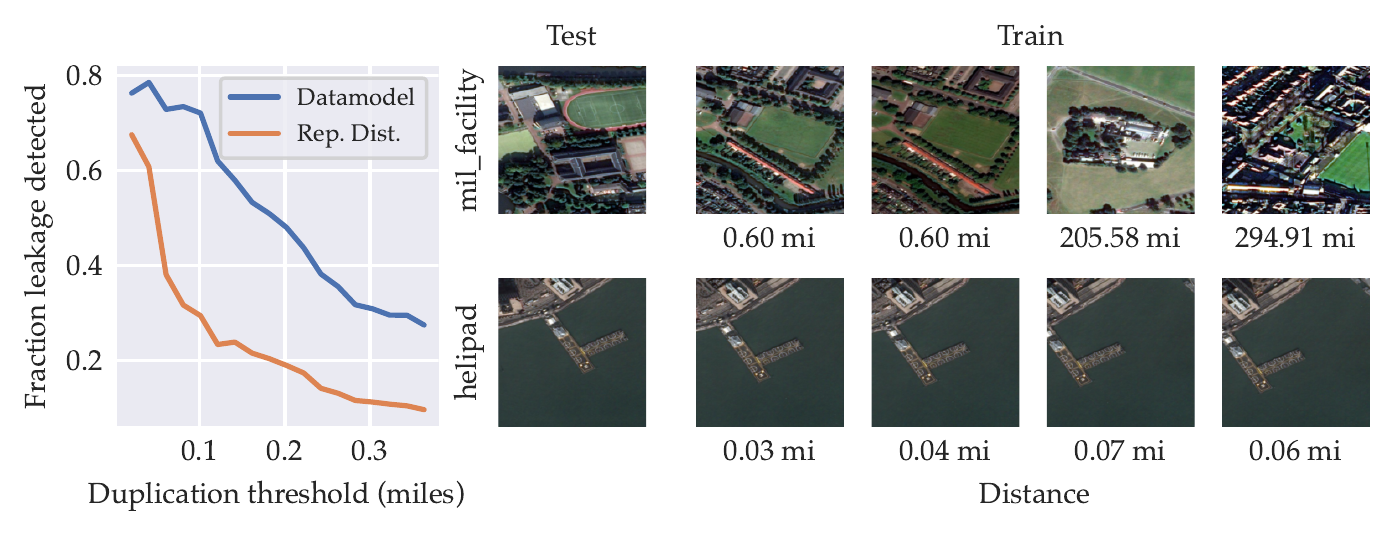}
     \caption{{\bf Datamodels detect same-scene train-test leakage on \fmow{}.}
     \fmow{} images are annotated with geographic
     coordinates. For any distance $d$, we call a test image $x$
     ``leaked'' if it is within $d$ miles of {\em any} training image $x_s$.
     A leaked test image $x$ is considered ``detected'' if the corresponding
     training image $x_s$ has one of the 10 largest datamodel weights
     for $x$.
     {\bf (Left)} With $d$ on the x-axis, we plot the fraction of
     leaked test images that are also detected. As a baseline, we replace
     datamodel weights with (negative) distances in neural network
     representation space.
     {\bf (Right)} for two test examples (top: random; bottom: selected),
     we show the most similar train examples (by datamodel weight),
     labeled by their distance to the test example.}
     \label{fig:fmow_ttl}
\end{figure}

\paragraph{Train-test leakage in \fmow{}.}
To identify train-test leakage on \fmow{}, we begin with the same
candidate-finding process that we used for CIFAR-10.
However, \fmow{} differs from CIFAR in that
the examples (satellite images labeled by category,
e.g., ``port'' or ``arena'') are annotated with {\em geographic coordinates}.
These coordinates allow us to avoid crowdsourcing---instead, we compute the
geodesic distance between the test image and each of the candidates, and use a
simple threshold $d$ (in miles) to decide whether a given test example
constitutes train-test leakage.

Furthermore, we can calculate a ``ground-truth''
number of train-test leakage instances by counting the test
examples whose {\em geodesic} nearest-neighbor in the training set is within the
specified threshold~$d$.\footnote{
It turns out that despite having already been de-duplicated,
about 20\% and 80\% of \fmow{} test images are within 0.25 and 2.6 miles
of a training image, respectively---see Appendix Figure~\ref{appfig:fmow_cdf}.
}
Comparing this ground truth to the number of instances of leakage found within
the candidate examples yields a qualitative measure of the efficacy of our
method (i.e., the quality of candidates we generate).

In Figure \ref{fig:fmow_ttl}, we plot this measure of efficacy (\# instances
found / \# ground truth) as a function of the threshold $d$, and also visualize
examples images from the \fmow{} test set together with their corresponding
datamodel-identified training set candidates. To put our quantitative results
into context, we compare the efficacy of candidates derived from top datamodel
coordinates (i.e., the ones we use here and for CIFAR-10) to that of candidates
derived from {\em nearest neighbors} in the representation space of a pretrained
neural network \citep{bengio2013representation,zhang2018unreasonable}
(examining such nearest neighbors is a standard way of finding
train-test leakage, e.g., used by \citep{barz2020train} to study CIFAR-10 and
CIFAR-100). Datamodels consistently outperform this baseline.

%% file: sections/rep_view.tex
Sections \ref{sec:causal_view} and \ref{sec:distance_metric} illustrate the
utility of datamodels on a {\em per-example} level, i.e., for predicting
the outcome of training on arbitrary training subsets and evaluating on a
specific target example, or for finding similar training images (again, to a
specific target).
We'll conclude this section by demonstrating that datamodels can also help
uncover {\em global structure} in datasets of interest.

Key to this capability is
the following shift in perspective.
Consider a target example $x$ with corresponding {\em linear} datamodel
$g_\theta$, and recall that the datamodel is parameterized by a vector $\theta
\in \mathbb{R}^d$, where $d$ is the training set size.
We optimized $\theta$ to minimize the squared-error of $g_\theta$
\eqref{eq:testset_datamodels} when predicting the outcome of training on random
training subsets and evaluating on the target $x$.
Now, however, instead of viewing the vector $\theta$ as just a parameter of the
predictor $g_\theta$, we cast it as a {\em feature representation} for
the target example itself, i.e., a {\em datamodel embedding} of $x$ into
$\mathbb{R}^d$.
Since the datamodel $g_\theta$ is a linear function of the
presence of each point in the training set, each coordinate of this ``datamodel
embedding'' corresponds to a weight for a specific training example.
One can thus think of a datamodel embedding as a feature vector that
represents a target example $x$ in terms of how predictive each training
example is of model behavior on $x$.

Critically, the coordinates of a datamodel embedding have a {\em consistent}
interpretation across datamodel embeddings, even for different target examples.
That is, we expect similar target examples to be acted upon similarly by the
training set, and thus have similar datamodel embeddings. In the same way, if
model performance on two unrelated target examples is driven by two disjoint
sets of training examples, their datamodel embeddings will be orthogonal.
This intuition suggests that by embedding an entire dataset of examples
$\{x_i\}$ as a set of feature vectors $\{\theta_i \in \mathbb{R}^d\}$, we may be
able to uncover structure in the set of examples by looking for structure in
their datamodel embeddings, i.e., in the (Euclidean) space $\mathbb{R}^d$.

In this section we demonstrate, through two applications, the potential for
such datamodel embeddings to discover dataset structure in this way. In
Section \ref{subsec:pca}, we use datamodel embeddings to partition datasets into
disjoint clusters, and in Section \ref{sec:spectral} we use principal component
analysis to get more fine-grained insights into dataset structure. To emphasize
our shift in perspective (i.e., from $\theta$ being just a parameter of a
datamodel $g_\theta$, to $\theta$ being an embedding for the target example
$x$), we introduce an {\em embedding function} $\varphi(x) \mapsto \theta$ which
maps a particular target example to the weights of its corresponding datamodel.

\begin{appmode}[Datamodel embeddings]
    \vspace*{-0.5em}
    We can use datamodels as a way to embed any given target example into the
    same (Euclidean) space $\mathbb{R}^d$, where $d$ is the training set size.
    Specifically, we can define the {\em datamodel embedding} $\varphi(x)$ for a
    target example $x$ as the weight vector $\theta \in \mathbb{R}$ of the
    datamodel corresponding to $x$.
\end{appmode}

\subsubsection{Spectral clustering with datamodel embeddings}
\label{sec:spectral}
We begin with a simple application of datamodel embeddings, and show that they
enable high-quality clustering.
Specifically, given two examples $x_1$ and $x_2$, datamodel embeddings induce a
natural {\em similarity measure} between them:
\begin{equation}
    \label{eq:datamodel_sim}
    d(x_1, x_2) \coloneqq K(\varphi(x_1), \varphi(x_2)),
\end{equation}
where we recall that $\varphi(\cdot)$ is the {\em datamodel embedding function}
mapping target examples to the weights of their corresponding datamodels, and
$K(\cdot, \cdot)$ is any kernel function (below, we use the RBF kernel)\footnote{A
kernel function $K(\cdot, \cdot)$ is a similarity measure that
computes the inner product between its two arguments in a transformed inner product
space (see \citep{shawetaylor2004kernel} for an introduction). The
RBF kernel is $K(v_1, v_2) = \exp\{-\|v_1-v_2\|^2/2\sigma^2\}$}.
Taking this even further, for a set of $k$ target examples $\{x_1, \ldots, x_k\}$,
we can compute a full
{\em similarity matrix} $A \in \mathbb{R}^{k \times k}$, whose entries are
\begin{equation}
    A_{ij} = d(x_i, x_j).
\end{equation}
Finally, we can view this similarity matrix as an {\em adjacency matrix} for a
(dense) graph connecting all the examples $\{x_1,\ldots x_k\}$: the edge between
two examples will be $d(x_i, x_j)$, which is in turn the kernelized inner
product between their two datamodel weights. We expect similar examples to
have high-weight edges between them, and unrelated examples to have (nearly)
zero-weight edges between them.

Such a graph unlocks a myriad of graph-theoretic tools for exploring datasets
through the lens of datamodels (e.g., cliques in this graph should be
examples for which model behavior is driven by the same subset of training
examples).
However, a complete exploration of these tools is beyond the scope of our work:
instead, we focus on just one such tool: spectral clustering.

At a high level, spectral clustering is an algorithm that takes as input any
similarity graph $G$ as well as the number of clusters $C$, and outputs a
partitioning of the vertices of $G$ into $C$ disjoint subsets, in a way that
(roughly) minimizes the total weight of inter-cluster edges.
We run an off-the-shelf spectral clustering
algorithm on the graph induced by the similarity matrix $A$ above for the images
in the CIFAR-10 test set.
The result (\Cref{fig:clustering} and \Cref{app:spectral}) demonstrates a simple unsupervised method
for uncovering subpopulations in datasets.

\begin{figure}[!htbp]
    \centering
    \includegraphics[width=.9\linewidth,trim={0 2.5cm 0 0},clip]{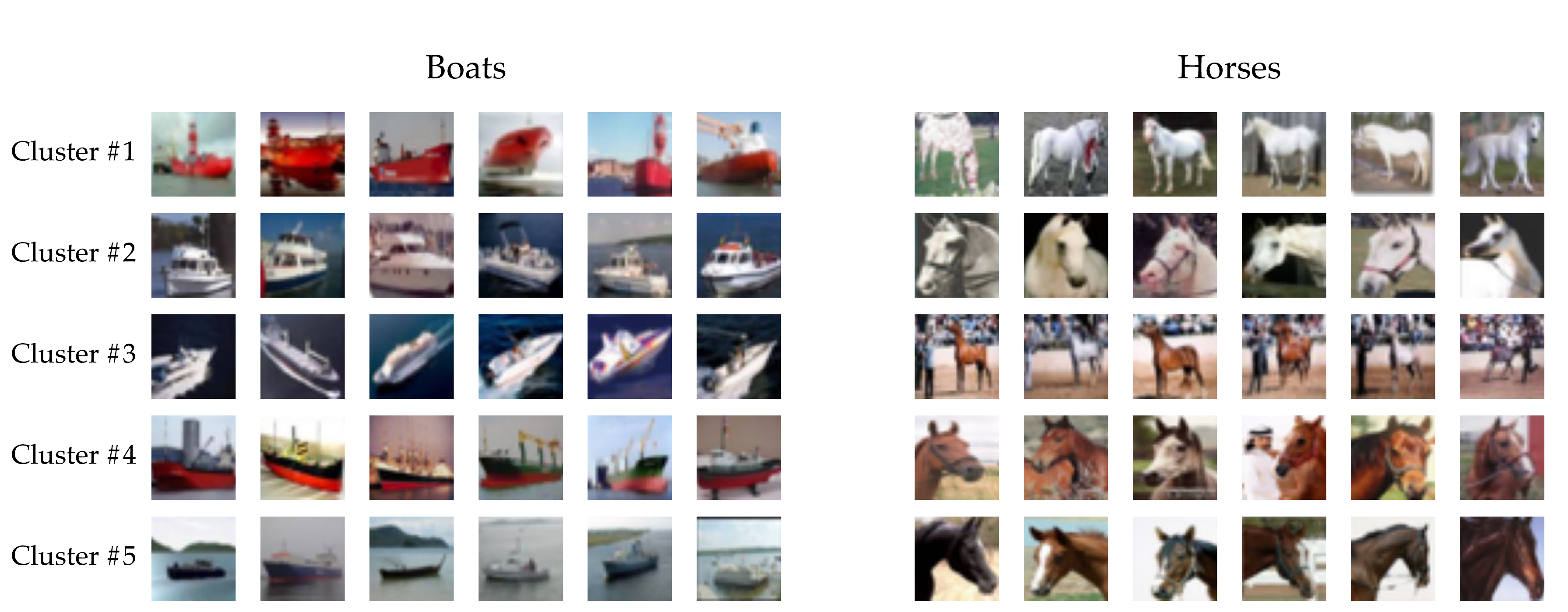}
    \caption{{\bf Spectral clustering on datamodel embeddings finds subpopulations.}
    For each CIFAR-10 class, we first compute a similarity score between all
    datamodel embeddings (we use $\alpha=20\%$ datamodels), 
    then run spectral clustering on the resulting matrix.
    We show the top clusters with the lowest average distance to the cluster center (in the embedding space); each row shows six random images from the given cluster. Each cluster seems to correspond to a specific subpopulation with shared, distinctive visual features.
    See~\Cref{fig:clustering_extra,fig:clustering_extra_compare} for more examples from other classes and comparison across $\alpha$.}
    \label{fig:clustering}
\end{figure}

\subsubsection{Analyzing datamodel embeddings with PCA}
\label{subsec:pca}
We observed above that datamodel embeddings encode enough
information about their corresponding examples to cluster them into (at least
qualitatively) coherent groups.
We now attempt to gain even further insight into the structure of these datamodel
embeddings, in the hopes of shedding light on the structure of the underlying
dataset itself.

Datamodel embeddings are both high-dimensional and sparse, making analyzing
them directly (e.g., by looking at the variation of each coordinate) a
daunting task.
Instead, we leverage a canonical tool for finding structure in
high-dimensional data: principal component analysis (PCA).

PCA is a dimensionality reduction technique which---given a set of
embeddings $\{\varphi(x_i) \in \mathbb{R}^d\}$ and any $k \ll d$---returns a
{\em transformation function} that maps any
embedding $\varphi(x) \in \mathbb{R}^d$ to a new embedding $\widetilde{\phi}(x)
\in \mathbb{R}^k$, such that:
\begin{itemize}
    \item[(a)] each of the $k$ coordinates of the transformed embeddings
    is a (fixed) linear combination of the coordinates of the initial datamodel
    embeddings, i.e., $\widetilde{\varphi}(x) = \bm{M}\cdot \varphi(x)$ for a fixed
    $k \times d$ matrix $\bm{M}$;
   \item[(b)] transformed embeddings preserve as much information as
    possible about the original ones.
    More formally, we find the matrix $\bm{M}$ that allows us to {\em
    reconstruct} the given set of embeddings $\{\varphi(x_i) \in \mathbb{R}^d\}$
    from their transformed counterparts with minimal error.
\end{itemize}
Note that in (a), the $i$-th coordinate of a transformed embedding is
always the {\em same} linear combination of the
corresponding original embedding (and thus, each coordinate of the transformed
embedding has a concrete interpretation as a weighted combination of datamodel
coefficients). The exact coefficients of this combination (i.e.,
the rows of the matrix $\bm{M}$ above) are called the first $k$ {\em principal
components} of the dataset.

We apply PCA to the collection of datamodel embeddings $\{\varphi(x_i) \in
\mathbb{R}^d\}_{i=1}^d$ for the CIFAR-10 training set, and use the result to
to compute new $k$-dimensional embeddings for each target example in both the
training set and the test set (i.e., by computing each target example's
datamodel embedding then transforming it to an embedding in $\mathbb{R}^k$).
We can then look at each coordinate in the new, much more manageable
($k$-dimensional) embeddings.\footnote{We first {\em
normalize} each datamodel embedding before transforming them (i.e., we transform
$\varphi(x)/\|\varphi(x)\|$).}

\paragraph{Coordinates identify subpopulations.} Our point of start in analyzing
these transformed embeddings is to examine each transformed coordinate separately.
In particular, in Figure \ref{fig:pca} we visualize, for a few sample coordinate
indices $i \in [k]$, the target examples whose transformed embeddings
have particularly high or low values of the given coordinate (equivalently,
these are the target examples whose datamodel embeddings have the highest or
lowest projections onto the $i$-th principal component). We find that:
\begin{itemize}
    \item[(a)] the examples whose transformed embeddings have a large $i$-th
    coordinate all (visually) share a common feature: e.g., the first-row images
    in Figure \ref{fig:pca} share similar pose and color composition;
    \item[(b)] this (visual) feature is consistent across both train and test set
    examples\footnote{Recall that we computed the PCA transformation to preserve the
    information in only the {\em training set} datamodel embeddings.
    Thus, this result suggests that the transformed embeddings computed by
    PCA are not ``overfit'' to the specific examples that we used to compute
    it.}; and
    \item[(c)] for a given coordinate, the most positive images and most
    negative images (i.e., the left and right side of each row of
    Figure~\ref{fig:pca}, respectively)
    either (a) have a differing label but share the
    same common feature or (b) have the same label but differ along the relevant
    feature.
\end{itemize}

\begin{figure}[!h]
    \centering
    \includegraphics[width=.9\linewidth,trim={0 0 0 1cm}, clip]{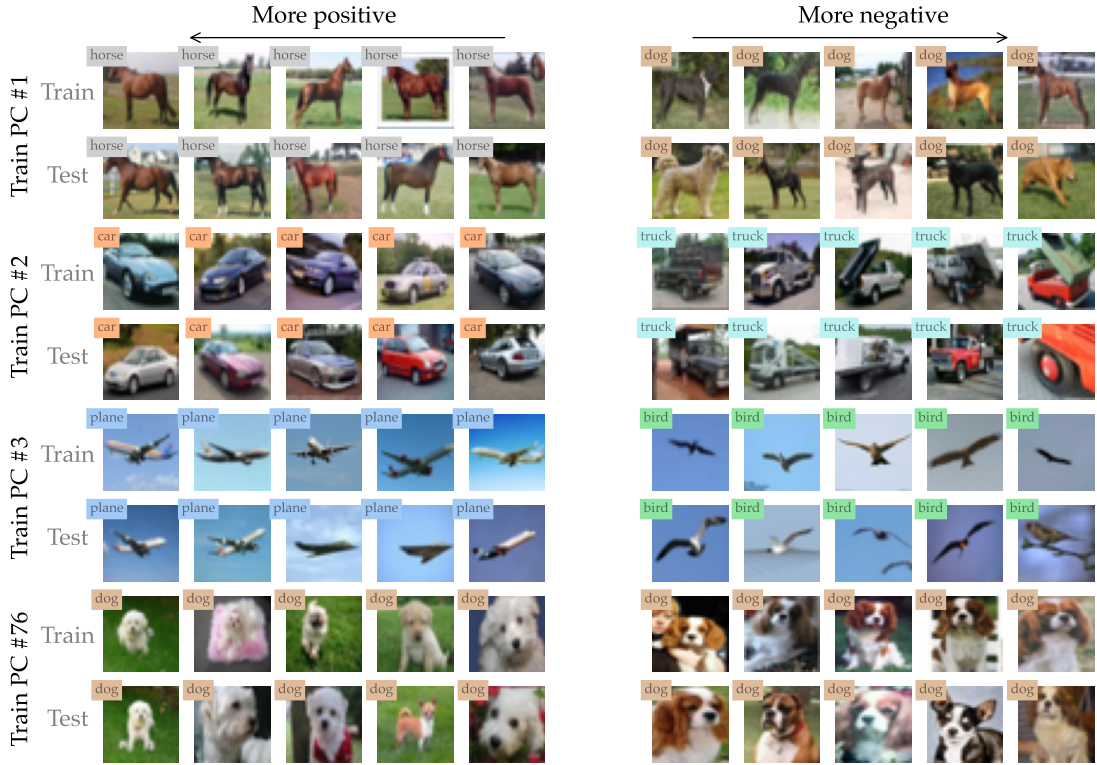}
    \caption{
        {\bf PCA on datamodel embeddings.}
        We visualize the top three principal components (PCs) and a randomly
        selected PC from the top 100.
        In the $i$-th row, the left-most (right-most) images
        are those whose datamodel embeddings have the highest (lowest)
        normalized projections onto the $i$-th principal component $v_i$.
        Highest magnitude images along each direction share qualitative
        features; moreover, images at opposite ends suggest a \emph{feature
        tradeoff}---a combination of images in the training set that helps
        accuracy on one subgroup but hurts accuracy on the other.
        See
        \Cref{fig:pca_extra} for more datamodel PCA components.}
    \label{fig:pca}
\end{figure}

\paragraph{Principal components are {\em model-faithful}.}
In Appendix~\ref{app:pca}, we verify that not only are the groups of images
found by PCA visually coherent, they are in fact rooted in how the model class
makes predictions. In particular, we show that one can find, for any coordinate
$i \in [k]$ of the transformed embedding, the training examples that are most
important to that coordinate. 
Furthermore, retraining without these examples significantly decreases
(increases) accuracy on the target examples with the most positive (negative)
coordinate $i$, suggesting that the identified principal components actually
reflect model class behavior.

\subsubsection{Advantages over penultimate-layer embeddings}
In the context of deep neural networks, the word ``embedding'' typically refers
to features extracted from the penultimate layer of a fixed pre-trained model
(see \citep{bengio2013representation} for an overview). These ``deep
representations'' can serve as an effective proxy for visual similarity
\citep{barz2020train, zhang2018unreasonable}, and also enable a suite of
applications such as clustering \citep{guerin2017cnn} and feature visualization
\citep{olah2017feature,engstrom2019adversarial,azizpour2015factors,ben2007analysis}.

Here, we briefly discuss a few advantages of datamodel-based embeddings over
their standard penultimate layer-based counterparts.
\begin{itemize}
    \item {\bf Axis-alignment}:
    datamodel embeddings are {\em axis-aligned}---each embedding component
    directly corresponds to index into the training
    set, as opposed to a more abstract or qualitative concept.
    As a corollary, aggregating or comparing different datamodel embeddings
    for a given dataset is straightforward, and does not require any alignment
    tools or additional heuristics. This is not the case for network-based
    representations, for which the right way to combine representations---even
    for two models of the same architecture---is still
    disagreed upon \citep{kornblith2019similarity,bansal2021revisiting}.
    In particular, we can straightforwardly compare datamodel embeddings across
    different target examples, model architectures,
    training paradigms, or even datamodel estimation techniques---as long as the
    set of training examples being stays the same, any resulting datamodel has a
    uniform interpretation.
    \item {\bf Richer representation}: the space of datamodel embeddings seems
    significantly richer than that of standard representation space. In
    particular, Appendix Figure \ref{fig:explained_variance} shows that for
    standard representation space, {\em ten linear directions} suffice to
    capture {\em 90\% of the variation} in training set representations.
    The ``effective dimension'' of datamodel representations is much higher,
    with the top {\em 500 principal components} explaining only {\em 50\% of the
    variation} in training set datamodel embeddings. This difference manifests
    qualitatively when we redo our PCA study on standard representations
    (Appendix Figure \ref{fig:pca_rep}): principal components beyond the 10th
    lack both the perceptual quality and train-test consistency exhibited
    by those of datamodel embeddings (e.g., for datamodels even the 76th
    principal component, shown in Figure~\ref{fig:pca}, exhibits these qualities).
    \item {\bf Ingrained causality}:
    datamodel embeddings inherently encode information about how the model
    class generalizes. Indeed, in Section \ref{subsec:pca} we verified via counterfactuals
    that insights extracted from the principal components of $\Theta$ actually
    reflect underlying model class behavior.
\end{itemize}

%% file: sections/role_of_alpha.tex
\begin{figure}[ht]
    \centering
    \includegraphics[width=.9\textwidth]{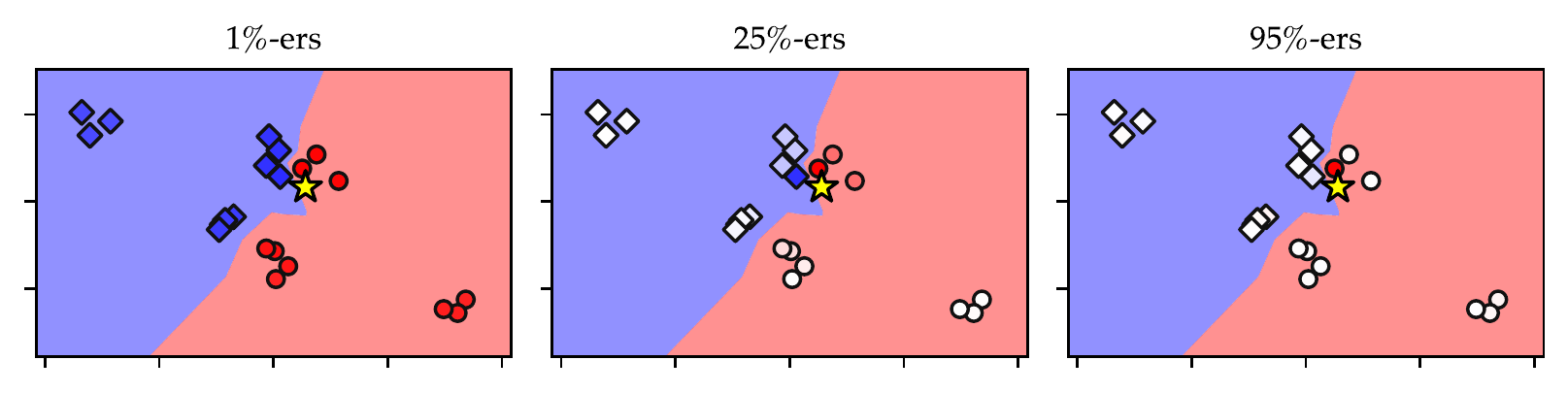}
    \caption{{\bf Datamodels capture data relationships at varying levels of granularity.}
    We illustrate the role of \emph{subsampling fraction} $\alpha$ of
    datamodels by considering a nearest-neighbor classifier in two dimensions.
    In the datamodel for the target example ($\star$, yellow), the red (blue) examples have  positive (negative) weights, with the shade indicating the magnitude.
    At large values of $\alpha$ (right),
    the model identifies only local relationships. Meanwhile, at small values of
    $\alpha$ (left), we can identify more global relationships, but at the cost of granularity. Intermediate values of $\alpha$ (middle) provide a
    smooth tradeoff between these two regimes.}
    \label{fig:2d-sim}
\end{figure}

We have used datamodels estimated using several choices of the subsampling fraction $\alpha$,
and saw that the value of $\alpha$ corresponding to the most useful datamodels
can vary by setting.
In particular, the visualizations in \Cref{fig:diff_alphas} suggest that
datamodels estimated with lower
$\alpha$ (i.e., based on smaller random training subsets) find
train-test relationships driven by larger groups of examples (and vice-versa).
Here, we explore this intuition further using thought
experiment, toy example, and numerical simulation. Our goal is to intuit how
different choices of $\alpha$ can lead to substantively different datamodels.

First, consider the task of estimating a datamodel for a prototypical
image $x$---for example, a plane on a blue sky background.
As $\alpha \to 1$, the sets $S_i$ sampled from $\mathcal{D}_S$ are relatively
large---if these sets have enough other
images of planes on blue skies, we will observe little to no
variation in $\modeleval{x}{S_i}$, since any predictor trained on $S_i$ will
perform very well on $x$. As a result, a datamodel for $x$ estimated with
$\alpha \to 1$ may assign very little weight to any particular image, even
if in reality their {\em total} effect is actually significant.

Decreasing $\alpha$, then, offers a solution to this problem. In particular, we allow
the datamodel to observe cases where entire {\em groups} of training examples are
not present, and {\em re-distribute} the corresponding effect back to the
constituents of the group (i.e., assigning them all a share of the weight).

Now, consider a highly atypical yet correctly classified example,
whose correctness relies on just the presence of just a few images from the
training set.
In this setting, datamodels estimated with a small value of $\alpha$ may be
unable to isolate these training points, since they will constantly distribute
variation in $\modeleval{x}{S_i}$ among a large group of non-present images.
Meanwhile, using a large value of $\alpha$ allows the estimated datamodel to
place weight on the correct training images (since $x$ will be
classified correctly until some of the important training images are not present
in $S_i$).

In line with this intuition, decreasing $\alpha$ in Figure \ref{fig:2d-sim}
(i.e., moving from right to left)
leads to datamodels that assign weight to increasingly large neighborhoods of
points around the target input.
This example and the above reasoning lead
us to hypothesize that larger (respectively, smaller) $\alpha$ are
better-suited to cases where model predictions are driven by smaller (respectively,
larger) groups of training examples.
In Appendix \ref{app:simulation}, we perform a more quantitative analysis of
the role of $\alpha$, this time by studying an underdetermined linear regression
model on data that is organized into overlapping {\em subpopulations}.
Our findings in this setting (see Figure \ref{fig:sim_correlations})
mirror our intuition thus far---in particular, smaller
values of $\alpha$ result in datamodels that were more predictive on
{\em larger} subpopulations in the training set, whereas higher values of
$\alpha$ tended to work better {\em smaller} subpopulations.

%% file: sections/related.tex
Datamodels build on a rich and growing body of literature in machine learning,
statistics, and interpretability.
In this section, we illustrate some of the connections to these fields,
highlight a few of the most closely related works to ours.

\subsection{Connecting datamodeling to empirical influence estimation}
\label{sec:infl_connection}
\input{sections/recasting}

\subsection{Other connections}
\label{sec:other_related}
\paragraph{Influence functions and instance-based explanations.}
Above, we contrasted datamodels with
{\em empirical influence functions}, which measure the counterfactual effect of
removing individual training points on a given model output. Specifically, in
that section and the corresponding Appendix \ref{app:ate_compare}, we
discussed the subsampled influence estimator of \citet{feldman2020neural}, who
use influences to study the memorization behavior of standard vision models.
We now provide a brief overview of a variety of other methods for influence
estimation developed in prior works.

First-order influence functions are a canonical tool in robust
statistics that allows one to approximate the impact of removing a data point on
a given parameter without re-estimating the parameter itself
\citep{hampel2011robust}.
\citet{koh2017understanding} apply influence functions to both a variety of
classical machine learning models and to penultimate-layer embeddings from
neural network architectures,
to trace model's predictions back to individual training examples.
In classical settings (namely, for a logistic regression model),
\citet{koh2019accuracy} find that influence functions are also
useful for estimating the impact of \emph{groups} of examples.
On the other hand, \citet{basu2021influence} finds that approximate influence functions
scale poorly to deep neural network architectures; and \citet{feldman2020neural}
argue that understanding the dynamics of the penultimate layer is insufficient for
understanding deep models' decision mechanisms.
Other methods for influence approximation (or more generally, instance-level
attribution) include gradient-based methods \citep{pruthi2020estimating} and
metrics based on representation similarity
\citep{charpiat2019input,yeh2018representer}---see \citep{hanawa2021evaluation}
for a more detailed overview.
Finally, another related line of work \citep{ghorbani2019data,
jia2019towards,wang2021unified} uses {\em Shapley values}
\citep{shapley1951notes}
to assign a value to datapoints
based on their
contribution to some {\em aggregate} metric (e.g., test
accuracy).

As discussed in Section \ref{sec:infl_connection}, datamodels serve a different
purpose to influence functions---the former constructs an explicit
statistical model, whereas the latter
measures the counterfactual value of each training point. Nevertheless, we
find that wherever efficient influence approximations and datamodels are
quantitatively comparable (e.g., see Section \ref{sec:causal_view} or Appendix
\ref{app:ate_compare}) datamodels predict model behavior better.

\paragraph{Pixel-space surrogate models for interpretability.} Datamodels are
essentially surrogate models for the function mapping training data to
predictions.
Surrogate models from {\em pixel-space} to predictions are popular tools in
machine learning interpretability \citep{ribeiro2016why,lundberg2017unified,sokol2019blimey}.
For example, LIME \citep{ribeiro2016why} constructs a local linear model mapping
test images to model predictions.
Such surrogate models
try to understand, for a {\em fixed} model, how the features of a given test example
change the prediction. In contrast, datamodels hold the test
example fixed and instead study how the images present in the training set
change the prediction.

In addition to the advantages of our data-based view stated in Section
\ref{sec:intro}, datamodels have two further advantages over pixel-level
surrogate models: (a) a clear notion of {\em missingness} (i.e., it is easy to
remove a training example but usually hard to ``remove'' pixels 
\citep{sturmfels2020visualizing,jain2022missingness}); and
(b) {\em globality} of predictions---pixel-level surrogate models are typically
accurate within a small neighborhood of a given input in pixel space,
whereas datamodels model entire distribution over subsets of the training set,
and remain useful both on- and off-distribution.

In other contexts, surrogate models are also used to evaluate data points for active learning and
coreset selection \citep{lewis1994heterogeneous, coleman2020selection}.
\citet{coleman2020selection} find that shallow neural networks trained with
fewer epochs can be a good proxy for a larger model when evaluating data for
these applications.

\paragraph{Model understanding beyond fixed weights.} Recall (from Section \ref{sec:intro}) that datamodels are, in part,
inspired by the fact that re-training deep neural networks using the same data
and model class leads to models with similar accuracies but vastly different
individual predictions.
This phenomenon has been observed more broadly.
For example, \citet{sellam2021multiberts}
make this point explicitly in the context of BERT \citep{devlin2019bert} pre-trained
language models.
Similarly, \citet{nakkiran2020distributional} make note of this non-determinism for networks trained on
the same training {\em distribution} (but not the same data), while
\citet{jiang2021assessing} find that the same is true for networks trained on
the same exact data. \citet{damour2020underspecification} find that on
out-of-distribution data even overall accuracy is highly random.
More closely to the spirit to our work, \citet{zhong2021larger}
find that non-determinism of individual predictions poses a
challenge for comparing different model architectures. (They also propose a set of
statistical techniques for overcoming this challenge.)
More traditionally, the non-determinism is leveraged by Bayesian \citep{neal1996bayesian} and ensemble methods \citep{lakshminarayanan2017simple}, which use a distribution over model weights to improve aspects of inference such as calibration of uncertainty.

\paragraph{Learning and memorization.} Recent work (see
\citep{feldman2019does,chatterjee2018learning,zhang2016understanding,bresler2020corrective}
and references therein) brings to light the interplay between learning and
memorization, particularly in the context of deep neural networks.
While memorization and generalization may seem to be at odds, the picture is more subtle.
Indeed, \citet{chatterjee2018learning} builds a network of small lookup tables on small vision
datasets to show that purely memorization-based systems can still generalize-well.
\citet{feldman2019does} suggests that memorization of atypical examples
may be {\em necessary} to generalize well due to a long tail of subpopulations
that arises in standard datasets. \citet{feldman2020neural} find some empirical
support for this hypothesis by identifying memorized images on CIFAR-100 and ImageNet and showing that removing them
hurts overall generalization. Relatedly, \citet{brown2021when} proves that for certain natural distributions, memorization of a
large fraction of data, even data irrelevant to the task at hand, is necessary
for close to optimal generalization.
For state of the art models, recent works (e.g.,
\citep{carlini2019secret, carlini2021extracting}) show that one can indeed extract sensitive
training data, indicating models' tendency to memorize.

Conversely, it has been observed that differentially private (DP)
machine learning models---whose aim is precisely to avoid memorizing the training
data---tend to exhibit poorer generalization than their memorizing counterparts
\citep{abadi2016deep}. Moreover, the impact on generalization from DP is
disparate across subgroups \citep{bagdasaryan2019differential}. A similar
effect has been noted in the context of neural network pruning \citep{hooker2019selective}.
Datamodeling may be a useful tool for studying these phenomena and, more broadly,
the mechanisms mapping data to predictions for modern learning algorithms.

\paragraph{Brittleness of conclusions.} A long line of work in statistics
focuses on testing the {\em robustness} of statistical conclusions to the
omission of datapoints.
\citet{broderick2021automatic} study the robustness of econometric analyses to removing a
(small) fraction of data. Their method uses a Taylor-approximation based metric
to estimate the most influential subset of examples on some target quantity,
similar in spirit to our use of datamodels to estimate data support for a
target example (as in Figure \ref{fig:cifar_brittleness}). Datamodels may be a
useful tool for extending such robustness analyses to the context of
state-of-the-art machine learning models.

%% file: sections/recasting.tex
We start by discussing the particularly important connection between datamodels and another
well-studied concept that has recently been applied to the machine learning
setting: influence estimators.
In particular, a recent line of work aims to compute the {\em empirical
influence} \citep{hampel2011robust} of training points $x_i$ on predictions $f(x_j)$,
i.e.,
\begin{align*}
    \text{Infl}[x_i \to x_j] \coloneqq
    \mathbb{P}\left(\text{model trained on $S$ is correct on $x_j$} \right) -
    \mathbb{P}\left(\text{model trained on $S \setminus \{x_i\}$ is correct on $x_j$} \right),
\end{align*}
where randomness is taken over the training algorithm.
Evaluating these influence functions naively requires training
$C\cdot \numtrain$ models where $\numtrain$ is again the size of
the train set and $C$ is the number of samples necessary for an accurate
empirical estimate of the probabilities above. To circumvent this prohibitive sample complexity, a recent line of
work has proposed approximation schemes for $\text{Infl}[x_i \to x_j]$. We
discuss these approximations (and their connection to our work) more generally in
~\Cref{sec:other_related}, but here we focus on a specific approximation
used by \citet{feldman2020neural} (and in a similar form, by
\citep{ghorbani2019data} and \citep{jia2019towards})\footnote{In fact,
\eqref{eq:sample_efficient_influences} is ubiquitous---e.g., in causal
inference, it is called the {\em average
treatment effect} of training on $x_i$ on the correctness of $x_j$.}:
\begin{align}
    \nonumber
    \widehat{\text{Infl}}[x_i \to x_j] =\
    &\mathbb{P}_{S \sim \mathcal{D}_S}\left(\text{model trained on $S$ is correct on $x_j$}\vert x_i \in S \right)  \\
    &- \mathbb{P}_{S \sim \mathcal{D}_S}\left(\text{model trained on $S$ is correct on $x_j$}\vert x_i \not\in S \right).
    \label{eq:sample_efficient_influences}
\end{align}
This estimator improves sample efficiency by reusing
the same set of models to compute influences between different input pairs.
More precisely, \citet{feldman2020neural} show that the size of the random
subsets trades off sample efficiency (model reuse is maximized when the subsets
are exactly half the size of the training set, since this maximizes the number
of samples available to estimate each term in \eqref{eq:sample_efficient_influences})
and accuracy with respect to the true empirical influence (which is maximized
as the subsets $S_i$ get larger).
Despite its different goal, formulation, and estimation procedure, it turns out
that we can cast the difference-of-probabilities estimator~\eqref{eq:sample_efficient_influences} above
as a rescaled datamodel (in the infinite-sample limit).
In particular, in Appendix \ref{app:ate_proof} we show:

\begin{restatable}{lemma}{atelinear}
    \label{lem:atelinear}
    Fix a training set $S$ of size $n$, and a test example $x$.
    For $i \in [m]$, let $S_i$ be a random variable denoting a random
    50\%-subset of the the training set $S$. Let $\bm{w}_{infl} \in
    \mathbb{R}^n$ be the estimated
    empirical influences \eqref{eq:sample_efficient_influences} onto $x$
    estimated using the sets $S_i$. Let $\bm{w}_{OLS}$ be the least-squares
    estimator of whether a particular model will get image $x$ correct, i.e.,
    \[
        \bm{w}_{OLS} \coloneqq \arg\min_{w} \frac{1}{m} \sum_{i=1}^m \lr{w^\top \bm{z_i} - \bm{1}\{\text{model trained on $S_i$ correct on $x$}\}}^2,
        \qquad
        \text{ where } \bm{z}_i = 2\cdot \mask{S_i} - \bm{1}_n.
    \]
    Then, as $m \to \infty$, $$\left\|(1 + 2/n)\bm{w}_{OLS} -
    \frac{1}{2}\bm{w}_{infl}\right\|_2 \to 0.$$
\end{restatable}
We illustrate this result quantitatively in Appendix \ref{app:ate_compare} and
perform an in-depth study of influence estimators as datamodels.
As one might expect given their different goal, influence estimates significantly underperform explicit
datamodels in terms of predicting model outputs with respect to every metric we
studied (\Cref{tab:ate_compare}, \Cref{fig:ate_compare_line_graphs}).
We then attempt to explain this performance gap and reconcile it with Lemma
\ref{lem:atelinear} in terms of the estimation algorithm (OLS vs. LASSO), scale (number
of models trained), and output function (0/1 loss vs. margins).

In addition to forging a connection between datamodels and
influence estimates, this result also provides an alternate perspective on the
parameter $\alpha$. %
Specifically, in light of our discussion in \Cref{sec:alpha_choice}, it suggests that $\alpha$ may control the {\em kinds} of correlations that are
surfaced by empirical influence estimates.

%% file: sections/future_work.tex
Our instantiation of the datamodeling framework yields
both good predictors of model behavior and a variety of direct applications.
However, this instantiation is fairly basic and thus leaves significant room for improvement along
several axes. More broadly, datamodeling provides a lens under which we can
study a variety of questions not addressed in this work.
In this section, we identify (a subset of) these questions and provide
connections to existing lines of work on them across machine learning and statistics.

\subsection{Improving datamodel estimation}
In Section \ref{sec:constructing}, we outlined our basic procedure for fitting
datamodels: we first sample subsets uniformly at random, then fit a sparse
linear model from (the characteristic vectors of) training subsets to model outputs (margins) via $\ell_1$
regularization. We first discuss various ways in which this paradigm might be
improved to yield even better predictions.

\begin{itemize}
\item \textbf{Correlation-aware estimation.} One key feature of our estimation
methodology is that the same set of models is used to estimate datamodel
parameters for an entire test set of images at once.
This significantly reduces the sample complexity of estimating datamodels but
also introduces a correlation between the errors in the estimated parameters.
This correlation is driven by the fact that model outputs are not i.i.d. across
inputs---for example, if on a picture of a dog $x$ a given model has very large
output (compared to the ``average'' model, i.e., if
$\modeleval{x}{S_i} - \mathbb{E}[\modeleval{x}{S_i}]$ is large), the model
is also more likely to have large output on another picture of a dog (as opposed to,
 e.g., a picture of a cat).

Parameter estimation in the presence of such correlated
outputs
is an active area of
research in statistics (see \citep{daskalakis2019regression, li2019prediction}
and references therein).
Applying the corresponding techniques (or modifications thereof) to
datamodels may help calibrate predictions and improve sample-efficiency.

\item \textbf{Confidence intervals for datamodels.} In this work we have
focused on attaining point estimates for datamodel parameters via simple linear
regression. A natural extension to these results would be to obtain {\em
confidence intervals} around the datamodel weights. These could, for example,
(a) provide interval estimates for model outputs rather than simple point
estimates; and (b) decide if a training input is indeed a ``significant''
predictor for a given test input.

\item \textbf{Post-selection inference.} Relatedly, the high
input-dimensionality of our estimation problem and the sparse nature of the
solutions suggests that a {\em two-stage} procedure might
improve sample efficiency. In such procedures, one first selects (often
automatically, e.g., via LASSO) a subset of the coefficients deemed to be
``significant'' for a given test example, then re-fits a linear model for {\em
only} these coefficients. This two-stage approach is particularly attractive in
settings where the number of subset-output pairs $(S_i, \modeleval{x}{S_i})$ is
less than the size of the training set $|S|$ being subsampled.

Unfortunately, using the data itself to perform model selection in this
manner---a paradigm known as {\em post-selection inference}---violates the
assumptions of classical statistical inference (in particular, that the model
class is chosen independently of the data) and can result in significantly
miscalibrated confidence intervals.
Applying {\em valid} two-stage estimation to
datamodeling would be an area for further improvement upon the protocol presented
in our work.

\item \textbf{Improving subset sampling.} Recall (cf. \Cref{sec:constructing}) that our framework
uses a distribution over subsets $\mathcal{D}_S$ to generate the ``datamodel
training set.'' In this paper, we fixed $\mathcal{D}_S$ to be random
$\alpha$-subsets of the training set, and used a nearest-neighbors example (see \Cref{fig:2d-sim}) to
provide intuition around the role of $\alpha$.
While this design choice did yield useful datamodels, it is unclear whether this
class of distributions is optimal.
In particular, a long line of literature in causal inference focuses on
{\em intervention design} \citep{eberhardt2007interventions}; drawing upon this line of work may lead
to a better choice of subsampling distribution.
Furthermore, one might even go beyond a fixed distribution $\mathcal{D}_S$ and
instead choose subsets $S_i$ {\em adaptively} (i.e., based on the datamodels estimated with the previously sampled subsets) in order to reduce sample complexity.

\item \textbf{Devising better priors.} Finally, in this paper we employed simple
least-squares regression with $\ell_1$ regularization (tuned through a held-out
validation set).
While the advantage of this rather simple prior---namely, that datamodels are {\em sparse}---is that the resulting estimation
methodology is largely data-driven, one may consider incorporating
domain-specific knowledge to design better priors.
For instance, one can use structured-sparsity \citep{huang2011learning} to take advantage of any additional structure.
\end{itemize}

\subsection{Studying generalization}
Datamodels also present an opportunity to study generalization more
broadly:

\begin{itemize}
\item \textbf{Understanding linearity.} The key simplifying assumption behind
our instantiation of the datamodeling framework is that we can approximate the
final output of training a model on a subset of the trainset as a {\em linear}
function of the presence of each training point.
While this assumption certainly leads to a simple estimation procedure, we have
very little justification for why such a linear model should
be able to capture the complexities of end-to-end model training on data
subsets.
However, we find that datamodels {\em can} accurately predict
ground-truth model outputs (cf. Sections \ref{sec:constructing}). In fact,
we find a tight {\em linear} correlation between datamodel predictions and
model outputs even on out-of-distribution (i.e., not in the support of
$\mathcal{D}_S$) counterfactual datasets. Understanding {\em
why} a simple linearity assumption leads to effective datamodels for deep neural
networks is an interesting open question. Tackling this question may necessitate a better understanding of the
training dynamics and implicit biases behind overparameterized training \citep{bartlett2021deep,sagawa2020investigation}.

\item \textbf{Using sparsity to study generalization.} A recent line of work in
machine learning studies the interplay between learning, overparameterization,
and memorization \citep{feldman2019does,chatterjee2018learning,zhang2016understanding,bresler2020corrective, zhang2020identity}.
Datamodeling may be a helpful tool in this pursuit, as it connects
predictions of machine learning models directly to the data used to train
them.
For example, the {\em data support} introduced in Section \ref{subsec:brittleness}
provides a quantitative measure of ``how memorized'' a given test input is.

\item \textbf{Theoretical characterization of the role of $\alpha$.}
In line with our intuitions in \Cref{sec:alpha_choice}, we have observed both
qualitatively (e.g., Figure \ref{fig:diff_alphas}) and quantitatively (e.g.,
Appendix \ref{app:simulation})
that estimating datamodels using different values of $\alpha$ identifies
correlations at varying granularities.
However, despite empirical results around the clear role of
$\alpha$---Appendix \ref{app:simulation} even isolates its effect on
datamodels for simple underdetermined linear regression---we lack a crisp {\em
theoretical} understanding of how $\alpha$ affects our estimated datamodels.
A better theoretical understanding of the role of $\alpha$, even for simple
models trained on structured distributions, can provide us with more
rigorous intuition for the phenomena observed here, and can in turn guide the
development of better choices of sampling distribution for datamodeling.
\end{itemize}

\subsection{Applying datamodels}
Finally, each of the presented perspectives in Section \ref{sec:applications}
can be taken further to enable even better data and model understanding. For example:

\begin{itemize}
    \item \textbf{Interpreting predictions.} For a given test example, the
    training images corresponding to the largest-magnitude datamodel weights
    both (a) share features in common with the test example; and (b) seem to be
    causally linked to the test example (in the sense that removing the training
    images flips the test prediction). This immediately suggests the potential
    utility of datamodels as a tool for {\em interpreting} test-time predictions
    in a counterfactual-centric manner. Establishing them as such requires
    further evaluation through, for example, human-in-the-loop studies.
    \item \textbf{Building data exploration tools.} In a similar vein, another
    opportunity for future work is in building user-friendly {\em data
    exploration} tools that leverage datamodel embeddings. In this paper we
    present the simplest such example in the form of PCA, but leave the vast
    field of data bias and feature discovery methods (cf.
    \cite{carter2019activation} and \citet{leclerc20213db} for a survey)
    unexplored.

\end{itemize}

%% file: appendices/datamodel_pseudocode.tex
\section{Pseudocode for Estimating Datamodels}
\label{app:pseudocode}

\input{paper_figs/datamodel_algorithm}
\input{paper_figs/causality_algorithm}

%% file: paper_figs/datamodel_algorithm.tex
\begin{algorithm}
    \caption{An outline of the datamodeling framework: we use a simple
    parametric model as a proxy for the entire end-to-end training process.}
    \label{alg:datamodel_pseudo}
    \begin{algorithmic}[1]
        \Procedure{EstimateDatamodel}{target example $x$, trainset $S$ of size $d$, subsampling frac. $\alpha \in (0,1)$}
        \State $T \gets []$ \Comment{Initialize {\em datamodel training set}}
        \For{$i \in \{1,\ldots,m\}$}
        \State Sample a subset $S_i \subset S$ from $\mathcal{D}_S$ where $|S_i| = \alpha \cdot d$
        \State $y_i \gets \modeleval{x}{S_i}$
        \Comment{Train a model on $S_i$ using $\mathcal{A}$, evaluate on $x$}
        \State Define $\mask{S_i} \in \{0, 1\}^d$ as
        $(\mask{S_i})_j = 1$ if $x_j \in S_i$ else 0
        \State $T \gets T + [(\mask{S_i}, y_i)]$ \Comment{Update datamodel
        training set}
        \EndFor
        \State $\theta \gets$ \Call{RunRegression}{T} \Comment{Predict the
        $y_i$ from the $\mask{S_i}$ vectors}
        \State\Return $\theta$ \Comment{Result: a weight vector $\theta \in
        \mathbb{R}^d$}
        \EndProcedure
    \end{algorithmic}
\end{algorithm}

%% file: paper_figs/causality_algorithm.tex
\begin{algorithm}
    \caption{Assessing datamodels' ability to predict dataset counterfactuals.}
    \label{alg:causal_pseudo}
    \begin{algorithmic}[1]
        \Procedure{CounterfactualEval}{target example $x$, datamodel
        $\theta$, trainset of size $d$}
        \State $M_0 \gets$ \textsc{Average}[$\modeleval{x}{[d]}$ \textbf{ for }
        $i \gets 1 \ldots 100$]\Comment{Average
        margin on $x$ with full training set}
        \For{$k \in \{20, 40, 80, 160, 320\}$}
        \State $G \gets$ \Call{Top-K}{$\theta, k$}
        \Comment{Get the top-$k$ indices of $\theta$}
        \State $M \gets$ \textsc{Average}[$\modeleval{x}{[d] \setminus G}$ \textbf{ for } $i \gets 1 \ldots 20$]
        \Comment{Average margin without $G$ over 20 trials}
        \State $\Delta_{avg} \gets M_0 - M$
        \Comment{Actual counterfactual}
        \State $\widehat{\Delta} \gets \theta^\top \mask{G} \lr{= \sum_{i \in G} \theta_i}$ \Comment{Datamodel-predicted counterfactual}
        \EndFor
        \State Compare all values of $\widehat{\Delta}$ and $\Delta_{avg}$
        \EndProcedure
    \end{algorithmic}
\end{algorithm}

%% file: appendices/numerical_sim.tex
\section{Understanding the Role of \texorpdfstring{$\alpha$}{} through Simulation}
\label{app:simulation}

At a high level, our intuition for the subsampling fraction\footnote{See \Cref{sec:constructing} for definition.} $\alpha$ is that
datamodels estimated with higher $\alpha$ tend to detect more {\em local} effects (i.e., those driven by smaller groups of examples, such as near-duplicates or small subpopulations),
while those estimated with lower $\alpha$ detect more {\em global} effects (i.e., those driven by larger groups of images, such as large subpopulations or subclass biases).
To solidify and corroborate this intuition about $\alpha$, we analyze a basic simulated setting.%

\paragraph{Setup.}
We consider an underdetermined linear regression model operating on
$n$ data points with $d$ binary features, i.e., $x_i \in \{0, 1\}^d, y_i \in \mathbb{R}$.
Let $X \in \mathbb{R}^{n \times d}$ and $y \in \mathbb{R}^n$ denote their matrix and vector counterparts.
$S$ is training set consisting of these $n$ samples, and we use an equally sized held-out set $S_V$ for evaluation.

The feature coordinates are distributed as Bernoulli variables of varying
frequency:
\begin{align}
    \label{eq:sim_data_model}
    x_{ik} &\sim \text{Bernoulli}(p_k)  \text{ for } 1 \leq i \leq n, \\
    \nonumber
    \text{where } p_k &\in \left\{\frac{1}{10}, \frac{2}{10}, \frac{3}{10}, \frac{4}{10},
        \frac{5}{10}\right\} \text{ for } 1 \leq k \leq d.
\end{align}
Each feature $k \in [d]$ naturally defines a {\em subpopulation} $S_k$, the group
of training examples with feature $k$ active, i.e., $S_k \coloneqq \{x_i \in S: x_{ik} = 1\}$. Features with {\em lower} (resp. {\em higher}) frequency $p_k$ are intended to capture more {\em local} (resp. more {\em global}) effects.

    The observed labels are generated according to a linear model
$y \coloneqq X\bm{w} + \mathcal{N}(0, \epsilon),$
where $\bm{w}$ is the true parameter vector and $\epsilon > 0$ is a
constant. We generate samples with $d = 150, n = 125$ and use linear regression\footnote{As the system is underdetermined, we use the pseudoinverse of $X$ to find the solution with the smallest norm.} to estimate $\bm{w}$.

Now, to use datamodels to analyze the above ``training process'' of fitting a linear regression model, we will model the output function $\modeleval{}{}$ given by the prediction of the linear model at point $x_j$ when $w$ is estimated with samples $S \subset S$, e.g.
\begin{align}
    \modeleval{x_j}{S} = (X_{S}^\top(X_{S}X_{S}^\top)^{-1} y_{S}) \cdot x_j
    \label{eq:linear_output}
\end{align}
We generate  $m = 1,000,000$ subsampled training subsets\footnote{Large sample size make sampling error negligible.} along with their evaluations, and use ordinary least squares (OLS) to fit the datamodels. (Note that the use of OLS here is separate from the use of linear regression above as the original model class.)

\paragraph{Analysis.}
Our hypothesis is that datamodels estimated with lower (resp. higher) $\alpha$ are better at detecting the effect of features of higher (resp. lower) frequency.
To test this,
we estimate datamodels for the entire test set (stacking them into a matrix $\bm{\Theta} \in \mathbb{R}^{n\times n}$, where $\bm{\Theta}_{\cdot,j}$ is the datamodel for $x_j$)
. We do this for varying values of $\alpha \in (0, 1)$, and evaluate how well each set of datamodels predicts the effects of features across different frequencies.
First, to evaluate a datamodel on some feature $k$, we can compare the following two quantities for different test examples $x_j \in S_V$:
\begin{itemize}
    \item[(a)] The {\em actual} effect of removing the subpopulation $S_k$ on $x_j$, i.e., $\modeleval{x_j}{S} - \modeleval{x_j}{S \setminus S_k}$,
    \item[(b)] The {\em datamodel-predicted} effect of removing $S_k$, i.e.,
    $\sum_{x_i \in S} \bm{\Theta}_{ij} \cdot \ind{x_i \in S_k}$.
\end{itemize}
To quantify the predictiveness of the datamodel at frequency $p$,  we compute the {\em Pearson correlation} between the above two quantities over all features $k$ with frequency $p$ and all test examples; see \Cref{alg:simulation_pseudo} for a pseudocode.
We repeat this evaluation varying $p$ and the datamodel (varying $\alpha$).
According to our intuition, for features $k$ with lower (resp. higher) frequency $p_k$, this correlation should be maximized at higher (resp. lower) values of $\alpha$, where the datamodels capture more local (resp. global) effects.
Figure \ref{fig:sim_correlations} accurately reflects this intuition: more local (i.e., less frequent) features are best detected at higher $\alpha$.

\begin{algorithm}
    \caption{Evaluating datamodel's counterfactual predictiveness for features at a particular frequency.}
    \label{alg:simulation_pseudo}
    \begin{algorithmic}[1]
        \Function{FeatureCorrelation}{datamodel matrix $\bm{\Theta}$, feature frequency $p$}
        \For{$j \gets 1 \ldots d$ if $p_j = p$} \Comment{For all features with frequency $p$}
        \For{$x_j \in S_V$} \Comment{For all test examples}
        \State $\Delta \gets \modeleval{x_j}{S} - \modeleval{x_j}{S \setminus S_k}$
        \Comment{Actual counterfactual where $\modeleval{}{}$ is given by \eqref{eq:linear_output}}
        \State $\widehat{\Delta} \gets \sum_{x_i \in S} \bm{\Theta}_{ij} \cdot \ind{x_i \in S_k}$ \Comment{Datamodel-predicted counterfactual}
        \EndFor
        \EndFor \\
        \Return Pearson correlation between $\widehat{\Delta}$ and $\Delta_{avg}$
        \Comment{Across all features and test examples}
        \EndFunction
    \end{algorithmic}
\end{algorithm}

\begin{figure}[h]
    \centering
\includegraphics[width=0.7\textwidth]{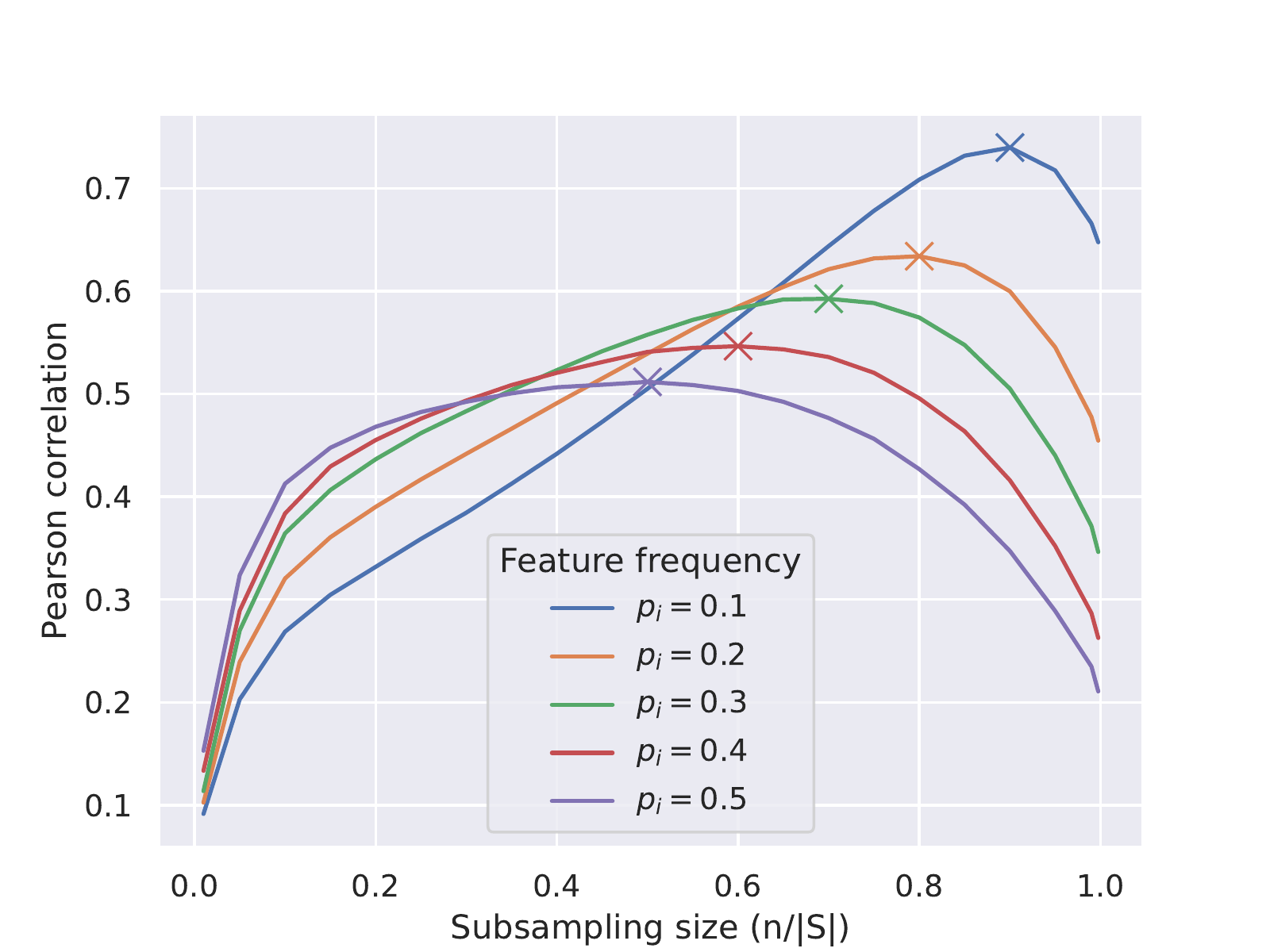}
\caption{{\bf The role of the subsampling fraction $\alpha$ in a simulated linear model.} The data consists of $d$-dimensional binary vectors $x_i$, which comprise (overlapping) {\em subpopulations} $S_k$ defined by a shared feature $k$, and their corresponding labels are generated according to a linear model, i.e. $y \coloneqq X\bm{w} + \mathcal{N}(0, \epsilon)$.
We estimate datamodels using various $\alpha$, and measure their ability to detect features at different frequencies $p$.
To quantify latter, we compute the {\em Pearson correlation} between i) the {\em actual effect} of removing the subpopulation $S_k$ on a test example and ii) the {\em datamodel-predicted} effect, across all features with frequency $p$.
Each line in the above plot corresponds to features of a particular frequency $p$, and shows the correlation ($y$-axis) while varying the datamodel ($\alpha$, $x$-axis).
Consistent with our intuition, we observe that higher (resp. lower) values of $\alpha$ are better at detecting less (resp. more) frequent features, i.e. more local (resp. global) effects.}
\label{fig:sim_correlations}
\end{figure}

%% file: appendices/choosing_margins.tex
\section{Selecting Output Function to Model}
\label{app:output_fn}
In this section, we outline a heuristic method for selecting the output
function $\modeleval{x}{S}$ to model. The heuristic is neither sufficient {\em nor}
necessary for least-squares regression to work, but may provide some signal as to
which output may yield better datamodels.

The first problem we would like to avoid is ``output saturation,'' i.e.,
being unable to learn a good datamodel due to insufficient variation in the
output. This effect is most pronounced when we measure model correctness:
indeed, over 30\% of the CIFAR-10 test set is either always correct or always
incorrect over all models trained, making datamodel estimation impossible.
However, this issue is not unique to correctness. We propose a very simple test
inspired by the idealized ordinary least squares model to measure how normally
distributed a given type of model output is.

\paragraph{Normally distributed residuals.} In the idealized ordinary least squares model,
the observed outputs $\modeleval{x}{S}$ would follow a normal
distribution with fixed mean $(\theta^{\star})^\top \mask{S}$ and unknown variance, where $\theta^{\star}$ is the true parameter vector.
Although we cannot guarantee this condition, we can measure the
``normality'' of the outputs (again, for a single fixed subset), with the
intuition that the more normal the observed outputs are, the
better a least-squares regression will work.
Hence, compare different output functions by estimating the noise distribution of datamodels
given each choice of output function.
We leverage our ability---in contrast to typical settings for regression---
to sample multiple response variables $\modeleval{x}{S}$ for a fixed
$S$ (by retraining several models on the same data and recording the output on
a fixed test example).

In Figure~\ref{fig:heuristic_output_fn}, we show the results of normality test for residuals arising from different choices of $\modeleval{\cdot}{S}$: correctness
function, confidence on the correct class, cross-entropy loss, and finally
correct-class margin\footnote{Correct-class margin is the difference between the
correct-class logit and the highest incorrect-class logit; it is unbounded by definition, and its sign indicates the correctness of the classification.}.
Correct-class margins is the only choice of $\modeleval{\cdot}{S}$ where
the $p$-values are distributed nearly uniformly, which is consistent with the
outputs being normally distributed.
Hence, we choose to use the correct-class margins
as the dependent variable for fitting our datamodels.

\begin{figure}[h]
    \includegraphics[width=\textwidth]{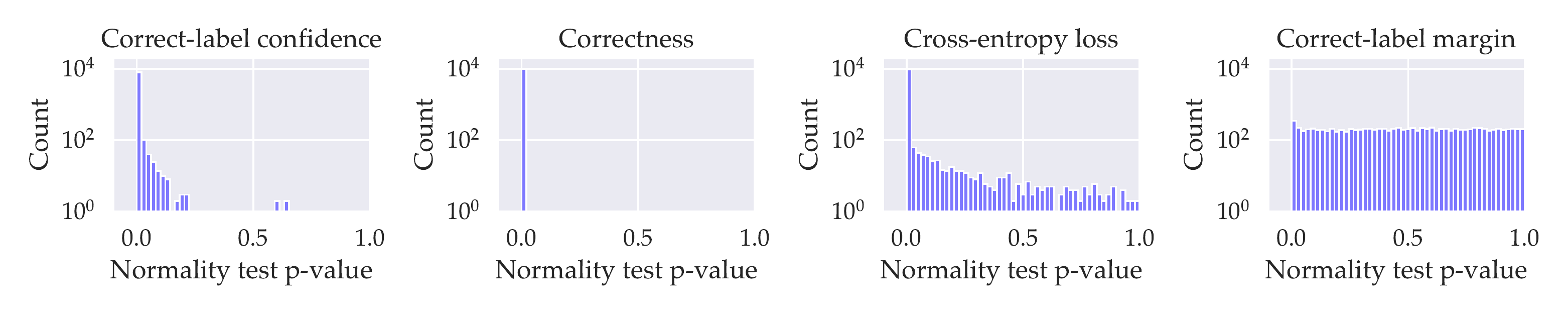}
    \caption{{\bf Correct-class margins are close to normally distributed.} For each choice of output function $f_{\mathcal{A}}$: we (a)
    fix a random subset $S' \sim  \mathcal{D}_{S}$, where $S$
    is the CIFAR-10 train set, (b) train 200 models on $S'$ and evaluate them on the
entire CIFAR-10 test set, and (c) for each image $x_i$ in the test set, calculate a $p$-value for rejecting the normality of $\modeleval{x_i}{S'}$. We plot a histogram of these
    $p$-values above. For every output function other than correct-class
    margin, almost every test is rejected, whereas for margins the
    distribution of $p$-values is uniform across $[0, 1]$, which is
    consistent with the null hypothesis (normality).}
    \label{fig:heuristic_output_fn}
\end{figure}

%% file: appendices/experimental_setup.tex
\section{Experimental Setup}
\subsection{Datasets}
\label{app:datasets}
\paragraph{CIFAR-10.} We use the standard CIFAR-10 dataset \citep{krizhevsky2009learning}.

\paragraph{\fmow.}
\fmow{} \citep{christie2018functional} is a land use classification dataset based on satellite imagery.
WILDS \citep{koh2020wilds} uses a subset of \fmow{} and repurposes it as a benchmark for out-of-distribution (OOD) generalization; we use same the variant (presized to 224x224, single RGB image per example rather than a time sequence). We perform our analysis only on the in-distribution train/test splits (e.g. overlapping years) as our focus is not on OOD settings.
Also, we limit our data to the year 2012.
(These restrictions are only for convenience, and our framework can easily extend and scale to more general settings.)

Properties of both datasets are summarized in \Cref{table:dataset}.

\begin{table}[h]
\centering
\caption{Properties of datasets used.}
\label{table:dataset}
\begin{tabular}{lrrrrrr}
\toprule
\textbf{Dataset} & Classes & Size (Train/Test) & Input Dimensions \\
\midrule
CIFAR-10 &  10 & 50,000/10,000 & $3 \times 32\times 32$ \\
\fmow & 62 & 21,404/3,138 & $3 \times 224\times 224$ \\
\bottomrule
\end{tabular}
\end{table}

\subsection{Models and hyperparameters}
\label{app:hyperparams}
\paragraph{CIFAR-10.}
We use a ResNet-9 variant from Kakao Brain\footnote{\url{https://github.com/wbaek/torchskeleton/blob/master/bin/dawnbench/cifar10.py}} optimized for fast training.
The hyperparameters (\Cref{table:hyperparams}) were chosen using a grid search.
We use the standard batch SGD.
For data augmentation, we use random 4px random crop with reflection padding, random horizontal flip, and $8\times8$ CutOut \citep{devries2017improved}.

For counterfactual experiments with ResNet-18 (\Cref{fig:rn18_causality}), we use the standard variant \citep{he2016deep}.

\paragraph{\fmow.}
We use the standard ResNet-18 architecture \citep{he2016deep}.
The hyperparameters (\Cref{table:hyperparams}) were chosen using a grid search, including over different optimizers (SGD, Adam) and learning rate schedules (step decay,
cyclic, reduce on plateau).
As in \citet{koh2020wilds}, we  do not use any data augmentation.
Unlike prior work, we do not initialize from a pre-trained ImageNet model; while this results in lower accuracy, this allows us to focus on the role of the \fmow{} dataset in isolation.

\begin{table}[h]
\centering
\caption{Hyperparameters for used model class.}
\label{table:hyperparams}
\begin{tabular}{lcccccc}
\toprule
\textbf{Dataset} & Initial LR & Batch Size & Epochs & Cyclic LR Peak Epoch & Momentum & Weight Decay \\
\midrule
CIFAR-10 &  0.5 & 512 & 24 & 5 & 0.9 & 5e-4 \\
\fmow & 0.4 & 512 & 15 & 6 & 0.9 & 1e-3 \\
\bottomrule
\end{tabular}
\end{table}

\paragraph{Performance.}
In~\Cref{table:acc}, we show for each dataset the accuracies of the chosen model class (with its specific hyperparameters), across different values of $\alpha$.

\begin{table}
    \centering
    \caption{Accuracies for our chosen model classes on CIFAR-10 and \fmow{} across varying $\alpha$.}
    \label{table:acc}
    \begin{tabular}[h]{@{}ccc@{}}
        \toprule
        & \multicolumn{2}{c}{Accuracy (\%)} \\
        Subset size ($\alpha$) & {CIFAR-10} & {\fmow} \\
        \cmidrule{1-3}
        1.0 & 93.00  & 33.76 \\
        0.75 & 91.77 & 31.16 \\
        0.5 & 89.61 & 25.97 \\
        0.2 & 81.62 & 14.70 \\
        0.1 & 71.60 & N/A \\
        \bottomrule
    \end{tabular}
\end{table}

\subsection{Training infrastructure}
\paragraph{Computing resources.}
We train our models on a cluster of machines, each with 9 NVIDIA A100 GPUs and 96 CPU cores.
We also use half-precision to increase training speed.

\paragraph{Data loading.}
We use FFCV \citep{leclerc2022ffcv}, which removes the data loading bottleneck for smaller models and allows us achieve a throughput of over 5,000 CIFAR-10 models a day \emph{per} GPU.

\paragraph{Data processing.}
Our datamodel estimation uses (the characteristic vectors) of training subsets and model outputs (margins) on train and test sets.
Hence, we do not need to store any model checkpoints, as it suffices to store the training subset and the model outputs after evaluating at the end of training.
In particular, training subsets and model outputs can be stored as  $\nummodels \times \numtest$ or $\nummodels \times \numtrain$ matrices,
with one row for each model instance and one column for each train or test example.
All subsequent computations only require the above matrices.

%% file: appendices/regression.tex
\section{Regression}

\subsection{Solver details}
\label{app:ffcv_regression}
As mentioned in Section \ref{sec:constructing}, we construct datamodels by
running $\ell_1$-regularized linear regression, predicting correct-class margins
from characteristic vectors, or {\em masks}, $\mask{S_i}$. The resulting optimization problem is rather
large: for example, estimating datamodels for $\alpha = 50\%$ requires running
LASSO with a covariate matrix $X$ of size $50,000 \times 300,000$, which
corresponds to about 60GB of data; for $\alpha = 10\%$,
datamodels this increases five folds as there are 1.5 million models.
Moreover, we need to solve up to $60,000$ regression problems (one datamodel each train / test example).
The large-scale nature of our estimation problem rules out
off-the-shelf solutions such as scikit-learn \citep{pedregosa2011scikit}, GLMNet
\citep{friedman2010regularization}, or Celer \citep{massias2018celer}, all of which either runs out of memory or
does not terminate within reasonable time.

Note that solving large linear systems efficiently is an area of active research
(\citep{martinsson2020randomized}), and as a result we anticipate that datamodel estimation
could be significantly improved by applying techniques from numerical
optimization.
In this paper, however, we take a rather simple approach based on the SAGA
algorithm of \citep{gazagnadou2019optimal}. Our starting point is the GPU-enabled implementation
of \citet{wong2021leveraging}---while this implementation terminated (unlike the
CPU-based off-the-shelf solutions), the regressions are still prohibitively
slow (i.e., on the order of several GPU-hours per single datamodel estimation).
To address this, we make the following changes:

\paragraph{Fast dataloading.} The first performance bottleneck turns out to be
in dataloading. More specifically, SAGA is a minibatch-based algorithm: at each
iteration, we have to read $B$ masks (50,000-dimensional binary vectors) and $B$
outputs (scalars) and move them onto the GPU for processing. If the masks are
read from disk, I/O speed becomes a major bottleneck---on the other hand, if we
pre-load the entire set of masks into memory, then we are not able to run
multiple regressions on the same machine, since each regression will use
essentially the entire RAM disk. To resolve this issue, we use the FFCV library
\citep{leclerc2022ffcv} for dataloading---FFCV is based on memory mapping, and thus allows
for multiple processes to read from the same memory (combining the benefits of
the two aforementioned approaches). FFCV also supports batch pre-loading and
parallelization of the data processing pipeline out-of-the-box---adapting the
SAGA solver to use FFCV cut the runtime significantly.

\paragraph{Simultaneous outputs.} Next, we leverage the fact that the SAGA
algorithm is trivially parallelizable across different instances (sharing the same input matrix), allowing us to estimate
{\em multiple} datamodels at the same time. In particular, we estimate
datamodels for the entire test set in one pass, effectively cutting the
runtime of the algorithm by the test set size (e.g., 10,000 for CIFAR-10).

\paragraph{Optimizations.} In order to parallelize across test examples, we
need to significantly reduce the GPU memory footprint of the SAGA solver. We
accomplish this through a combination of simple code optimization (e.g., using
in-place operations rather than copies) as well as writing a few custom CUDA
kernels to speed up and reduce the memory consumption of algorithms such as soft
thresholding or gradient updating.

\paragraph{Experimental details.}
For each dataset considered, we chose a maximum $\lambda$: $0.01$ for CIFAR-10 test, $0.1$ for CIFAR-10 trainset, and
$0.05$ for \fmow{} datamodels. Next, we chose
$k = 100$ logarithmically spaced intermediate values between $(\lambda/100,
\lambda)$ as the regularization path.
We ran one regression per intermediate $\lambda$, using $m - 50,000$ samples
(where $m$ is as in the table in Figure \ref{fig:estimation} (right)) to
estimate the parameters
of the model and the remaining $50,000$ samples as a validation set. For each
image in the test set, we select the $\lambda$ corresponding to the
best-performing predictor (on the heldout set) along the regularization path. We
then {\em re-run} the regression once more using these optimal $\lambda$ values
and the full set of $m$ samples.

\clearpage
\subsection{Omitted results}
\label{app:regression_extra}

\begin{figure}[ht]
    \begin{minipage}{0.55\textwidth}
        \centering
        \includegraphics[width=\textwidth]{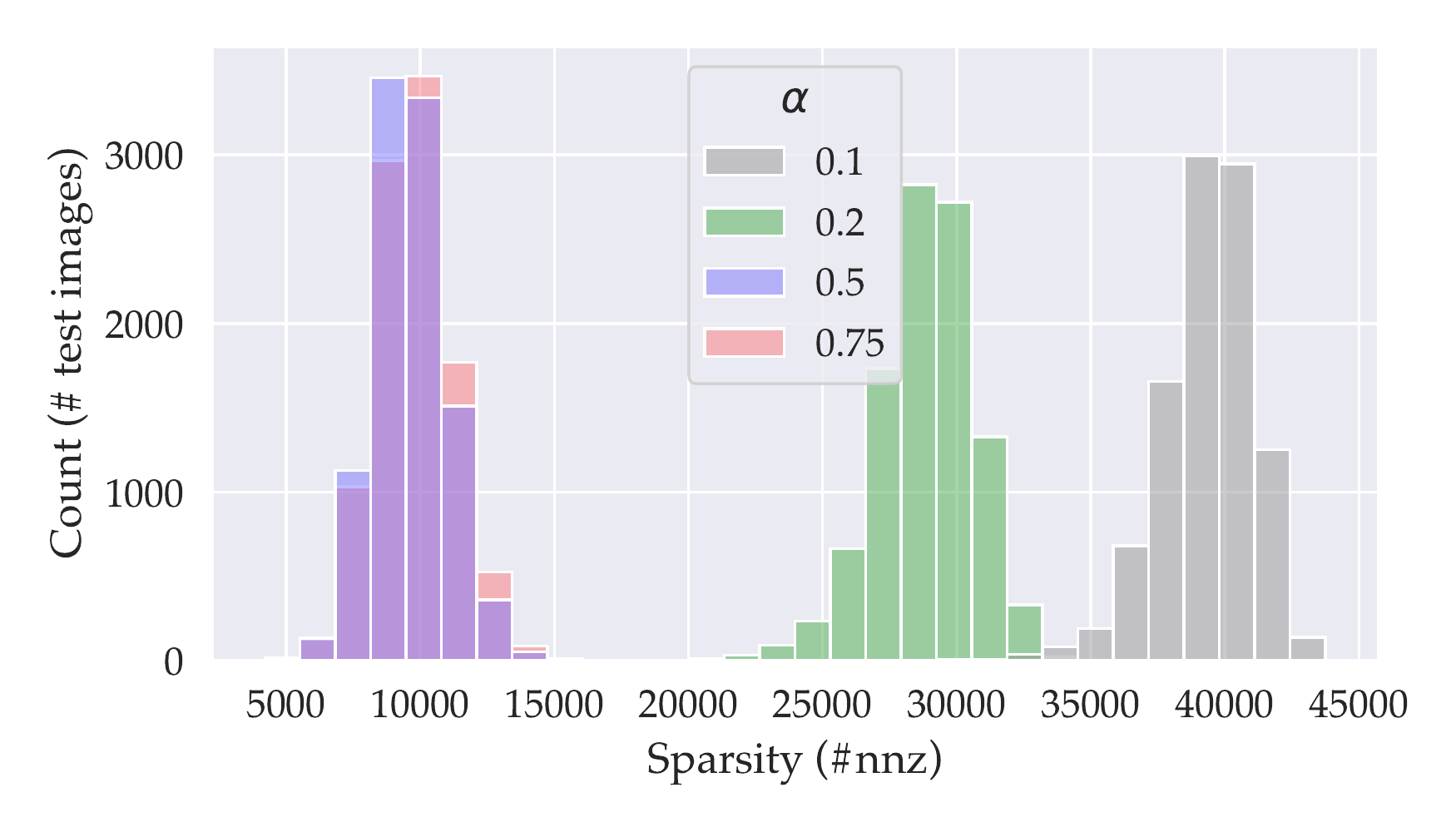}
        \caption{{\bf Sparsity distribution of different datamodels.}
        Above shows the distribution of datamodel sparsity over test examples on CIFAR-10, compared across different $\alpha$. Sparsity decreases with higher $\alpha$, which is consistent with our intuitions (\Cref{sec:alpha_choice}) that higher $\alpha$ captures relationships driven by smaller groups of images.}
        \label{fig:cifar_sparsity}
    \end{minipage}
    \hspace{1em}
    \begin{minipage}{0.35\textwidth}
        \centering
        \includegraphics[width=\textwidth,trim={16.85cm 0 0 0},clip]{paper_figs/estimation_fig.pdf}
        \caption{{\bf The role of regularization.} Average in-sample and
        out-of-sample MSE (i.e., \eqref{eq:mse}) for datamodels on CIFAR-10 estimated by
        optimizing the regularized least-squares objective
        \eqref{eq:testset_datamodels} for varying $\lambda$.}
        \label{fig:cifar_lambda_effect}
    \end{minipage}
\end{figure}

\begin{figure}[h]
    \centering
    \includegraphics[width=.8\linewidth]{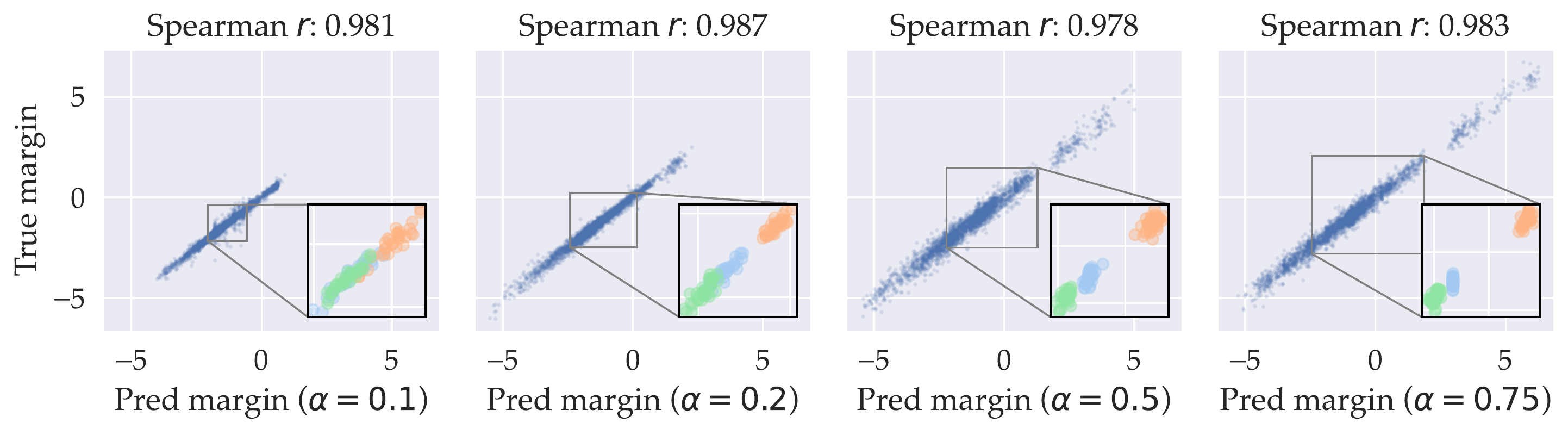}
    \caption{Identical results to Figure \ref{fig:ondistribution} for \fmow{}.}
    \label{fig:fmow_ondistribution}
\end{figure}

\begin{figure}[!h]
    \centering
    \includegraphics[width=.8\linewidth]{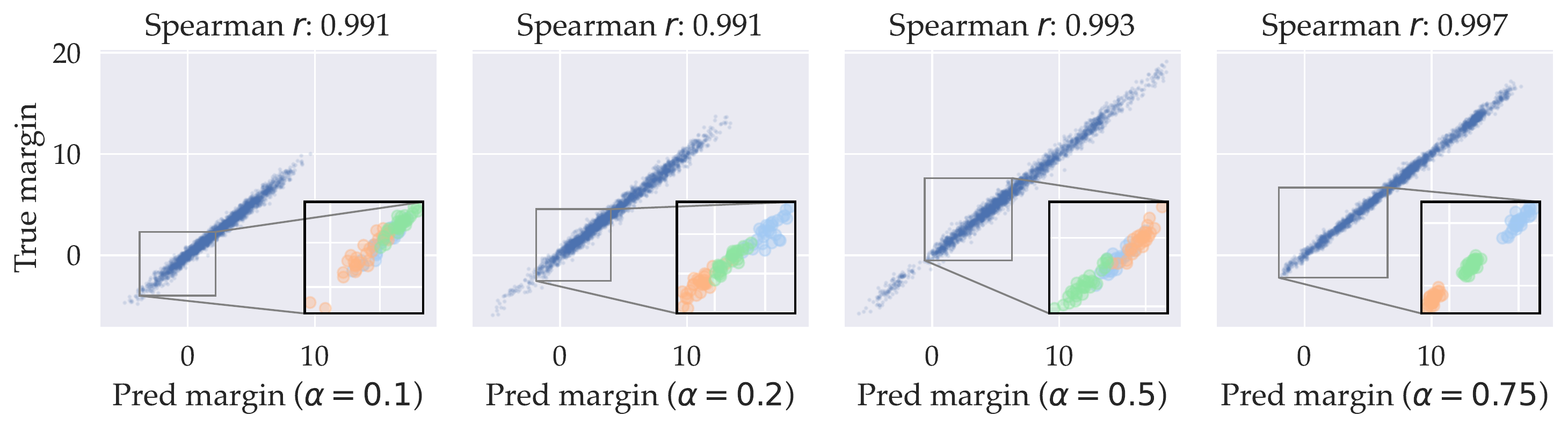}
    \caption{Identical results to Figure \ref{fig:ondistribution} for CIFAR-10 for
    all values of $\alpha$.}
    \label{fig:cifar_ondistribution_extra}
\end{figure}

%% file: appendices/causal_view_extras.tex
\section{Counterfactual Prediction with Datamodels}
\label{app:causal}

\subsection{General setup}
\label{app:causal_details}

\paragraph{Sample selection.}
For all of our counterfactual experiments, we use a random sample of the respective test datasets.
We select at random 300 test images for CIFAR-10 (class-balanced; 30 per class) and 100 test images for \fmow{}. For the CIFAR-10 baselines, we consider counterfactuals for a 100 image subset of the 300.

\paragraph{Size of counterfactuals.}
For CIFAR-10, we remove top $k=\{10, 20, 40, 80, 160, 320, 640, 1280\}$ images
and bottom $k=\{20, 40, 80, 160, 320\}$ where applicable.
For \fmow{}, we remove top and bottom $k=\{10,$ $20,$ $40,$ $80,$ $160, 320, 640\}$.

\paragraph{Reducing noise by averaging.}
Each counterfactual (i.e., training models on a given training set $S'$) is evaluated over $T$ trials to reduce the variance that arises purely from non-determinism in model training.
We use $T=20$ for CIFAR-10 and \fmow{}, and $T=10$ for CIFAR-10 baselines. In \Cref{app:varying_T}, we show that using sufficiently high $T$ is important for reducing noise.

\paragraph{Control values.}
To calculate the actual effects in all of our counterfactual evaluation, we need control values $\mathbb{E}[\modeleval{x}{S}]$ for the ``null,'' i.e, margins
when trained on 100\% of the data. We estimate this by averaging over models on trained on the full training set (10,000 for CIFAR-10 and 500 for \fmow{}).

\subsection{Baselines}
\label{app:baselines}

We describe the baseline methods used to generate data support estimates and counterfactuals. Each of the methods gives a way to select training examples that are most similar or influential to a target example.
As in prior work \citep{hanawa2021evaluation,pezeshkpour2021empirical}, we consider a representative set of baselines spanning both methods based on representation similarity and gradient-based methods, such as influence functions.

\paragraph{Representation distance.}
We use $\ell_2$ distances in the penultimate layer's representation to rank the training examples in order of similarity to the target test example.
We also evaluated dot product, cosine, and mahalanobis distances, but they did not show much variation in their counterfactual effects.\footnote{With the exception of dot product, which performs poorly due to lack of normalization; this is consistent with the findings in \citet{hanawa2021evaluation}.}

In order to more fairly compare with datamodels---so that we can disentangle the variance reduction from using many models and the additional signal captured by datamodels---we also  averaged up to 1000 models' representation distances\footnote{We simply average the ranks from each model, but there are potentially better ways to aggregate them.}, but this had no discernible difference on the size of the counterfactual effects.

\paragraph{Influence functions.}
We apply the influence function approximation introduced in \cite{koh2017understanding}. In particular, we use the following first-order approximation for the influence of $z$ on the loss $L$ evaluated at $z_{\text{test}}$:
\begin{align*}
    \mathcal{I}_{\text{up,loss}}(z, z_{\text{test}})
    &= - \nabla_{\theta} \ell (z_{\text{test}}, \hat{\theta})^\top H^{-1}_{\hat{\theta}}
    \nabla_{\theta} \ell(z,\hat{\theta})
\end{align*}
where $\hat{\theta}$ is the empirical risk minimizer on the training set and $H$ is the Hessian of the loss.
The influence here is just the dot product of gradients, weighted by the Hessian.
We approximate these influence values by using the methods in \cite{koh2017understanding} and as implemented (independently) in \texttt{pytorch-influence-functions}.\footnote{\url{https://github.com/nimarb/pytorch_influence_functions}}
As in \cite{koh2017understanding}, we take a pretrained representation (of a ResNet-9 model, same as that modeled by our datamodels), and compute approximate influence functions with respect to only the parameters in the last linear layer.

\paragraph{TracIn.}
\citet{pruthi2020estimating} define an alternative notion of influence: the influence of a training example $z$ on a test example $z'$ is the total change in loss on $z'$ contributed by updates from mini-batches containing $z$---intuitively, this measures whether gradient updates from $z$ are helpful to learning example $z'$.
They approximate this in practice with \texttt{TracInCP}, which considers checkpoints $\theta_{t_1},...,\theta_{t_k}$ across training, and sums the dot product of the gradients at $z$ and $z'$ at each checkpoint:
\[
    \texttt{TracInCP}(z, z') = \sum\limits_{i=1}^k \eta_i \nabla_\theta\ell(z, \theta_{t_i}) \cdot
    \nabla_\theta \ell(z', \theta_{t_i})
\]
One can view \texttt{TracInCP} as a variant of the gradient dot product, but averaged over models at different epochs) and weighted by the learning rate $\eta_i$.

\paragraph{Random baseline.}
We also consider a random baseline of removing examples from the same class.

\subsection{Data support estimation}
\label{app:brittleness}

\paragraph{Setup.}
We use datamodels together with counterfactual evaluations in a guided search to efficiently estimate upper bounds on the size of data supports.
    For a given target example $x$ with corresponding datamodel $g_\theta$,
    we want to find candidate training subsets of small size $k$ whose removal most reduces the classification margin on $x$:
    \begin{equation}
        \label{eq:brittleness_min}
        G_k \coloneqq \arg\min_{|G| = k} g_\theta(S \setminus G_k).
    \end{equation}
    Because $g_\theta$ is a linear model in our case, the solution to the above
    minimization problem is simply the set corresponding to the largest $k$ coordinates of the datamodel parameter $\theta$:
    \begin{equation}
        G_k = \arg\max_{G \subset S; |G|=k} \theta^\top\mask{G}
        = \text{top-$k$ indices of $\theta$.}
        \label{eq:brittleness_min_linear}
    \end{equation}
    Our goal is to the find the smallest of these subsets $\{G_k\}_k$ so that $\modeleval{x}{S \setminus G_k} < 0$, i.e., the example is misclassified on average as per our definition.\footnote{Note that $\mathbb{E}\modeleval{x}{\cdot} < 0$ does not imply that the probability of misclassification is greater than 50\%. Nonetheless, it is a natural threshold.}
    Thus, for each target $x$, we try several values of $k \in \{10$, $20$, $40$, $80$, $160$, $320$, $640$, $1280\}$, training models on the set $S \setminus G_k$ and evaluating the resulting models on $x$.\footnote{While a binary search over $k$ for each $x$ would be more sample efficient, we collect the entire grid of samples for simplicity.}
    We train $T=20$ models on each counterfactual $S \setminus G_k$ to reduce variance.

    (Given that we are using datamodels as surrogates after all, one might wonder if the above counterfactual evaluations are actually necessary---one could instead consider estimating the optimal $k$ directly from $\theta$. We revisit a heuristic estimation procedure based on this idea at the end of this subsection.)

\paragraph{Estimation methodology.}
We assume that the expected margin $h(k) \coloneqq \modeleval{x}{S \setminus G_k}$ after removing $k$ examples decreases {\em monotonically} in $k$; this is expected from the linearity of our datamodels and is further supported empirically (see \Cref{fig:causal_per_ex}).
Then, our goal is to estimate the unique zero\footnote{More precisely, the upper ceiling as data support is defined as an integer quantity.} $\hat{k}$ of the above function $h(k)$ based on (noisy) samples of $h(k)$ at our chosen values of $k$.
Note that by definition, $\hat{k}$ is an upper bound on $\supp(x)$.
Now, because of our monotonicity assumption, we can cast estimating $\hat{k}$ as instance of an isotonic regression problem \citep{robertson1988order}); this effectively performs piecewise linear interpolation, while ensuring that monotonicity constraint is not violated.
We use \texttt{sklearn}'s \texttt{IsotonicRegression} to fit an estimate $h(k)$,
and use this to estimate $\hat{k}$.

\begin{figure}[!htb]
    \centering
    \includegraphics[width=1.\linewidth]{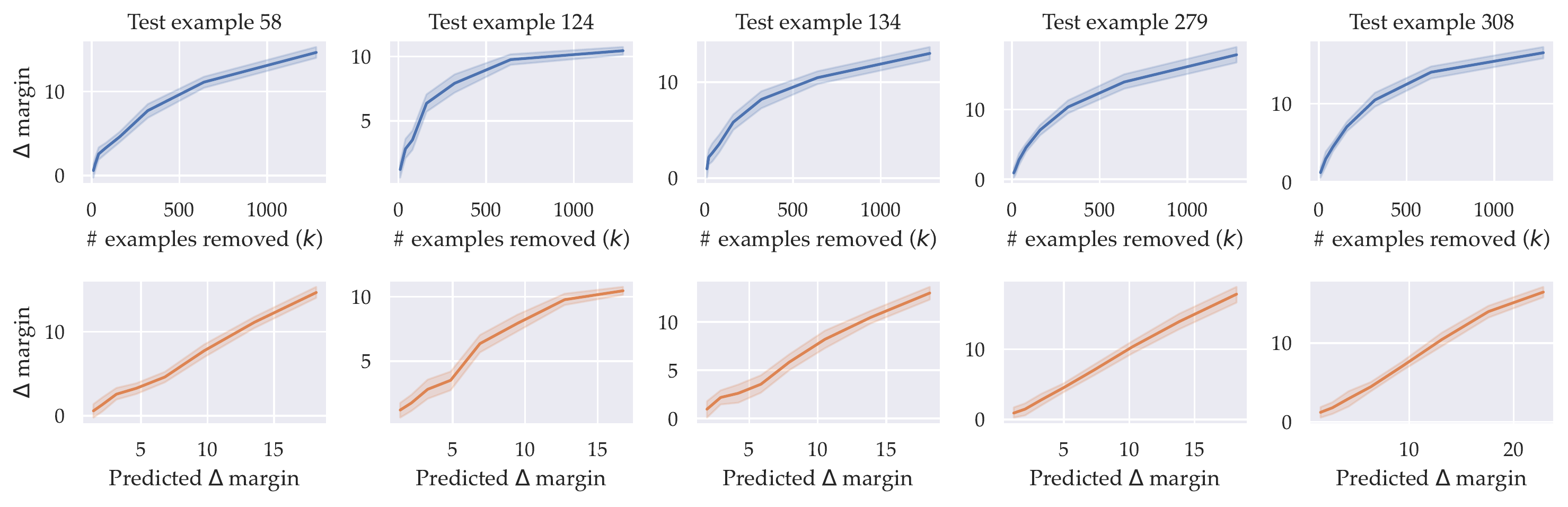}
    \caption{{\bf Counterfactuals for individual examples.}
    We plot the results of counterfactual evaluations (using $\alpha=0.5$ datamodels) for five individual examples, shown in separate columns.
    {\bf (Top)} The actual $\Delta$ margin changes monotonically with number of examples removed ($k$), corroborating the monotonicity assumption used in estimating data supports. {\bf (Bottom)} On $x$, we instead plot the {\em predicted} $\Delta$ margin using datamodels. This shows that the linearity seen in~\Cref{fig:cifar_causal} manifests even at a local level.}
    \label{fig:causal_per_ex}
\end{figure}

\paragraph{Verifying support estimates.}
Due to stochasticity in evaluating counterfactuals, the estimate $\hat{k}$ is noisy. Thus, it is possible that $\hat{k}$ is not a valid upper bound on $\supp(x)$, e.g. removing top $\hat{k}$ examples do not misclassify $x$.
In fact, removing $G_{\hat{k}}$ and re-training shows that only 67\% of the images are actually misclassified.
To establish an upper bound on $\supp(x)$ that has sufficient coverage, we
evaluate the counterfactuals after removing an additional 20\% of highest datamodel weights, e.g. removing top $\hat{k} \times 1.2$ examples
for each test example.
When an additional 20\% of training examples are removed,
92\% of test examples are misclassified.
Hence, we use $\hat{k} \times 1.2$ for our final estimates of $\supp(x)$.

\subsubsection{Estimation using baselines}
\label{app:brittleness_baselines}
As baselines, we use the same guided search algorithm described above, but instead
of using datamodel-predicted values to guide the search, we select the candidate subset using each of the baselines methods described in \Cref{app:baselines}.
In particular, we choose the candidate subset $R_k$ for a given $k$ as follows:
\begin{enumerate}
    \item {\em Representation distance}: top $k$ closest training examples to
    $x$ as measured by $\ell_2$ distance in the representation space of a
    pre-trained ResNet-9.
    \item {\em Influence estimates (influence functions and TracIn)}: top $k$ training examples with highest (most positive) estimated influence on the target example $x$:
    \item {\em Random}: first $k$ examples from a random ordering\footnote{The random ordering is fixed across different choices of $k$, but not across different targets.} of training examples from the same class as $x$.
\end{enumerate}

\subsubsection{Heuristic estimates for data support}
While we constructively estimate the data supports by training models on counterfactuals and using the above estimation procedure, we can also consider a simpler and cheaper heuristic to estimate $\supp(x)$ assuming the fidelity of the linear datamodels:
compute the smallest $k$ s.t. the sum of the $k$ highest datamodel weights for $x$ exceeds the average margin of $x$.
In~\Cref{fig:cifar_brittleness_heuristic}, we compare the predicted data supports based on this heuristic to the estimated ones from earlier, and find that they are highly correlated.
In practice, this can be a more efficient alternative to quantify brittleness without additional model training (beyond the initial ones to estimate the datamodels).

\begin{figure}[!htb]
    \begin{minipage}{0.5\linewidth}
        \centering
        \includegraphics[width=\linewidth]{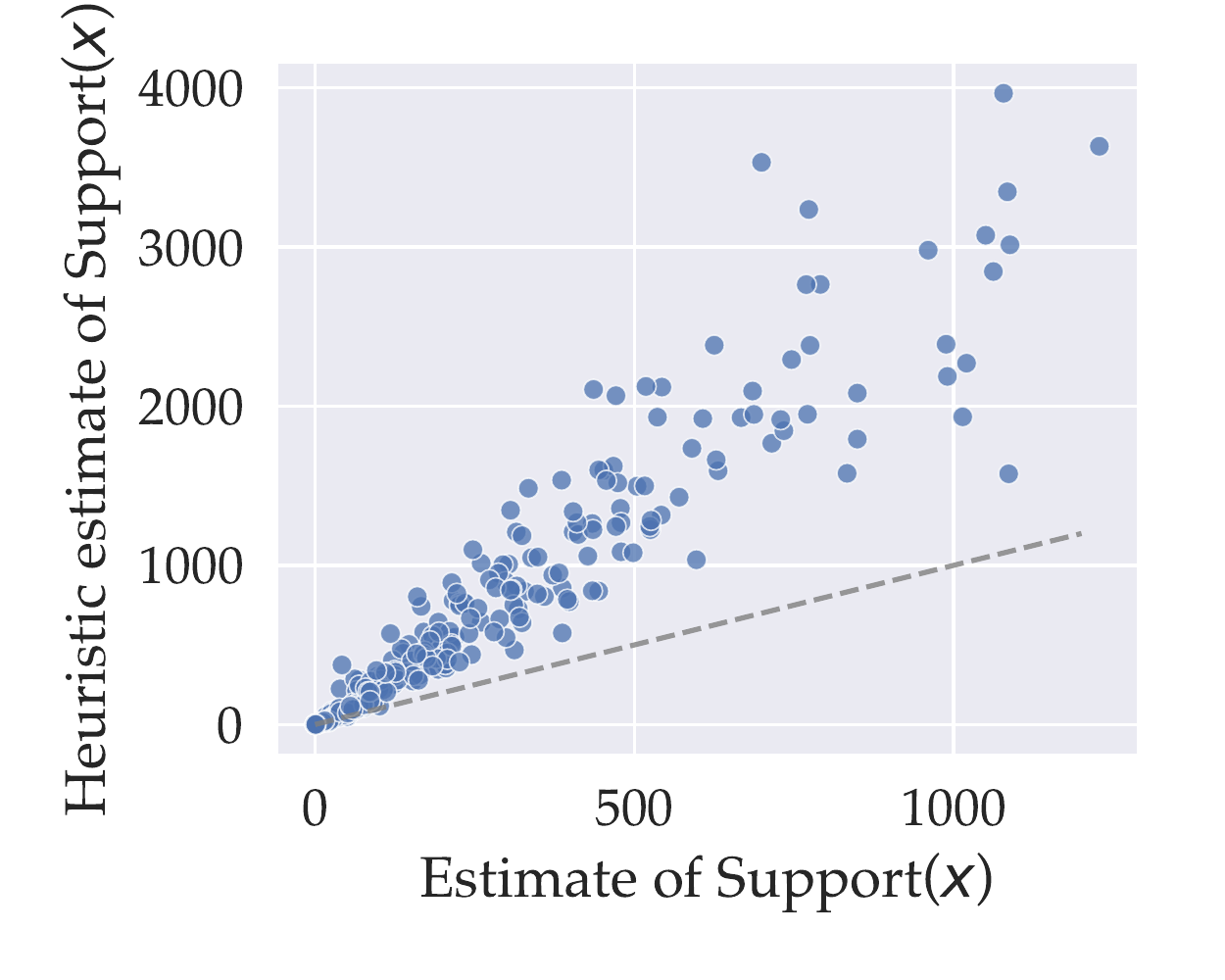}
        \caption{{\bf Heuristic predictions for data supports.}
        For each of the 300 test examples shown, the $x$-coordinate represents the previous estimates based on counterfactuals, and the $y$-coordinate represents the heuristic estimate.}
        \label{fig:cifar_brittleness_heuristic}
    \end{minipage}
    \hspace{1em}
    \begin{minipage}{0.5\linewidth}
        \centering
        \includegraphics[width=\linewidth]{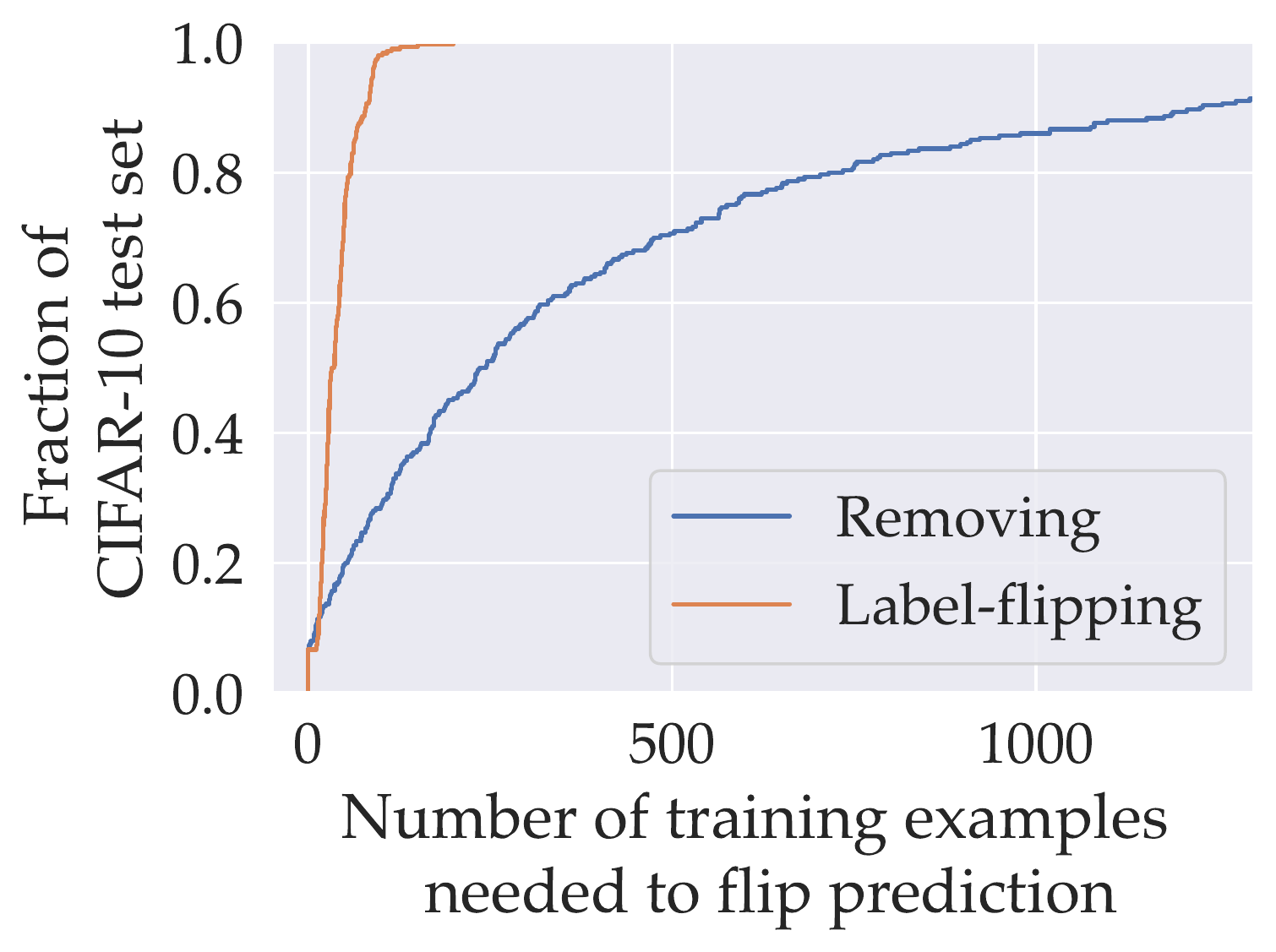}
        \caption{{\bf Brittleness to mislabeling.} We estimate an upper bound on the the smallest number of training images that can be mislabeled to flip a given target image. Much fewer images are required compared simply removing them (blue), as mislabeling provides additional signal to the model.}
        \label{fig:cifar_flippy}
    \end{minipage}
\end{figure}

\subsection{Brittleness to mislabeling}
\label{app:brittleness_flipping}

To study the brittleness of model predictions to mislabeling training images, we take the same 300 random CIFAR-10 test examples and analyze them as follows:
First, we find for each example the {\em incorrect} class with the highest average logit (across $\sim$10,000 models trained on the full training set).
Then, we construct counterfactual datasets similarly as in \Cref{app:brittleness} where we take the top $k=\{2,4,...,256\}$ training examples with the highest datamodel weights, but this time mislabel them with the incorrect class identified earlier.
After training $T=20$ models on each counterfactual, for each target example we estimate the number of mislabeled examples at which the expected margin becomes zero, using the same estimation procedure described in
\Cref{app:brittleness}.
The resulting mislabeling brittleness estimates are shown in \Cref{fig:cifar_flippy}.

\subsection{Comparing raw effect sizes}
Instead of comparing the data support estimates (which are derived quantities), here we directly compare the average counterfactual effect (i.e. delta margins) of groups selected using different methods.
\Cref{fig:cifar_baselines} shows again that datamodels identify much larger effects. Among baselines, we see that TracIn performs best, followed by representation distance. We also see that the representation baseline does not gain any additional signal from simple averaging over models.

\begin{figure}[!htb]
    \centering
    \includegraphics[width=1\linewidth]{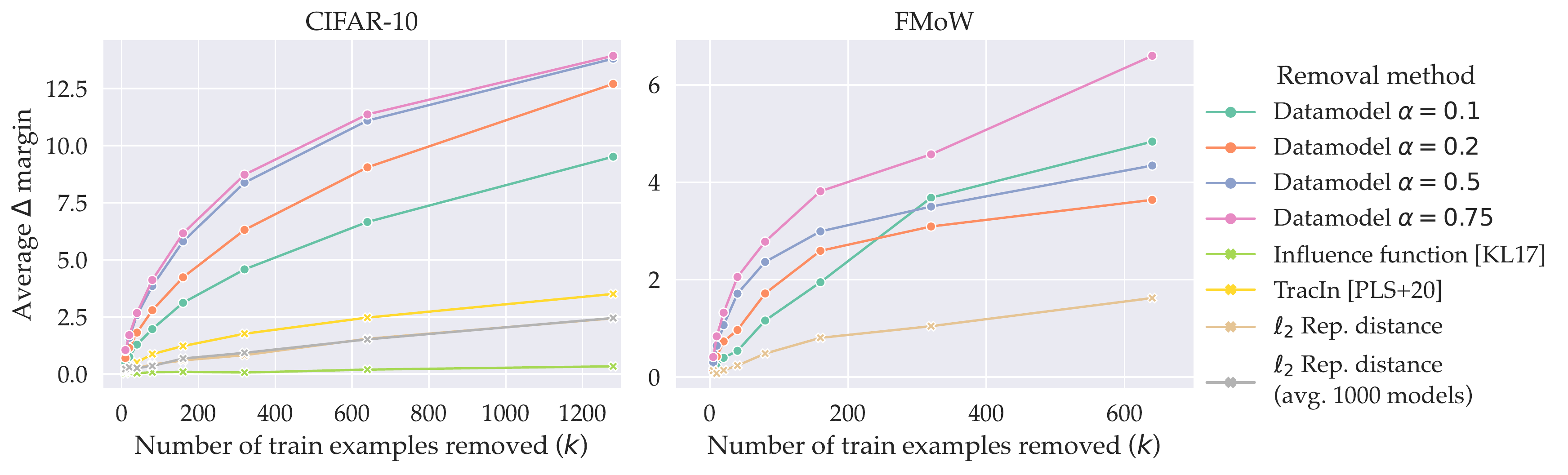}
    \caption{{\bf Comparing effect sizes with baselines}. This shows the raw evaluations of counterfactuals generated using different methods, which were also used for estimating data supports. The $y$-axis shows the effect on margin averaged across all target examples when top $k$ examples are removed for each target using each of the methods. Datamodels identify much larger effects compared to baselines. Among baselines, TracIn\citep{pruthi2020estimating} performs the best. For representation distance, there is no noticeable gain from reducing stochasticity by averaging over more models (1000 vs 1).}
    \label{fig:cifar_baselines}
\end{figure}

\subsection{Effect of training stochasticity}
\label{app:varying_T}
As described in \Cref{app:causal_details}, we re-train up to $T=20$ models for {\em each} counterfactual to reduce noise that arises solely from stochasticity of model training.
These additional samples significantly reduces unexplained variance: \Cref{fig:varying_T} shows the reduction in variance (``thickness'' in the $y$-direction) and the resulting increase in correlation as the number of re-training runs is increased from $T=1$ to $T=20$.

\begin{figure}[!htb]
    \begin{minipage}{0.60\textwidth}
        \centering
        \includegraphics[width=\textwidth]{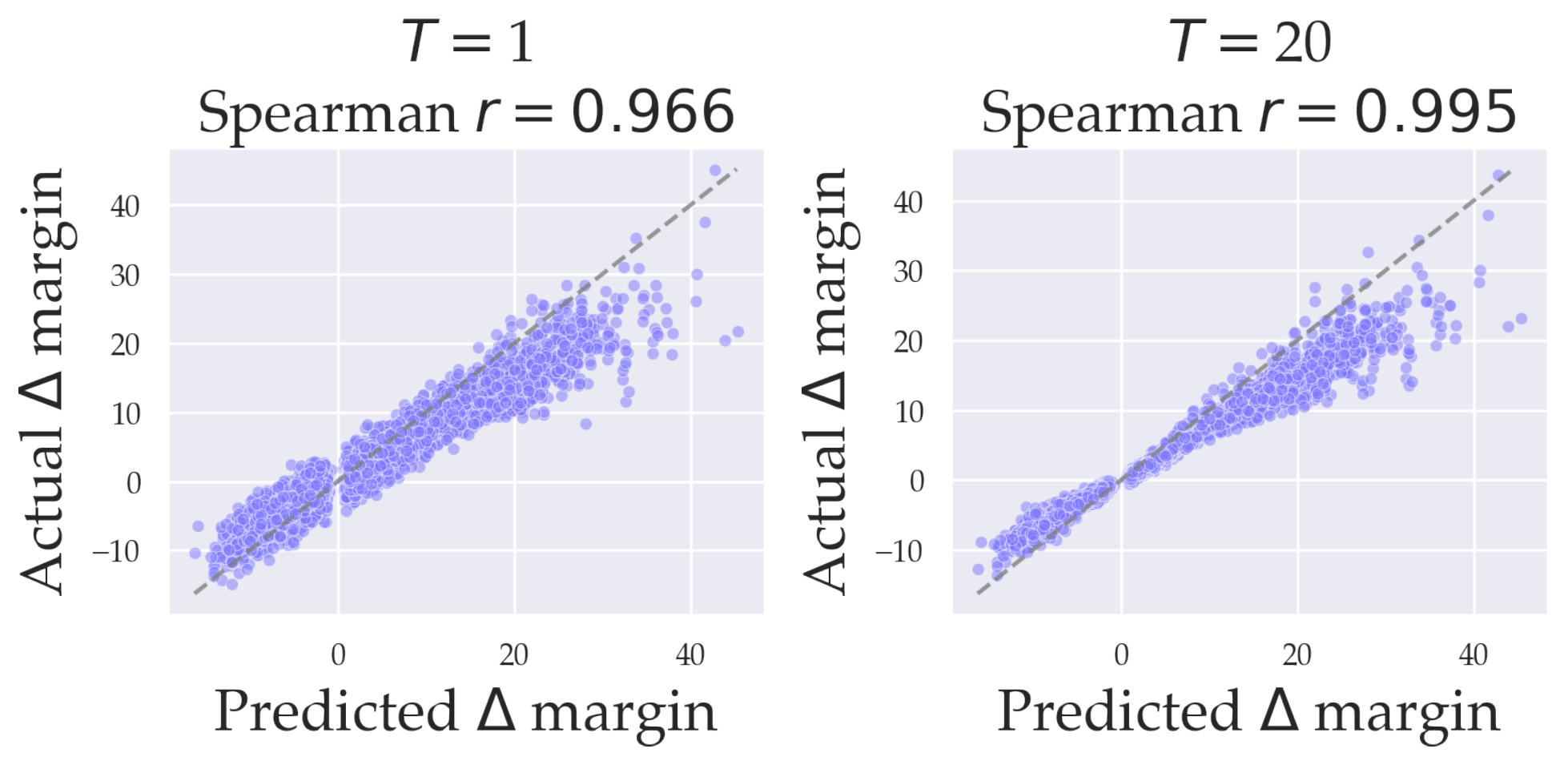}
        \caption{{\bf Effect of model averaging.} $T$ is the number of models trained per counterfactual.}
        \label{fig:varying_T}
    \end{minipage}
    \hspace{1em}
    \begin{minipage}{0.35\textwidth}
        \centering
        \includegraphics[width=\textwidth,trim={0 0 2.5cm 0},clip]{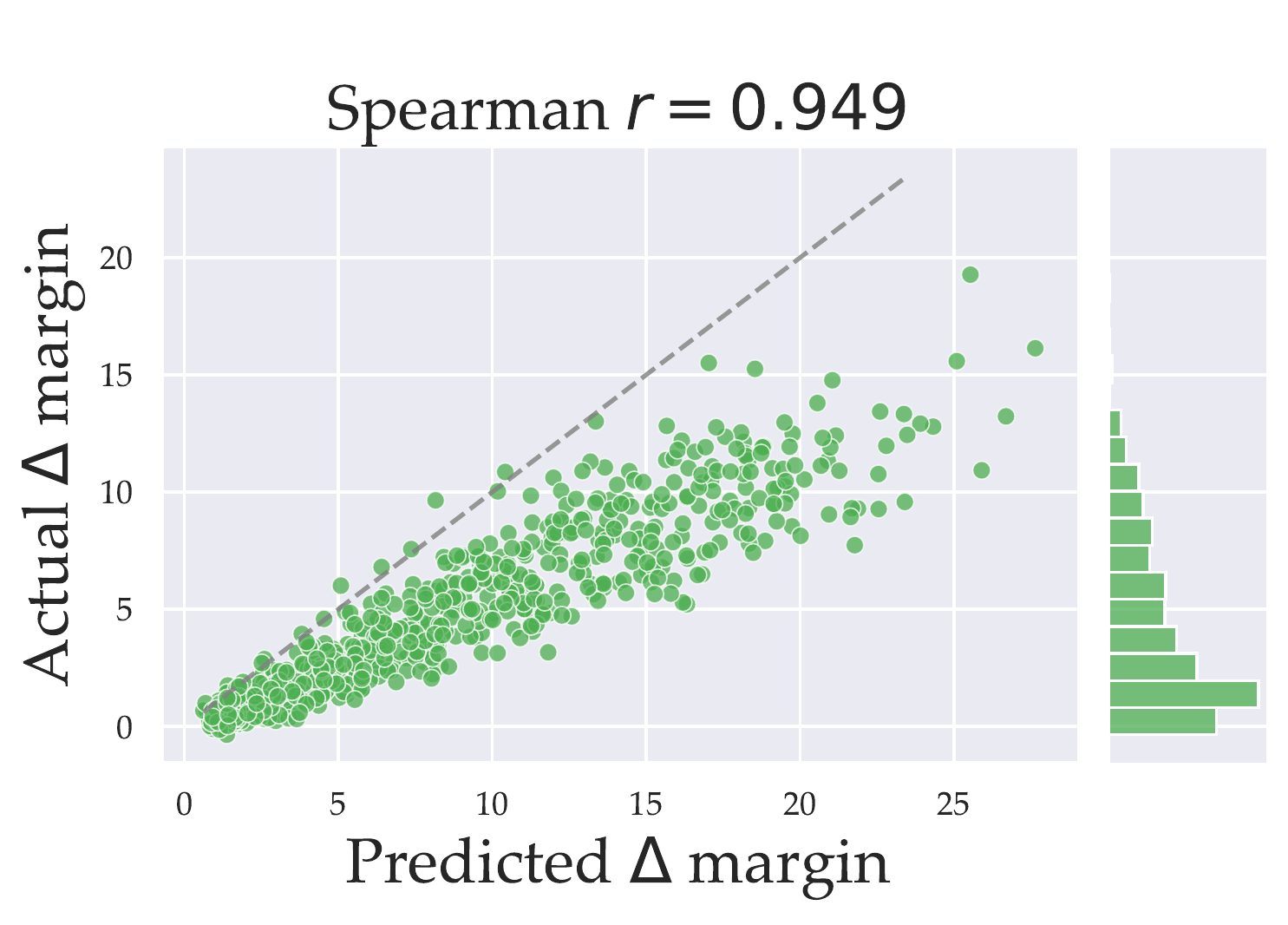}
        \caption{{\bf Transfer across model classes.}}
        \label{fig:rn18_causality}
    \end{minipage}
\end{figure}

\subsection{Transfer to different architecture}
While the main premise of datamodeling is understanding how data is used by a given {\em fixed} learning algorithm, it is natural to ask how well datamodels can predict across different learning algorithms.
We expect some degradation in predictiveness, as datamodels are fit to a particular learning algorithm; at the same time, we also expect some transfer of predictive power as modern deep neural networks are known to make similar predictions and errors \citep{mania2019model,toneva2019empirical}.

Here, we study one of the factors in a learning algorithm, the choice of architecture. We take the same counterfactuals and evaluate them on ResNet-18 models, using the same training hyperparameters.
As expected, the original datamodels continue to predict accurate counterfactuals for the new model class but with some degradation (\Cref{fig:rn18_causality}).

\subsection{Stress testing}
\label{app:causal_stress_test}

\Cref{sec:causal_view} showed that datamodels excel at predicting counterfactuals across a variety of removal mechanisms.
In an effort to find cases where datamodel predictions are not predictive of data counterfactuals, we evaluate the following additional counterfactuals:
\begin{itemize}
\item {\bf Larger groups
of examples (up to 20\% of the dataset)}: we remove $k=$ 2560, 5120, 10240 top weights using different datamodels $\alpha=0.1,0.2,0.5,0.75$. The changes in margin have more unexplained variance when larger number of images are removed; nonetheless, the overall correlation remains high
((\Cref{fig:cifar_causal_stress_k})).

\item {\bf Groups of training examples whose predicted effects are {\em zero}}: we remove $k=$ 20, 40, 80, 160, 320, 640, 1280, 2560 examples with zero weight ($\alpha=0.5$), chosen randomly among all such examples.
All of tested counterfactuals had negligible impact on the actual margin, consistent with the prediction of datamodels (\Cref{fig:cifar_causal_stress_zeros}).

\item {\bf Groups of examples whose predicted effect is {\em negative} according to baselines}: we test TracIn and influence functions.
(We do not consider the representation distance baseline here is there is no obvious way of extracting this information from it.)
Correlation degrades but remains high (\Cref{fig:cifar_causal_negative_baselines}). Note that the relative scale of the effects is much smaller compared to counterfactuals generated using datamodels (\Cref{fig:cifar_causal_stress_zeros}).

\end{itemize}

In general, although there is some reduction in datamodels'
predictiveness, we nevertheless find that datamodels continue to be accurate predictors of data counterfactuals.

\begin{figure}[!htb]
    \centering
    \includegraphics[width=\textwidth]{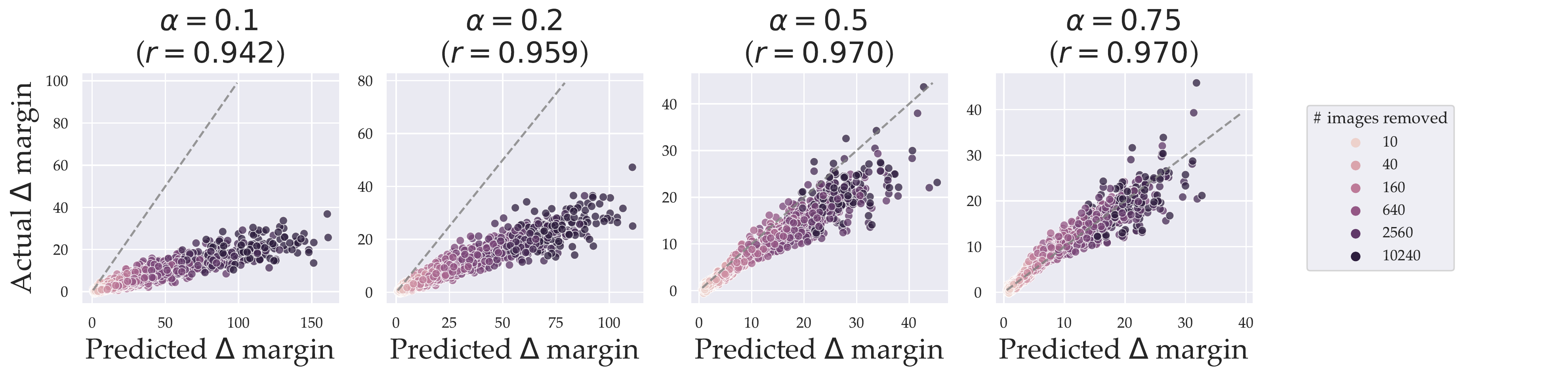}
    \caption{
        {\bf Stress testing datamodel counterfactuals by removing a large number of images.}
        Plot shows  datamodel counterfactuals from before ($k=10,...,1280$) along with additional samples $k=2560,5120, 110240$
        (shown with darker hue).
    }
    \label{fig:cifar_causal_stress_k}
\end{figure}

\begin{figure}[!htb]
    \centering
    \begin{subfigure}[]{0.48\linewidth}
        \centering
        \includegraphics[width=0.95\linewidth]{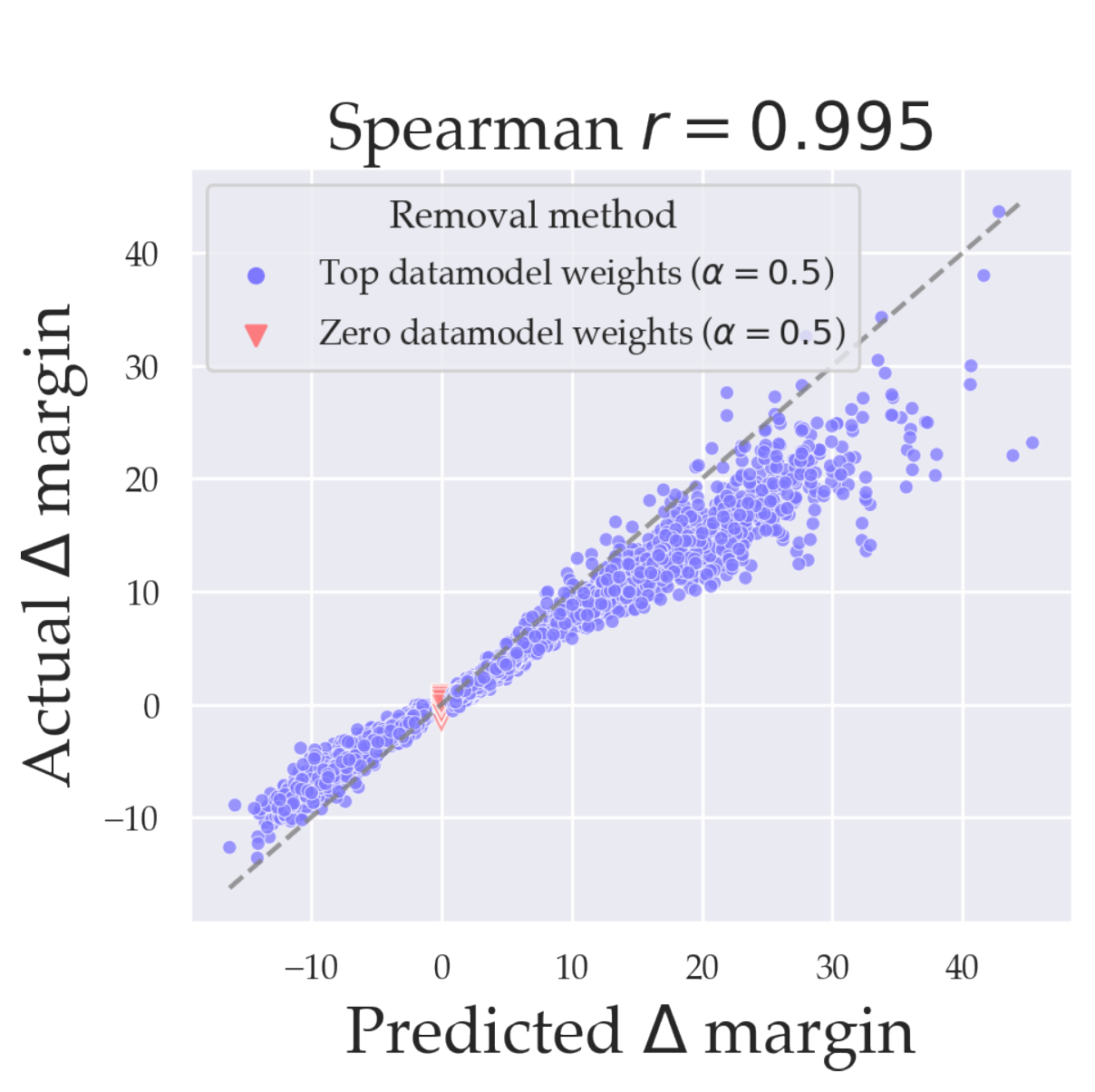}
        \caption{Comparing counterfactuals with highest vs {\em zero} predicted effect using datamodels ($\alpha = 0.5$)}
        \label{fig:cifar_causal_stress_zeros}
    \end{subfigure}
    \begin{subfigure}[]{0.48\linewidth}
        \centering
        \includegraphics[width=0.95\linewidth]{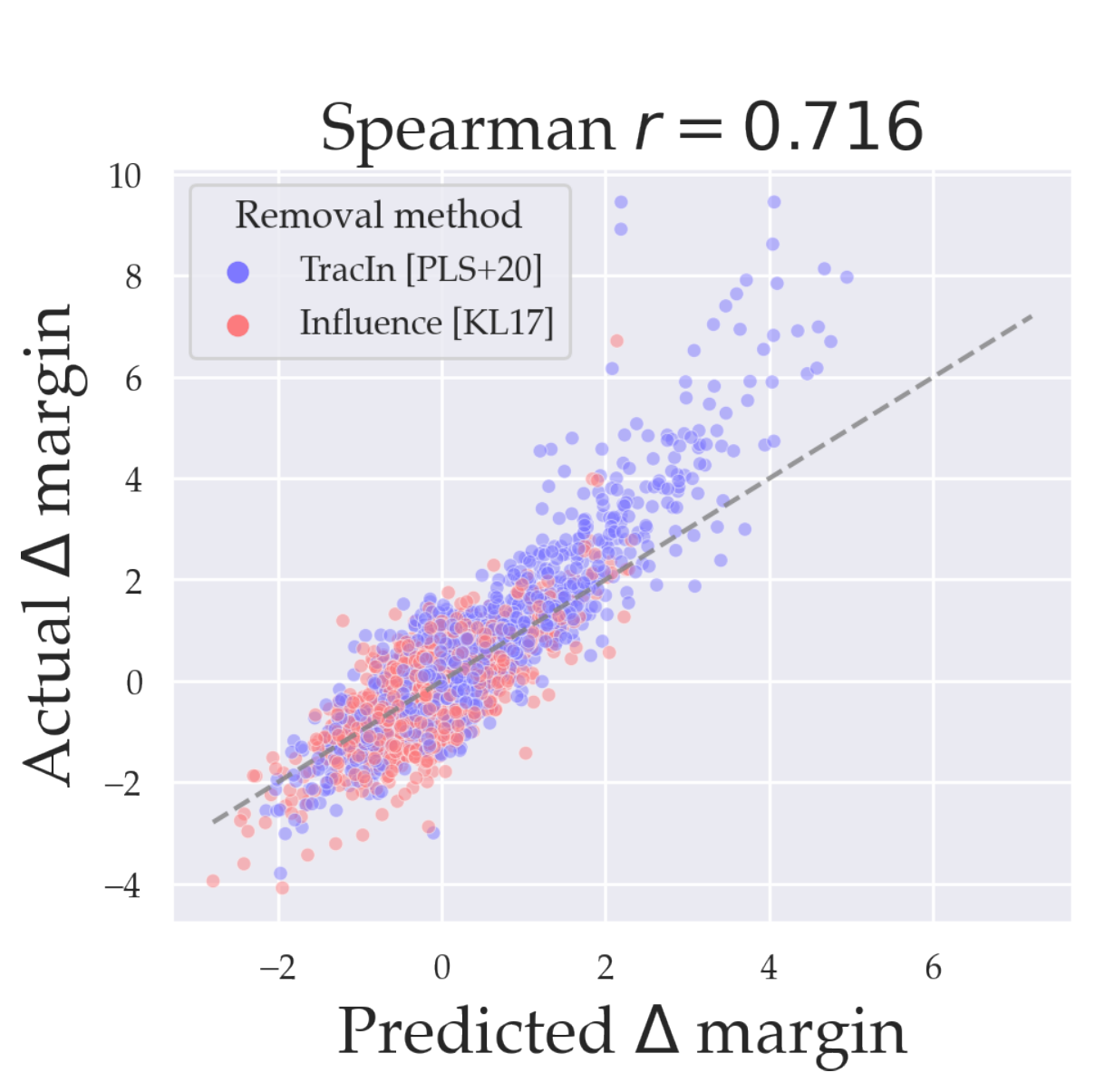}
        \caption{Removing top $k$ positive and negative influence training examples according to baseline methods.
        }
        \label{fig:cifar_causal_negative_baselines}
    \end{subfigure}
    \caption{
        {\bf Stress testing counterfactual prediction.}
    }
    \label{fig:cifar_causal_stress}
\end{figure}

\paragraph{Counterfactuals relative to a random control.}
All of the counterfactuals studied so far are relative a fixed control (the entire training set).
Here, %
we consider counterfactuals relative to a \emph{random} control $S_0 \sim \mathcal{D}_S$  at $\alpha=0.5$ (i.e. $|S_0| = \alpha |S|$).
The motivation for considering the shifted control is two folds:
first, the counterfactuals generated relative to such $S'$ are closer in distribution to the original distribution to which datamodels were fit to, so it is natural to study datamodels in this regime;
second, this tests whether the counterfactual predictability is robust to the exact choice of the trainset.
Latter is desirable, as ultimately we would like to understand how models behave on training sets similar in distribution to $S$, not the exact train set.

To implement above, after removing a target group $G$ from the full train set $S$, we subsample the remainder $S / G$ with probability $\alpha$.
We adjust the control values accordingly to
$\mathbb{E}_{S_0 \sim \mathcal{S}}[\modeleval{x}{S_0}]$,
where $\mathcal{D}_S$ is the $\alpha=0.5$ subsampling distribution.
The results show that one can indeed also predict counterfactuals relative a random control (\Cref{fig:cifar_causal_50}).

\begin{figure}[!htbp]
    \centering
    \begin{subfigure}[]{0.48\linewidth}
        \centering
        \includegraphics[width=0.95\linewidth,trim={0 0 2.5cm 0},clip]{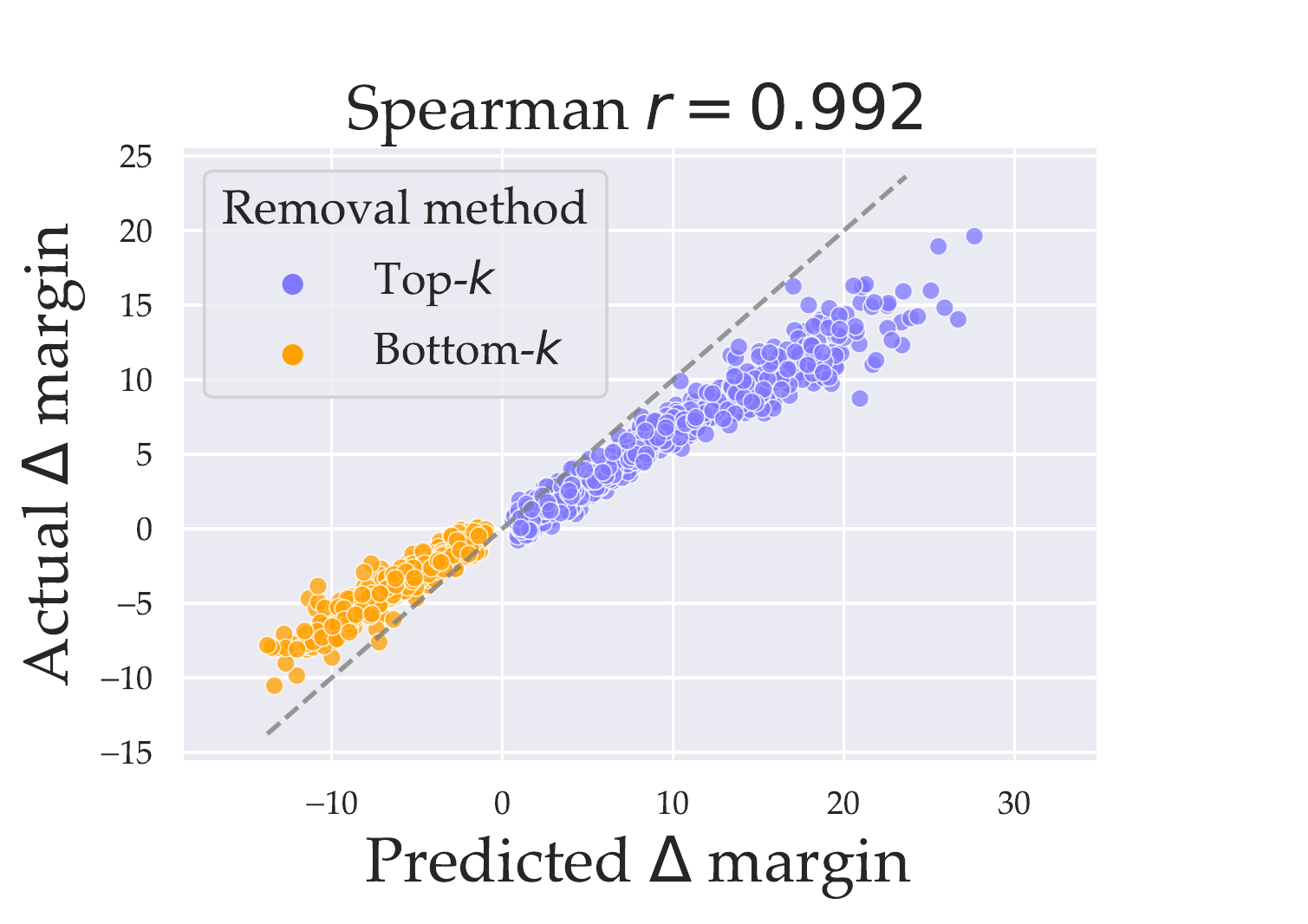}
        \caption{Random control ($\alpha = 0.5$ regime)}
        \label{fig:pct_50_orig}
    \end{subfigure}
    \begin{subfigure}[]{0.48\linewidth}
        \centering
        \includegraphics[width=0.95\linewidth,trim={0 0 2.5cm 0},clip]{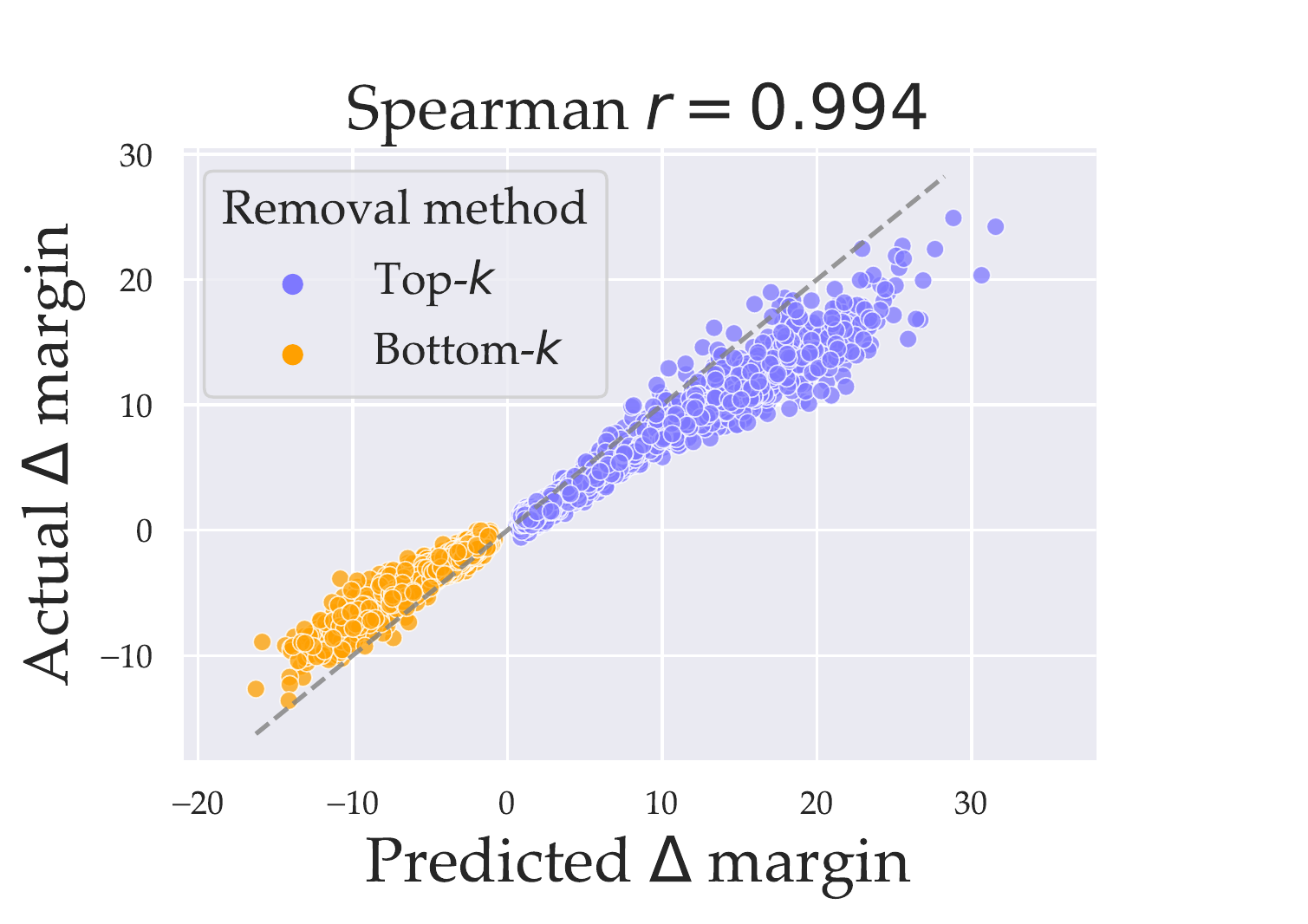}
        \caption{Fixed control ($\alpha = 1.0$)}
        \label{fig:pct_50_full}
    \end{subfigure}
    \caption{{\bf Datamodels can predict counterfactuals relative to random controls.}
    As in~\Cref{fig:cifar_causal}, each point in the graphs above corresponds to a test example and a counterfactual trainset $S'$ (a subset of the full training set, $S$). The counterfactual is relative a \emph{random} control $S_0 \sim \mathcal{D}_S$ with $\alpha=0.5$, e.g. a set randomly subsampled at $50\%$.
    The $y$-coordinate of each point represents the expected {\em ground-truth}
    difference, in terms of model output on $x$, between training on a random $S_0$, and training on $S'$.
    The $x$-coordinate of each point represents the {\em
    datamodel-predicted} value of this quantity.
    \textbf{(a)} We use the $\alpha = 0.5$ datamodels to predict counterfactuals
    generated by removing, for each test example, the training inputs
    corresponding to the top-$k$ and bottom-$k$ (for several $k$) datamodel weights.
    \textbf{(b)} Same, but relative to a fixed control $S_0 = S$, e.g. the full train set.
    }
    \label{fig:cifar_causal_50}
\end{figure}

\subsection{Additional plots for different \texorpdfstring{$\alpha$}{} values}
\label{app:causal_more_alpha}

\begin{figure}[!htb]
    \centering
    \includegraphics[width=\textwidth]{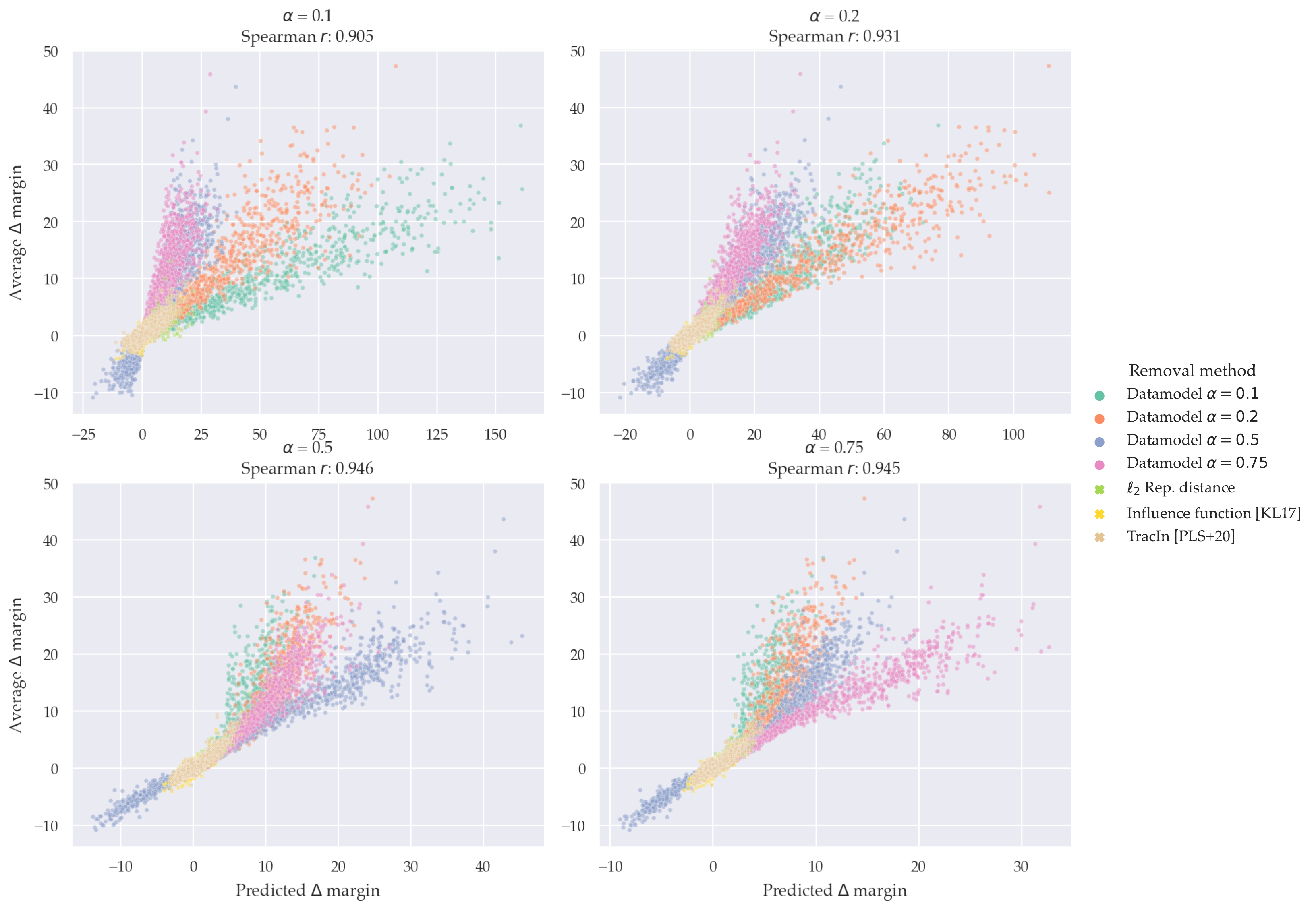}
    \caption{
        {\bf Varying $\alpha$ for counterfactual prediction (CIFAR-10).}
        Same plot as in~\Cref{fig:cifar_causal}, except varying the datamodels used for prediction; each plot uses datamodels with the given $\alpha$.
        As before, each point in the graphs above corresponds to a test example and a
        subset $R(x)$ of the original training set $S$, identified by a
        (color-coded) heuristic.
        The $y$-coordinate of each point represents the {\em ground-truth}
        difference, in terms of model output on $x$, between training on $S$, and
        training on $S \setminus R(x)$. The $x$-coordinate of each point represents the {\em
        datamodel-predicted} value of this quantity.
    }
    \label{fig:cifar_scatter_vary_alpha}
\end{figure}

\begin{figure}[!htb]
    \centering
    \includegraphics[width=\textwidth]{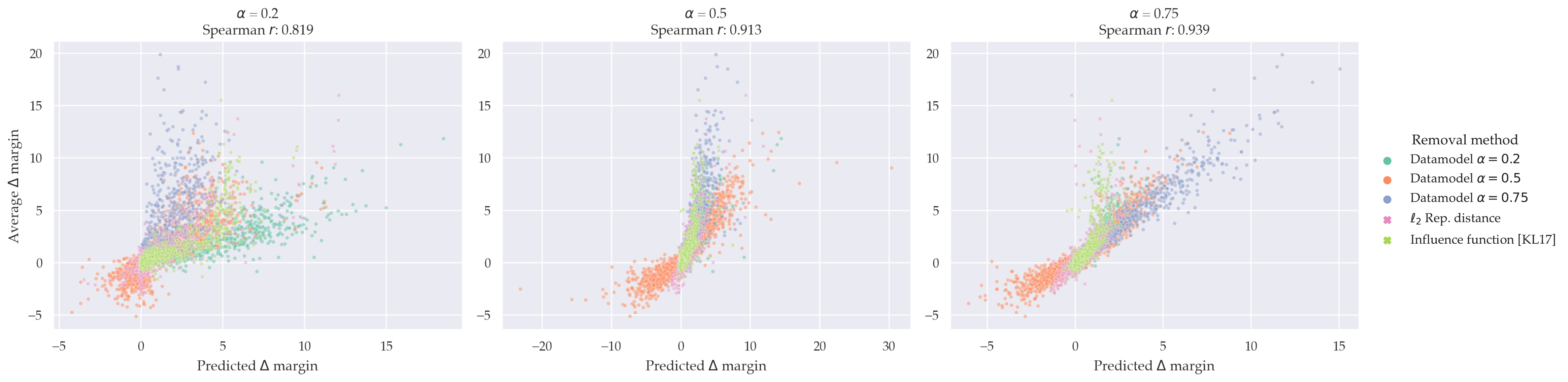}
    \caption{
        {\bf Varying $\alpha$ for counterfactual prediction (\fmow).}
        Same plot as in~\Cref{fig:cifar_causal}, except varying the datamodels used for prediction; each plot uses datamodels with the given $\alpha$.
        As before, each point in the graphs above corresponds to a test example and a
        subset $R(x)$ of the original training set $S$, identified by a
        (color-coded) heuristic.
        The $y$-coordinate of each point represents the {\em ground-truth}
        difference, in terms of model output on $x$, between training on $S$, and
        training on $S \setminus R(x)$. The $x$-coordinate of each point represents the {\em
        datamodel-predicted} value of this quantity.
    }
    \label{fig:fmow_scatter_vary_alpha}
\end{figure}

%% file: appendices/nearest.tex
\section{Nearest Neighbors}
\label{app:nns}
In this section we show additional examples of held-out images and their corresponding
train image, datamodel weight pairs.

\subsection{CIFAR}
In Figure~\ref{appfig:posneg} we show more {\em randomly} selected test images along with their
positive and negative weight training examples.
In Figure~\ref{appfig:more_alphas} we show
more examples of test images and their corresponding top train images as
we vary $\alpha$.
In Figure~\ref{appfig:nn_baseline} we compare most similar images to given test images identified using various baselines (see \Cref{app:baselines} for their description).

\begin{figure}[!h]
\centering
\includegraphics[width=\textwidth]{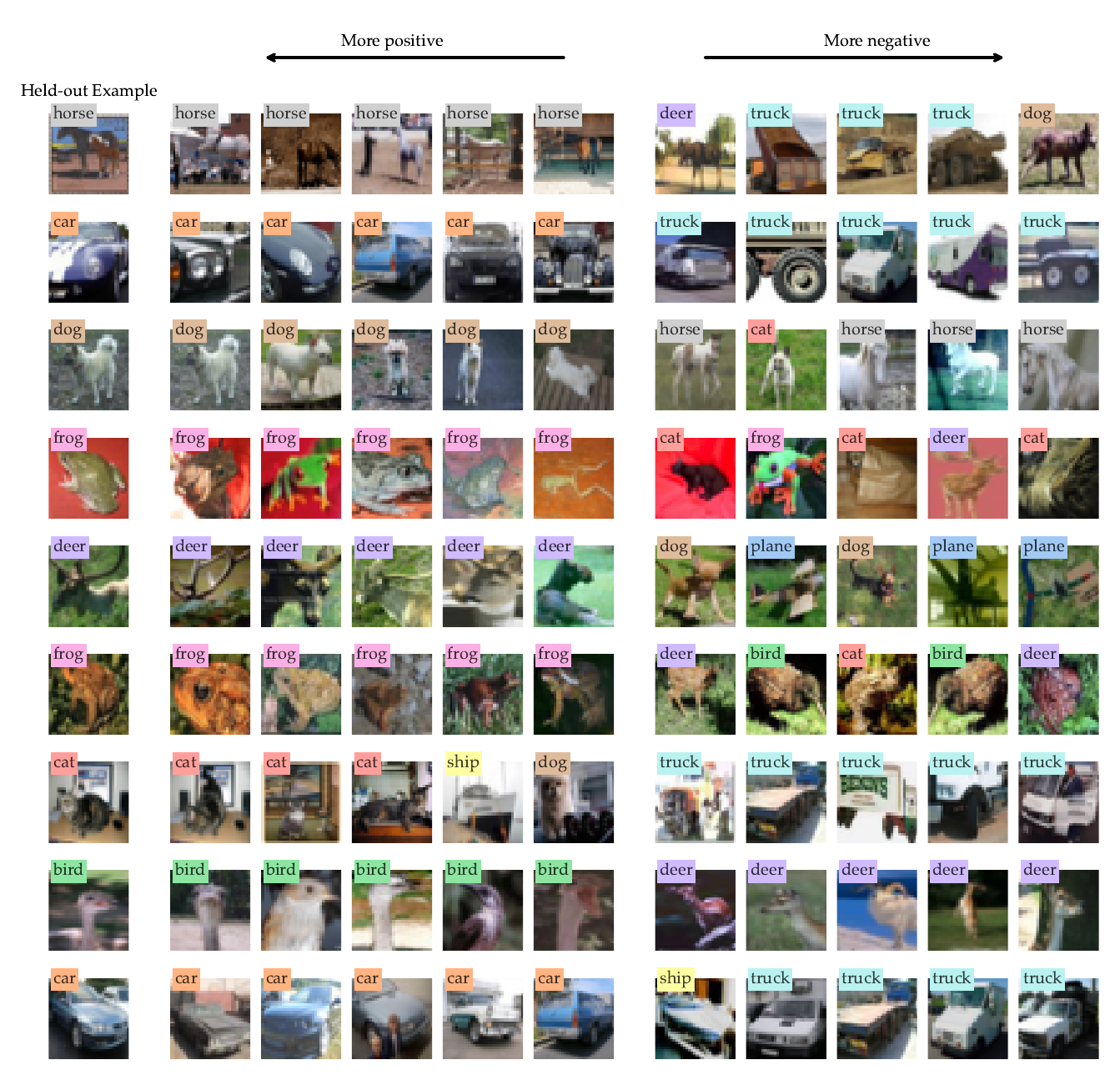}
\caption{Additional examples of held-out images and their corresponding
highest and lowest datamodel weight training images.}
\label{appfig:posneg}
\end{figure}

\begin{figure}[!hp]
\centering
\includegraphics[width=\textwidth]{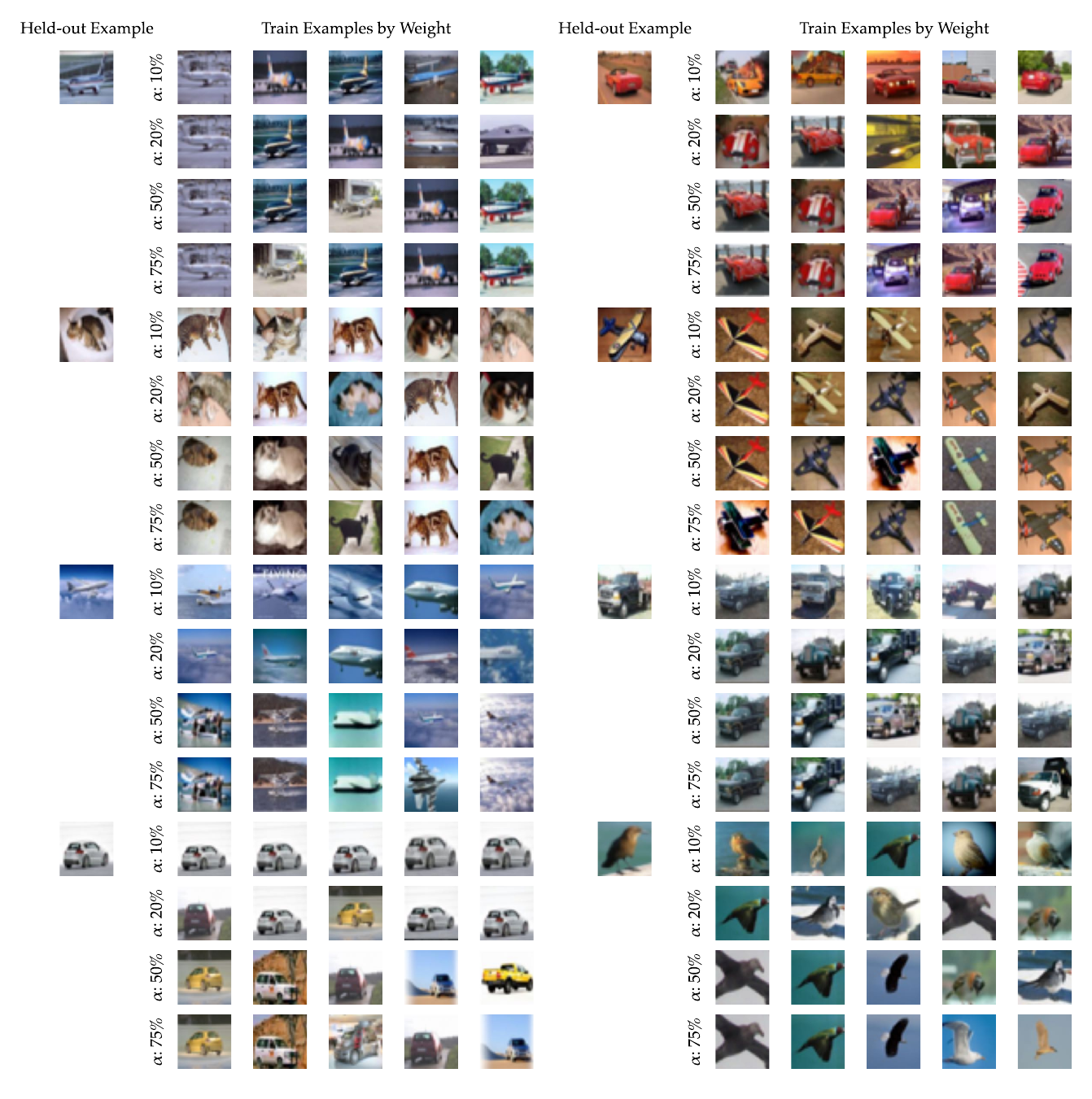}
\caption{Additional examples of held-out images and corresponding
most relevant training examples, while varying $\alpha$.}
\label{appfig:more_alphas}
\end{figure}

\begin{figure}[!hp]
    \centering
    \includegraphics[width=\textwidth]{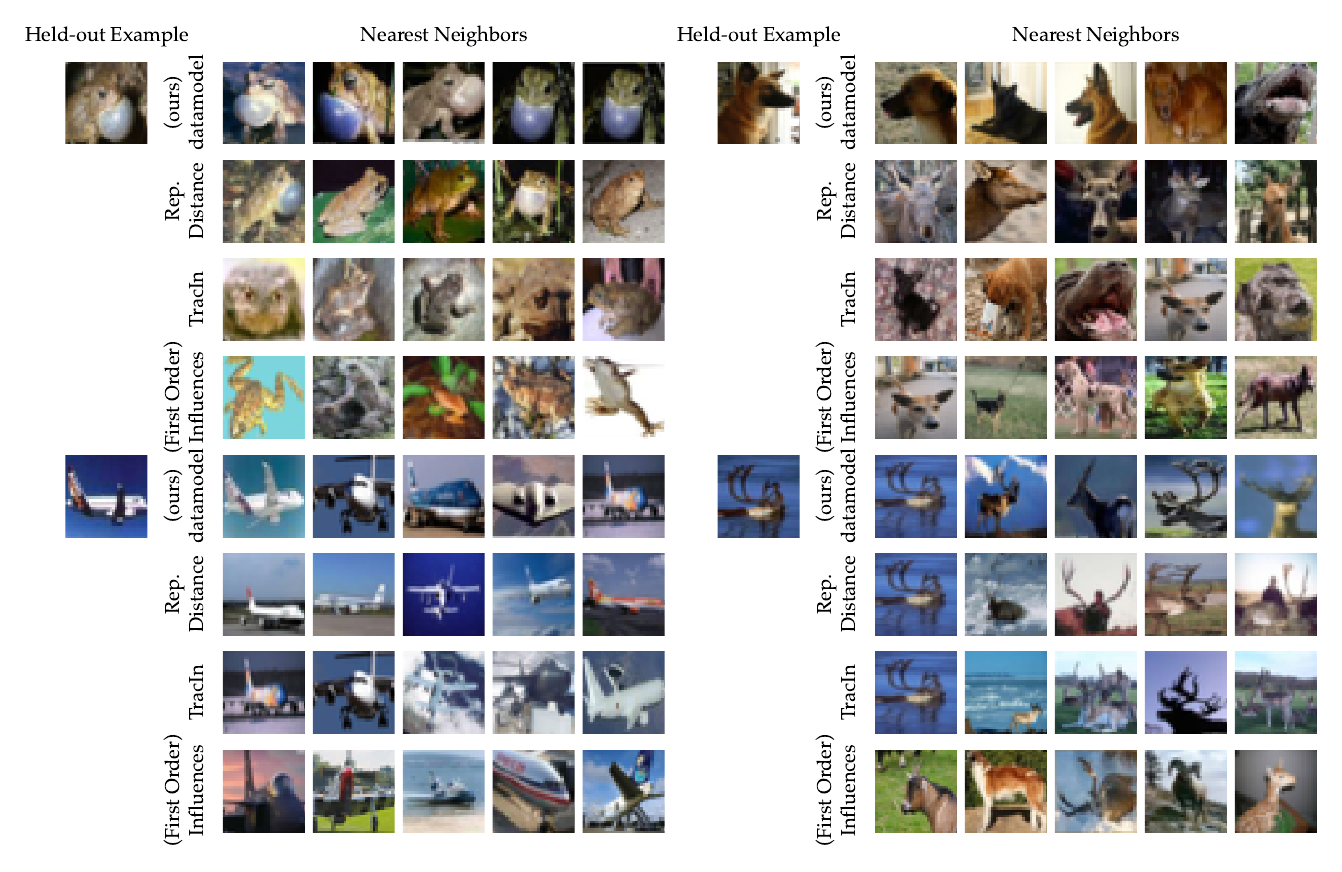}
    \caption{Comparisons of nearest neighbors found using different methods.}
    \label{appfig:nn_baseline}
\end{figure}

\subsection{\fmow{}}
In~\Cref{fig:fmow_nn} we show {\em randomly selected} target images along with their top-weight train images, using datamodels of different $\alpha$.
In Figure~\ref{appfig:more_alphas_fmow} we show
more examples of test images and their corresponding top train images as
we vary $\alpha$.

\begin{figure}[htbp]
    \centering
    \includegraphics[width=0.9\linewidth]{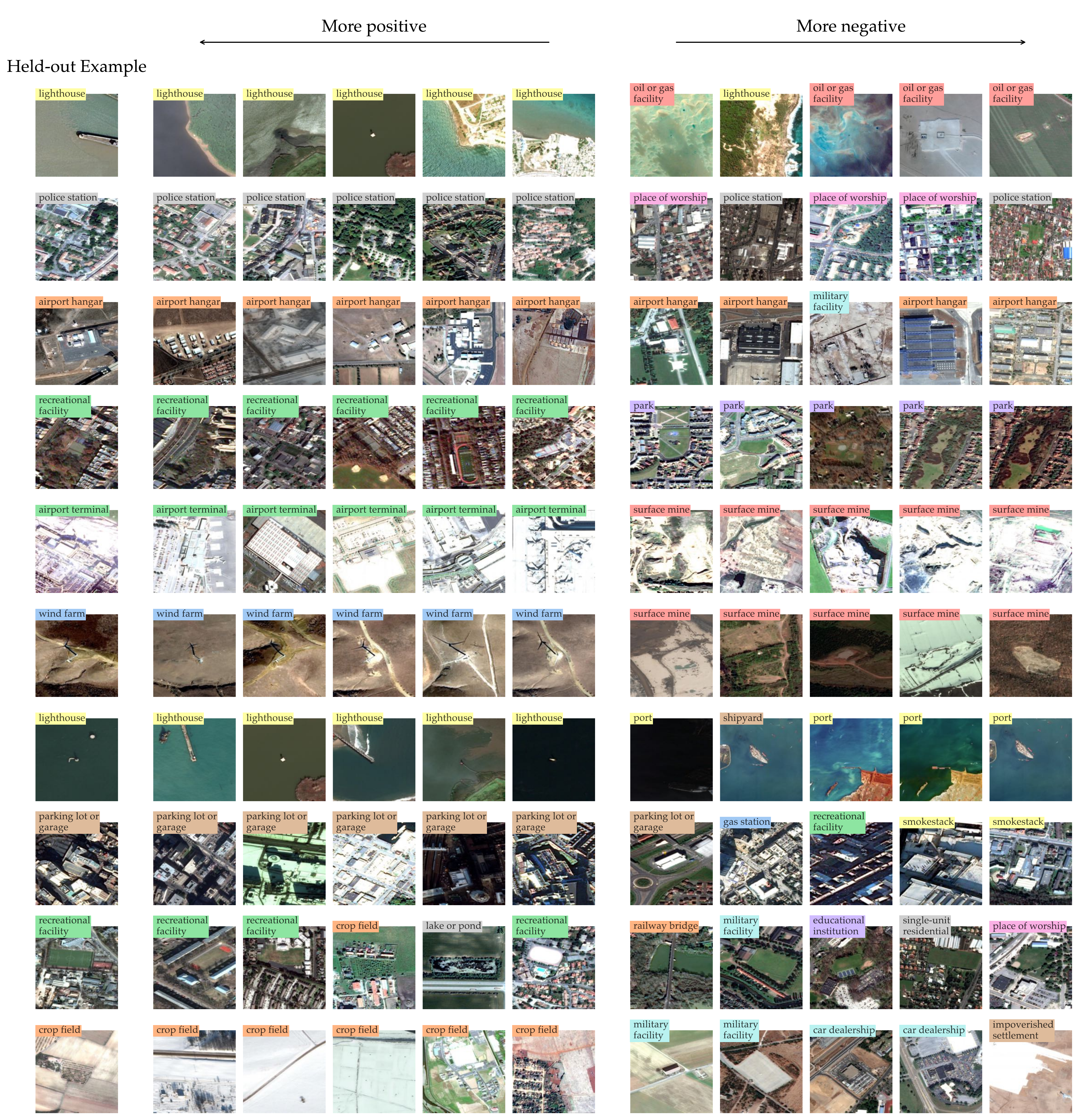}
    \caption{\fmow{} examples of held-out images and their corresponding highest and lowest datamodel weight training images.}
    \label{fig:fmow_nn}
\end{figure}

\begin{figure}[!hp]
    \centering
    \includegraphics[width=\textwidth]{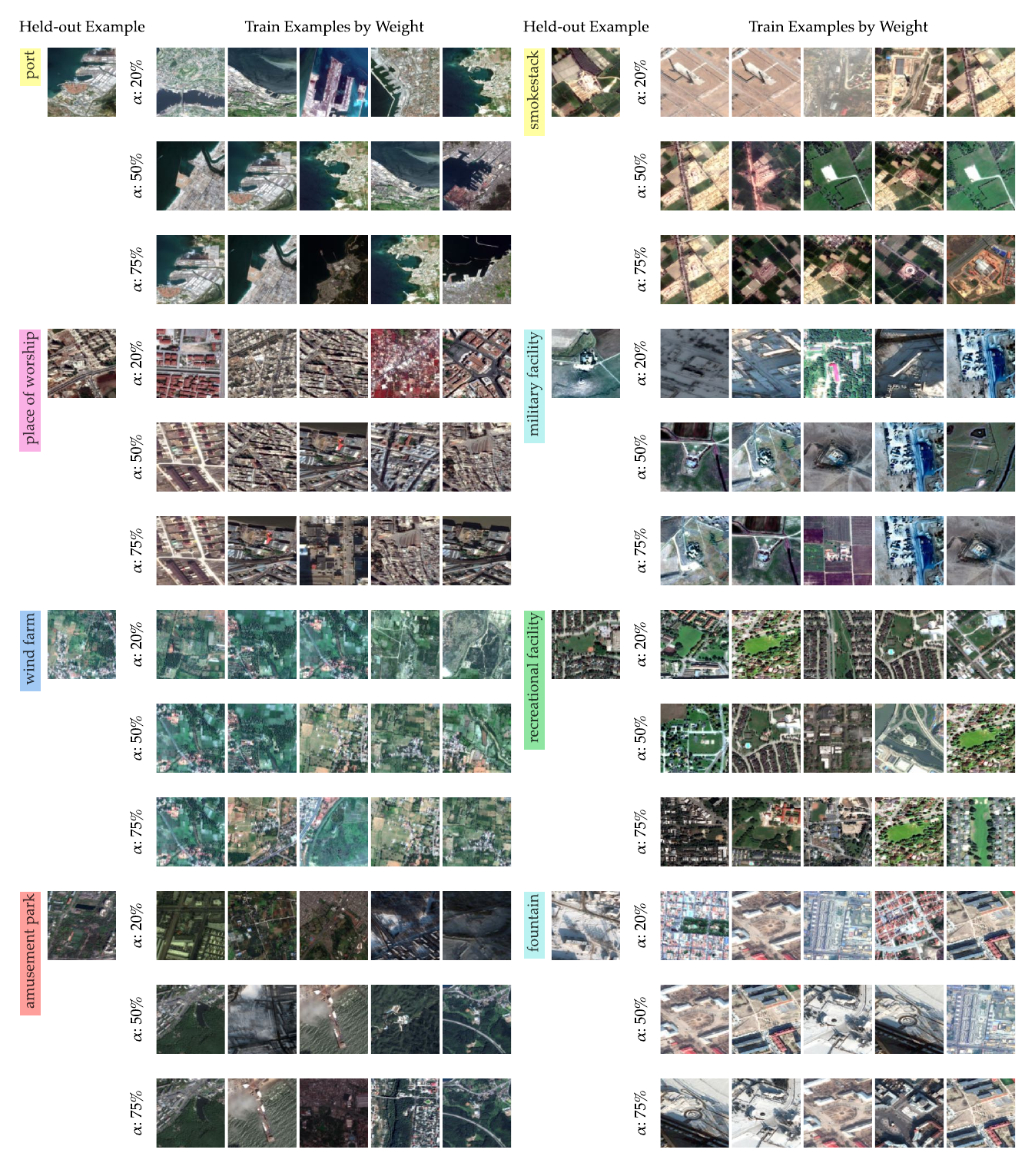}
    \caption{\fmow{} examples of held-out images and corresponding
    most relevant training examples, while varying $\alpha$.}
    \label{appfig:more_alphas_fmow}
\end{figure}

%% file: appendices/ttl.tex
\section{Train-Test Leakage}
\label{app:ttl}

\subsection{CIFAR}
\paragraph{Task setup.} The general setup is as described \Cref{sec:ttl}. In the Amazon Mechanical
Turk interface (\Cref{appfig:cifar_interface}), for each test image we displayed the top 5 and bottom 5 train examples by datamodel weight; the vast majority of
potential leakage found corresponded to the top 5 examples.
Nine different workers filled out each task. We paid 12 cents per task
completed and used these qualifications: locale in US/CA/GB and percentage of hits
approved $>95\%$.

\begin{figure}[h]
\centering
\includegraphics[width=\textwidth]{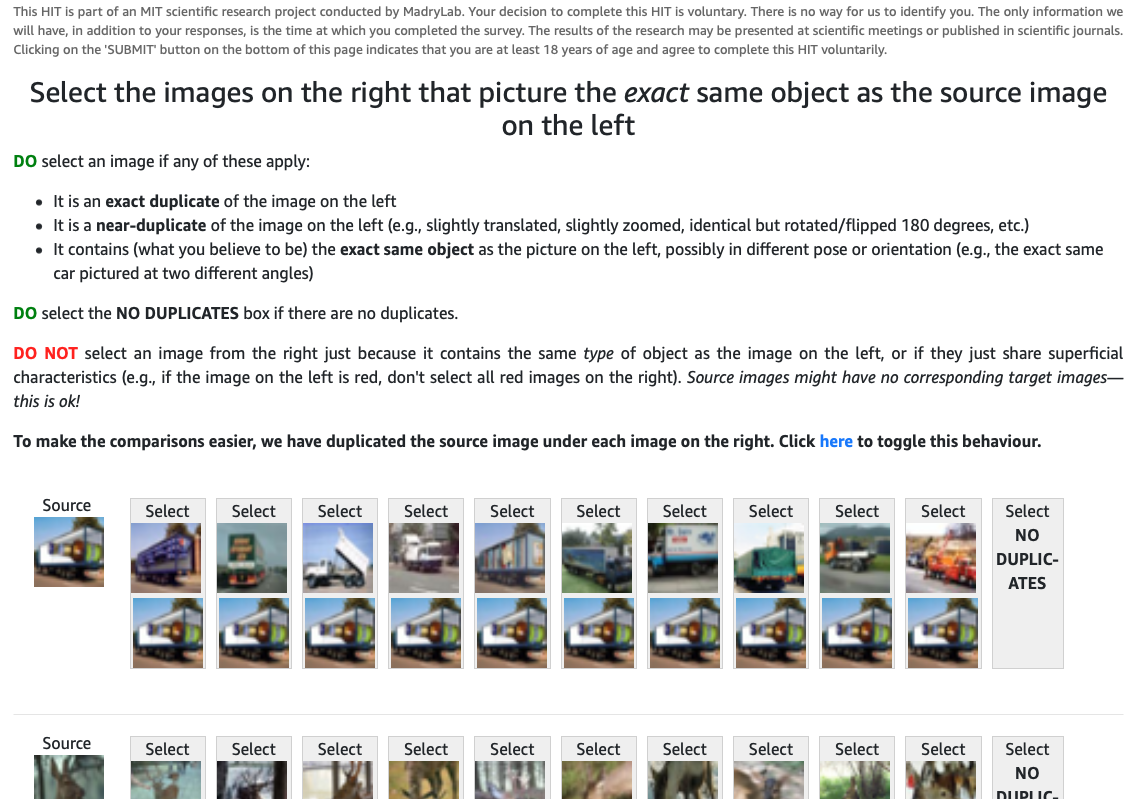}
\caption{The MTurk Interface, complete with instructions, shown to crowdsourced
annotators. Note that there are 5 rows of images in the actual interface, some
of which may require scrolling to get to.}
\label{appfig:cifar_interface}
\end{figure}

\paragraph{Examples.}
\Cref{appfig:cifar_ttl} shows more examples of (train, test)
pairs stratified by annotation score.
While there is no ground truth due
to lack of metadata, we see that the crowdsourced annotation combined with high quality candidates (as identified by datamodels) can effectively surface leaked examples.

\begin{figure}[h]
\centering
\includegraphics[width=0.9\textwidth]{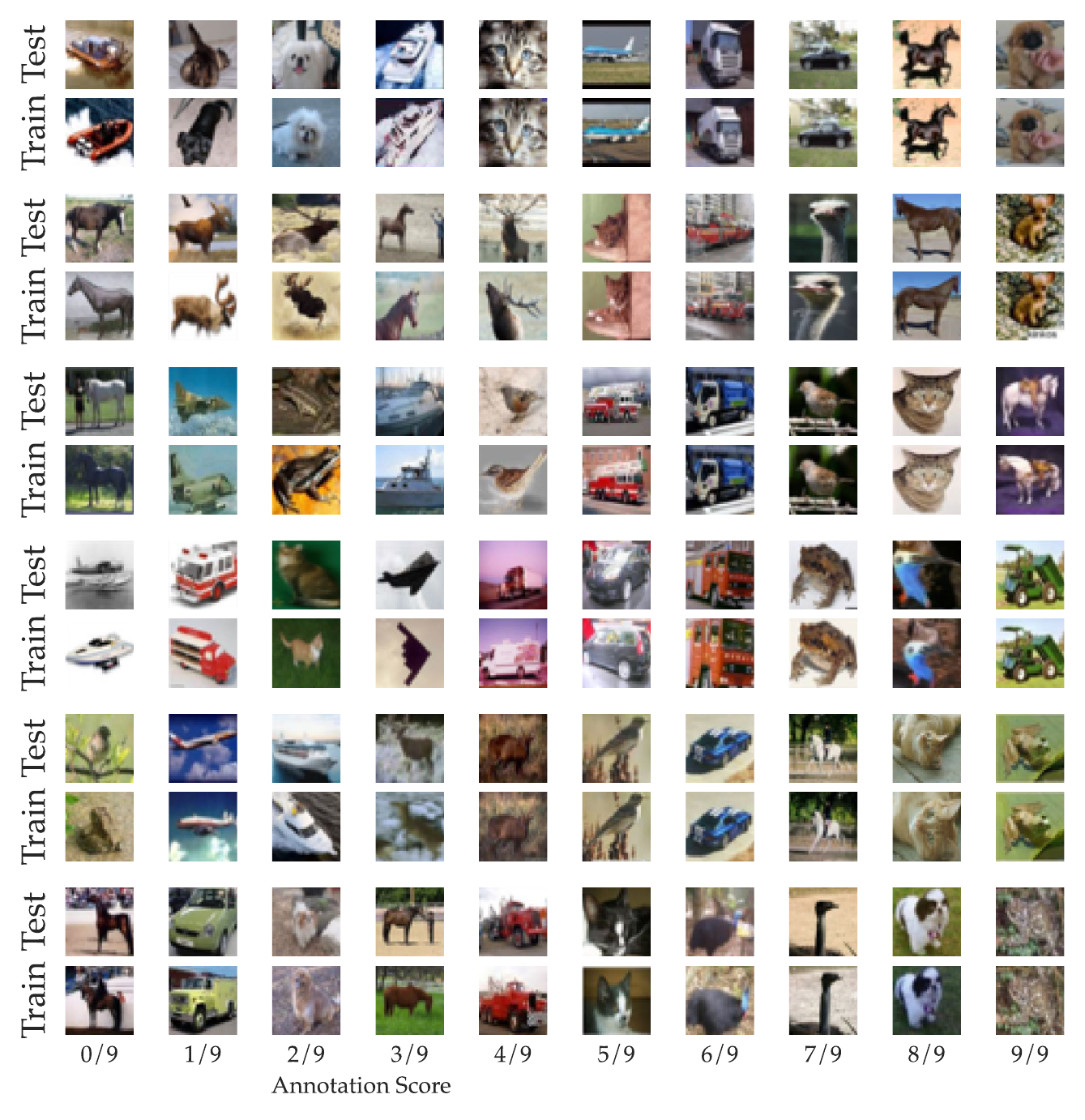}
\caption{More annotation scores paired with (train, test) leakage pairs. See
Figure~\ref{fig:cifar_ttl} and Section~\ref{sec:ttl} for more information.}
\label{appfig:cifar_ttl}
\end{figure}

\clearpage

\paragraph{Comparison with CIFAIR.} \citet{barz2020train} present CIFAIR, a
version of CIFAR with fewer duplicates. The authors define duplicates slightly
differently than our definition of same scene train-test leakage (cf. Section 3.2 of their work and our interface shown in \Cref{appfig:cifar_interface}).
They identify train-test leakage by
using a deep neural network to measure representation space distances between
images across training partitions and manually inspecting the lowest distances.

\subsection{FMoW}
Appendix Figure~\ref{appfig:fmow_cdf} shows the CDF of each test
image's minimum distance to a train image.

\begin{figure}[h]
\centering
\includegraphics[width=0.5\textwidth]{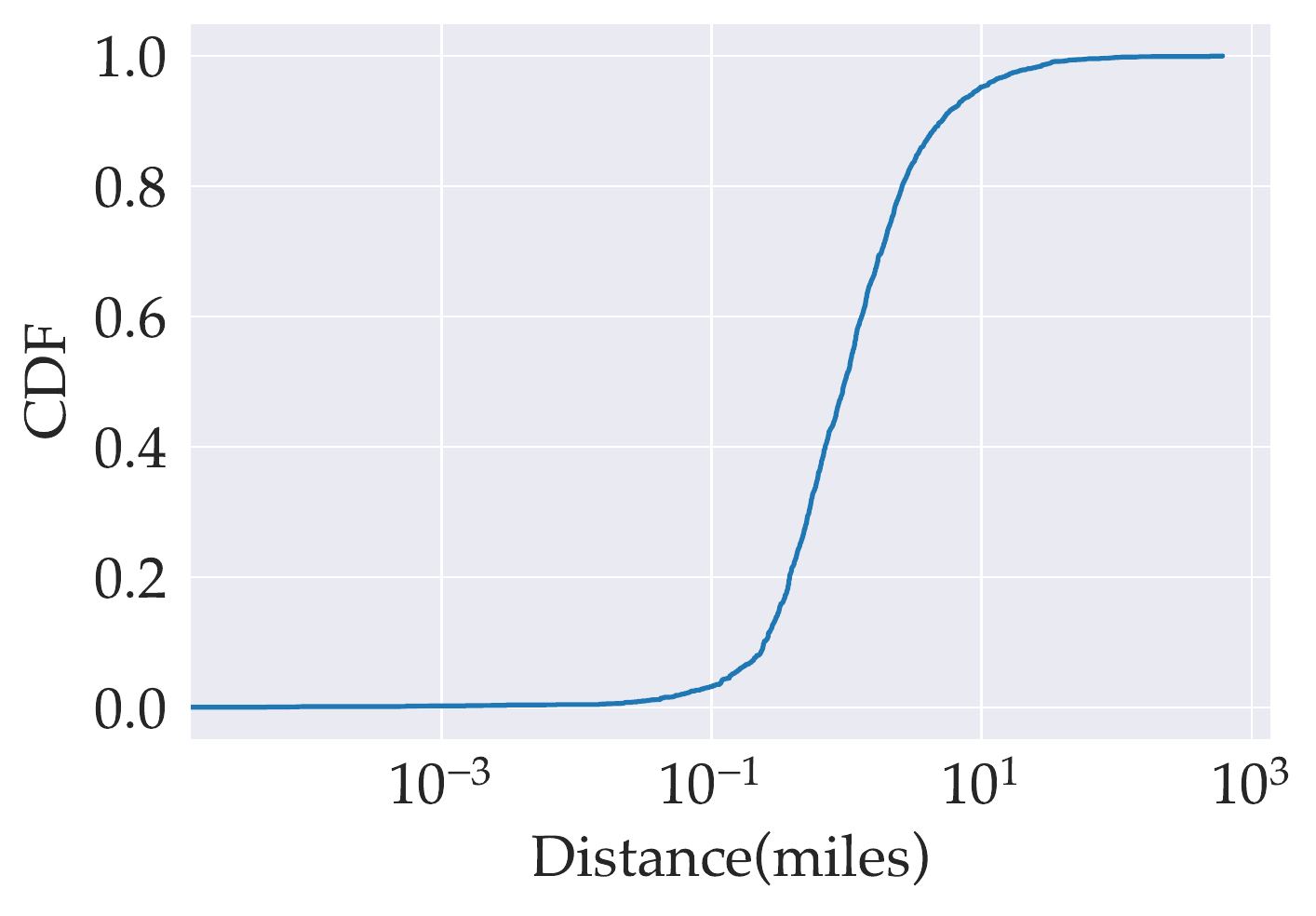}
\caption{CDF of distance in miles between each FMoW test set image and the nearest train set
image.}
\label{appfig:fmow_cdf}
\end{figure}

%% file: appendices/spectral_app.tex
\section{Spectral Clustering}
\label{app:spectral}

\subsection{Method}
We use \texttt{sklearn}'s \texttt{cluster.SpectralClustering}.
Internally, this computes similarity scores using the radial basis function (RBF) kernel on the datamodel embeddings. %
Then, it runs spectral clustering on the graph defined by the similarity matrix $A$: it computes a Laplacian $L$, represents each node using the first $k$ eigenvectors of $L$, and runs $k$-means clustering on the resulting feature representations. We use $k=100$.

\subsection{Omitted results}
\Cref{fig:clustering_extra_compare} compares top clusters for the horse class across different $\alpha$.
\Cref{fig:clustering_extra} shows additional clusters for eight other classes, apart from the ones shown in \Cref{fig:clustering}.

\begin{figure}[h]
    \centering
    \includegraphics[width=0.49\linewidth,trim={19cm 0 0 0},clip]{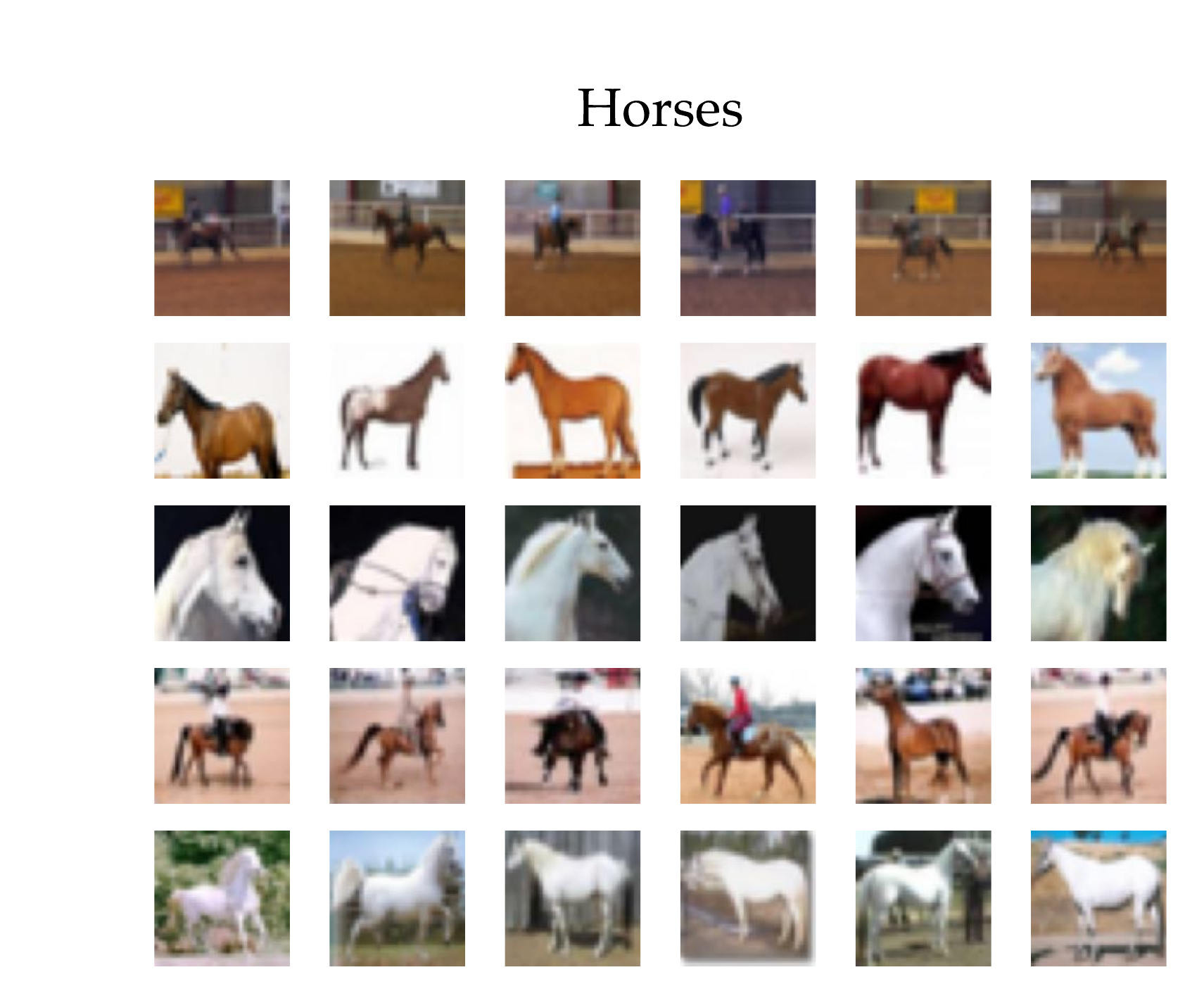}
    \includegraphics[width=0.49\linewidth,trim={19cm 0 0 0},clip]{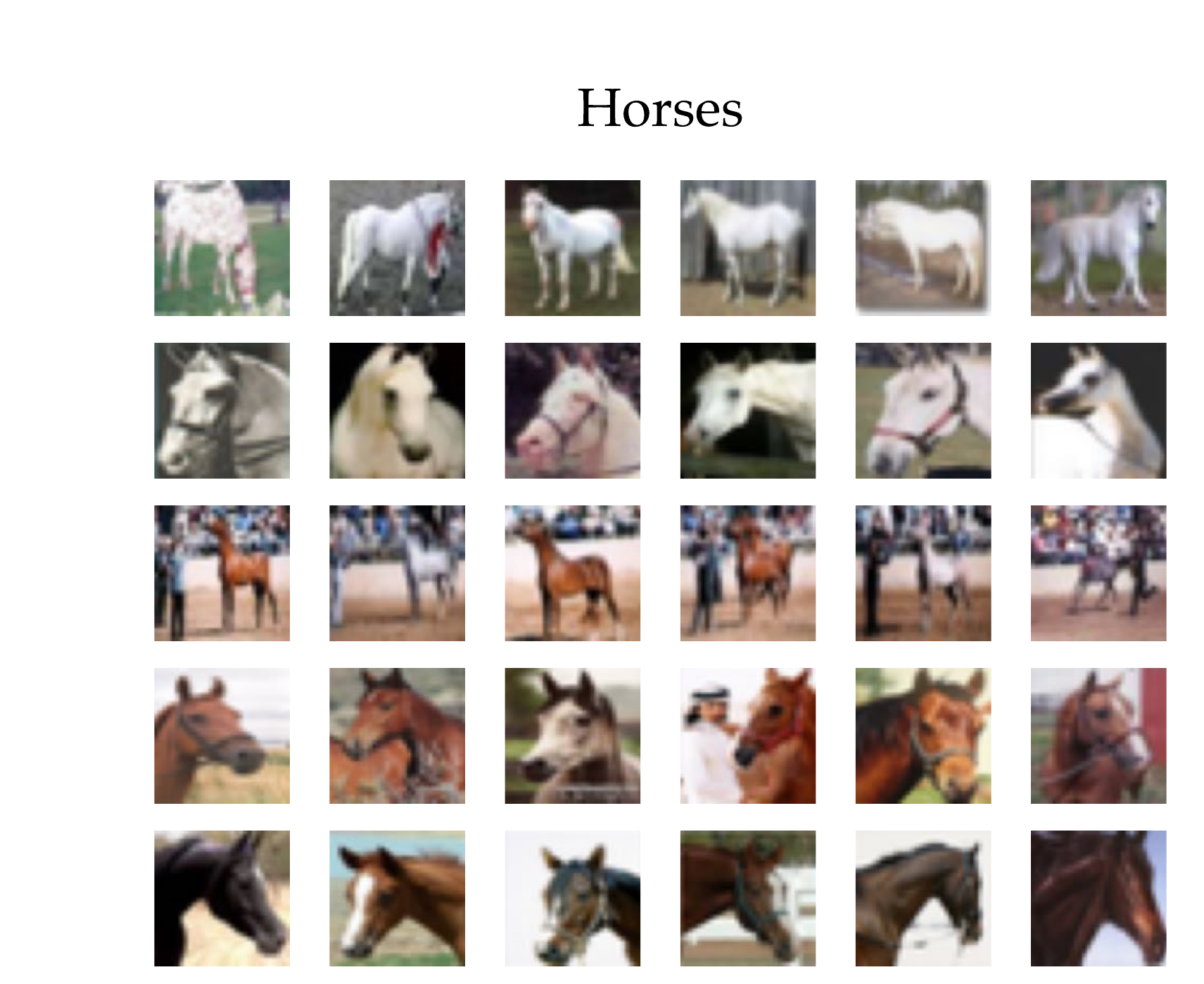}
    \includegraphics[width=0.49\linewidth,trim={19cm 0 0 0},clip]{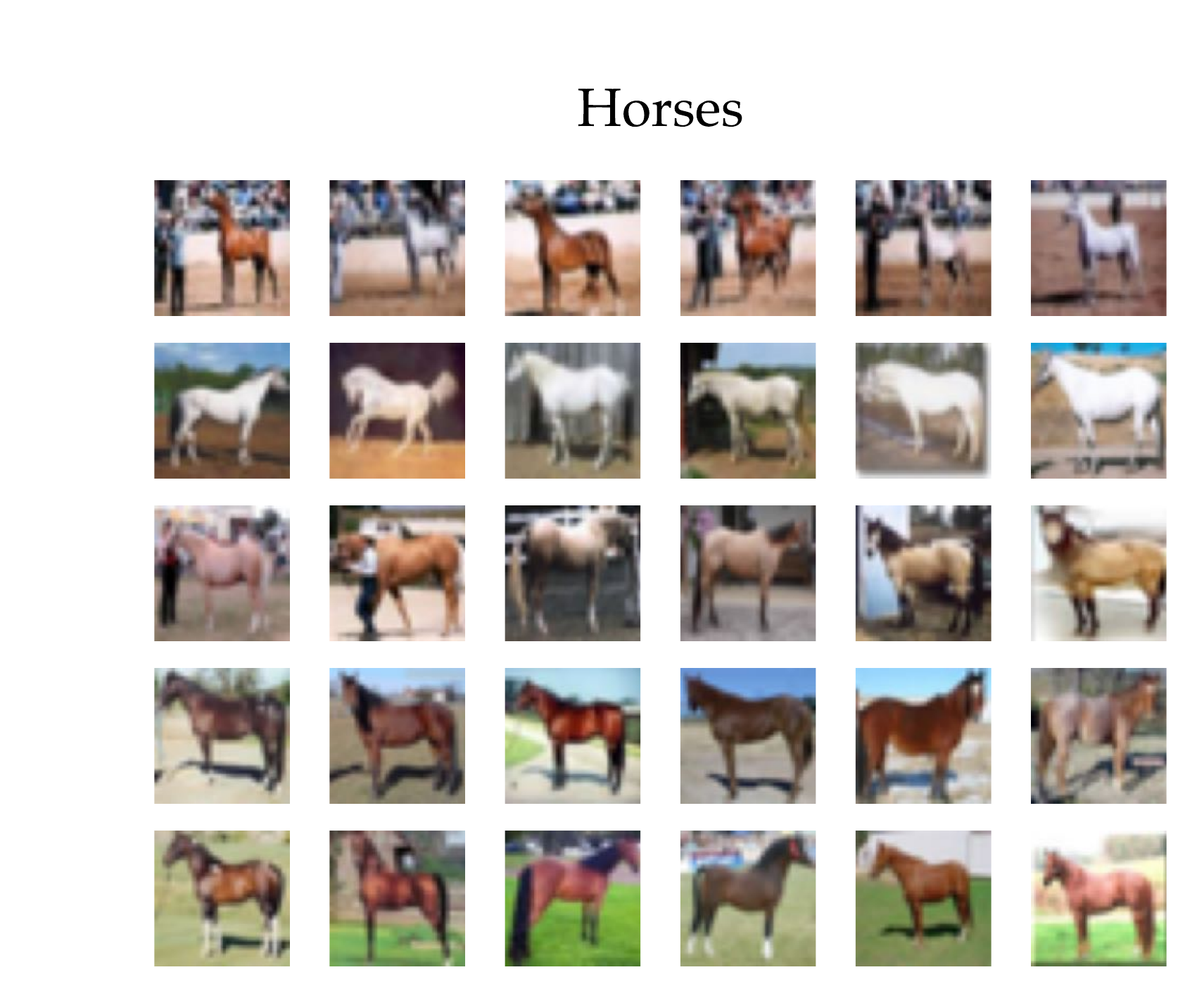}
    \includegraphics[width=0.49\linewidth,trim={19cm 0 0 0},clip]{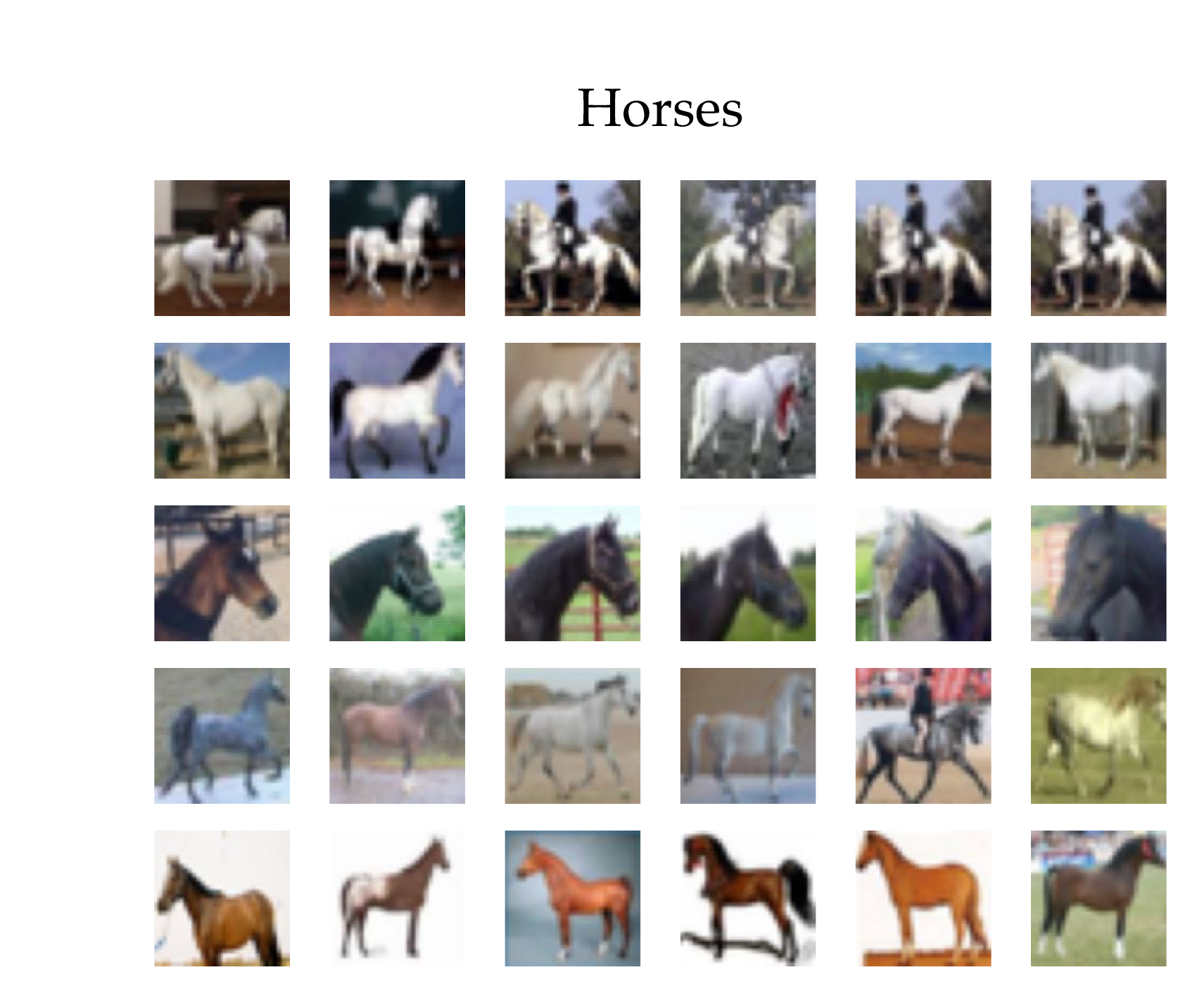}
    \caption{Omitted spectral clustering results for datamodels computed with
    $\alpha = 10\%$ (top left), $20\%$ (top right), $50\%$ (bottom left), and
    $75\%$ (bottom left).}
    \label{fig:clustering_extra_compare}
\end{figure}

\begin{figure}[h]
    \centering
    \includegraphics[width=0.9\linewidth,trim={0 0 0 1cm}, clip]{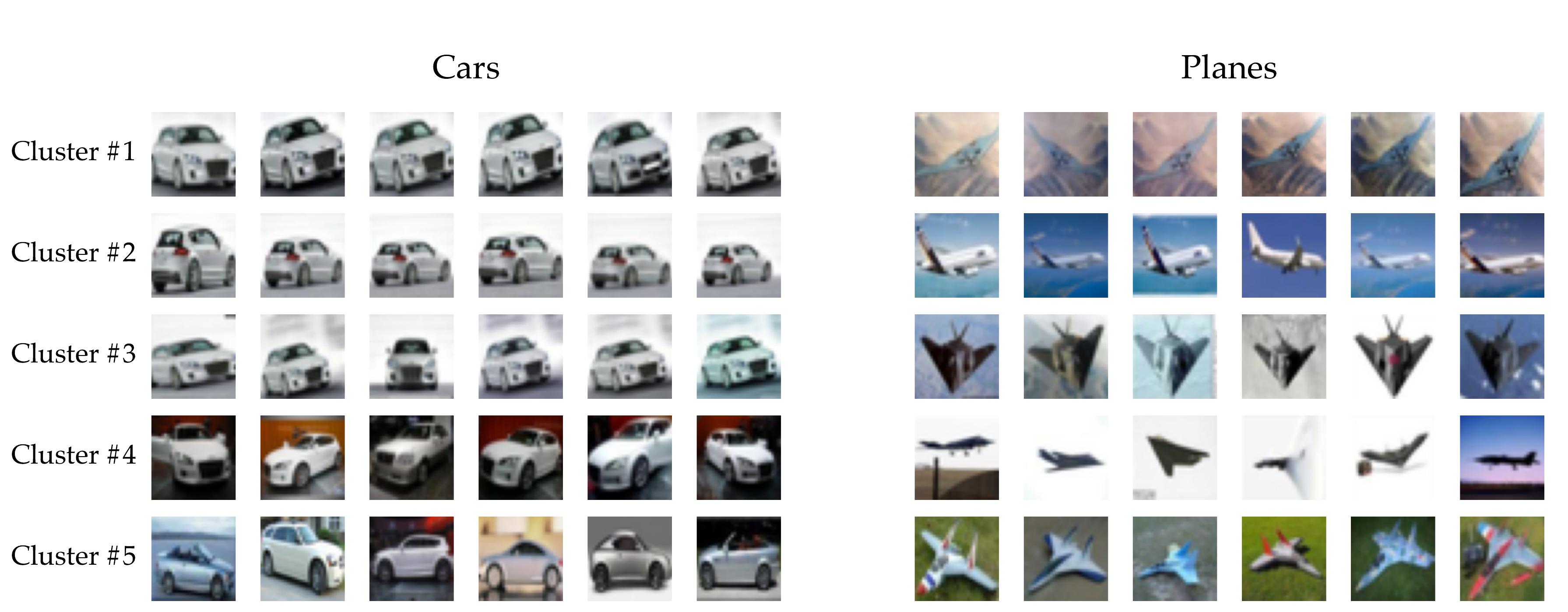}
    \includegraphics[width=0.9\linewidth,trim={0 0 0 1cm}, clip]{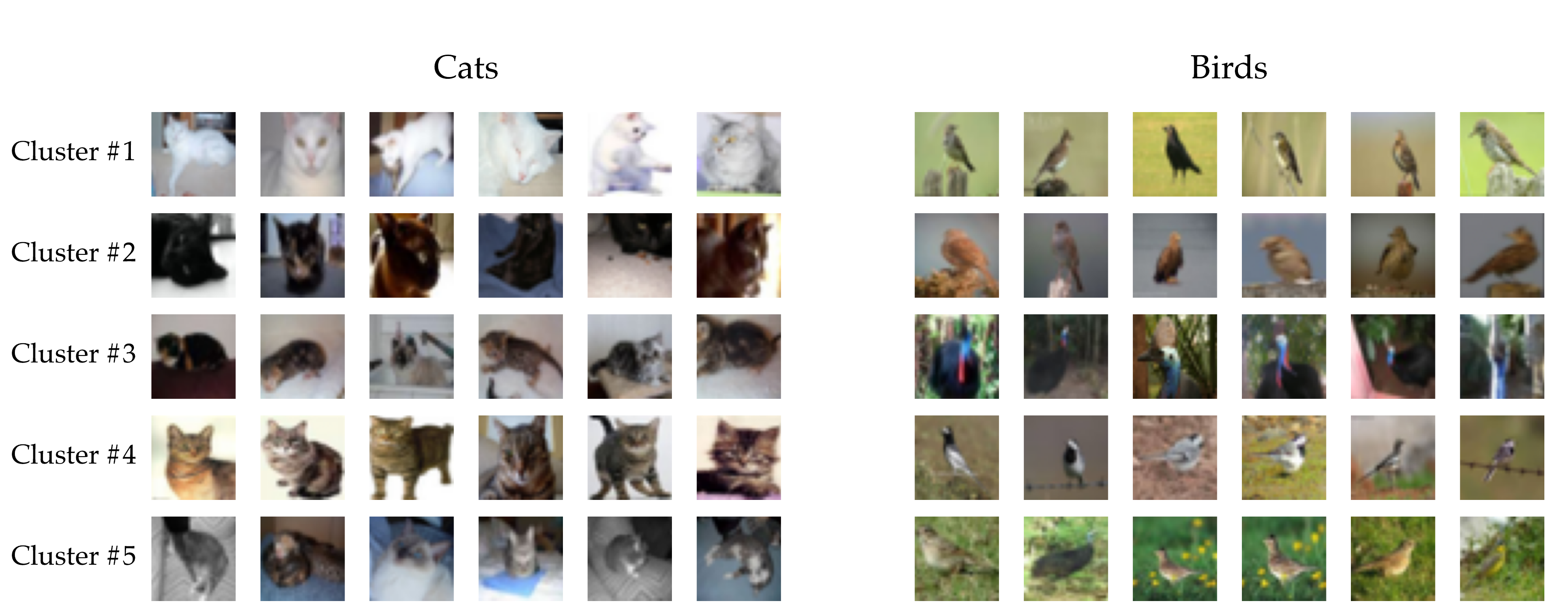}
    \includegraphics[width=0.9\linewidth,trim={0 0 0 1cm}, clip]{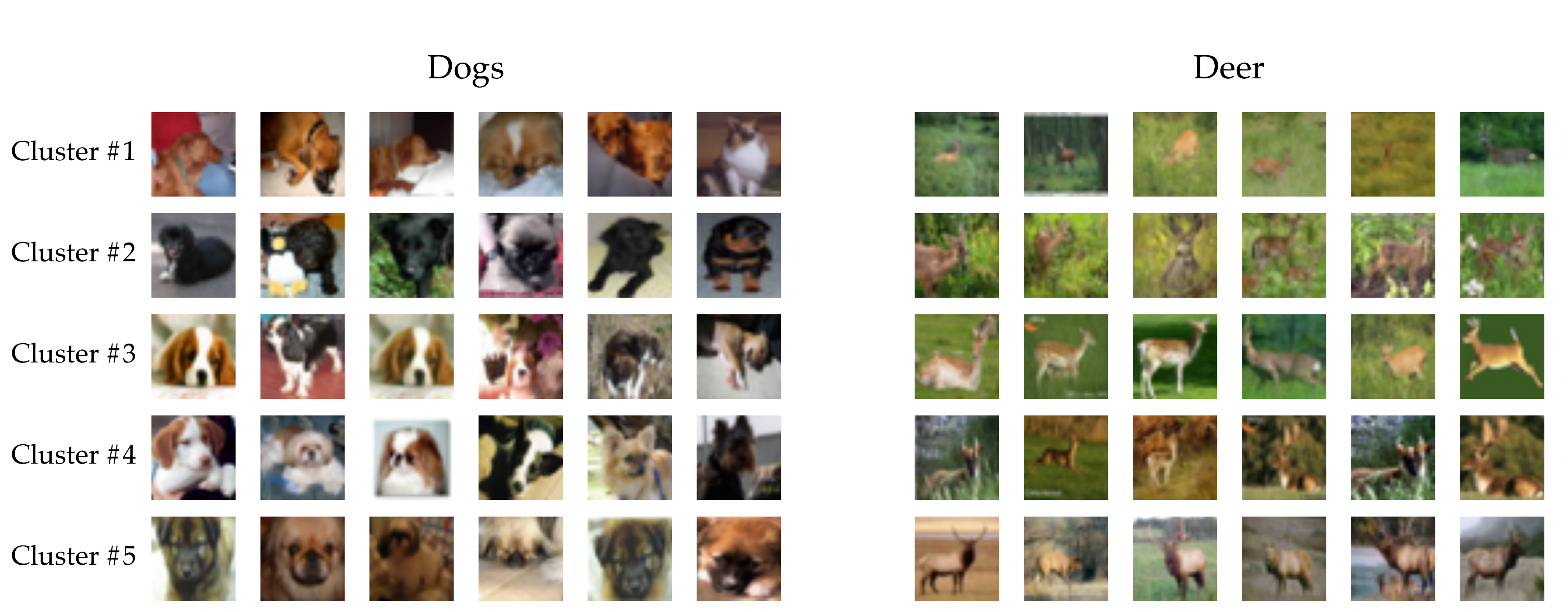}
    \includegraphics[width=0.9\linewidth,trim={0 0 0 1cm}, clip]{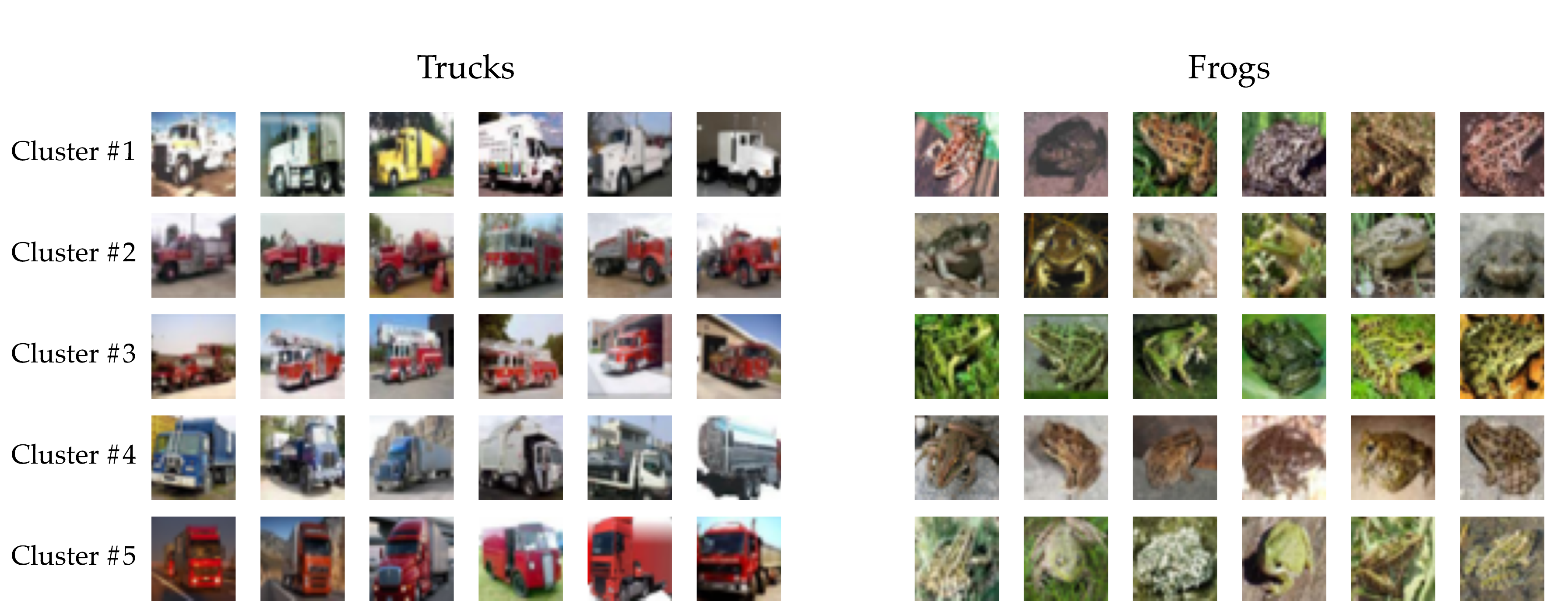}
    \caption{Omitted spectral clustering results for classes other than those in
    the main paper (Figure \ref{fig:clustering}).}
    \label{fig:clustering_extra}
\end{figure}

%% file: appendices/pca_app.tex
\section{PCA on Datamodel Embeddings}
\label{app:pca}

\subsection{CIFAR}
\paragraph{Setup.}
For the PCA experiments, we use datamodels for the training and test sets estimated with $\alpha=0.5$ unless mentioned otherwise.

\paragraph{Effective dimensionality.}
In \Cref{fig:explained_variance}, we compare the effective dimensionality of datamodel embeddings with that of a deep representation pretrained on CIFAR-10.
\begin{figure}[!hp]
    \centering
    \includegraphics[width=.5\linewidth]{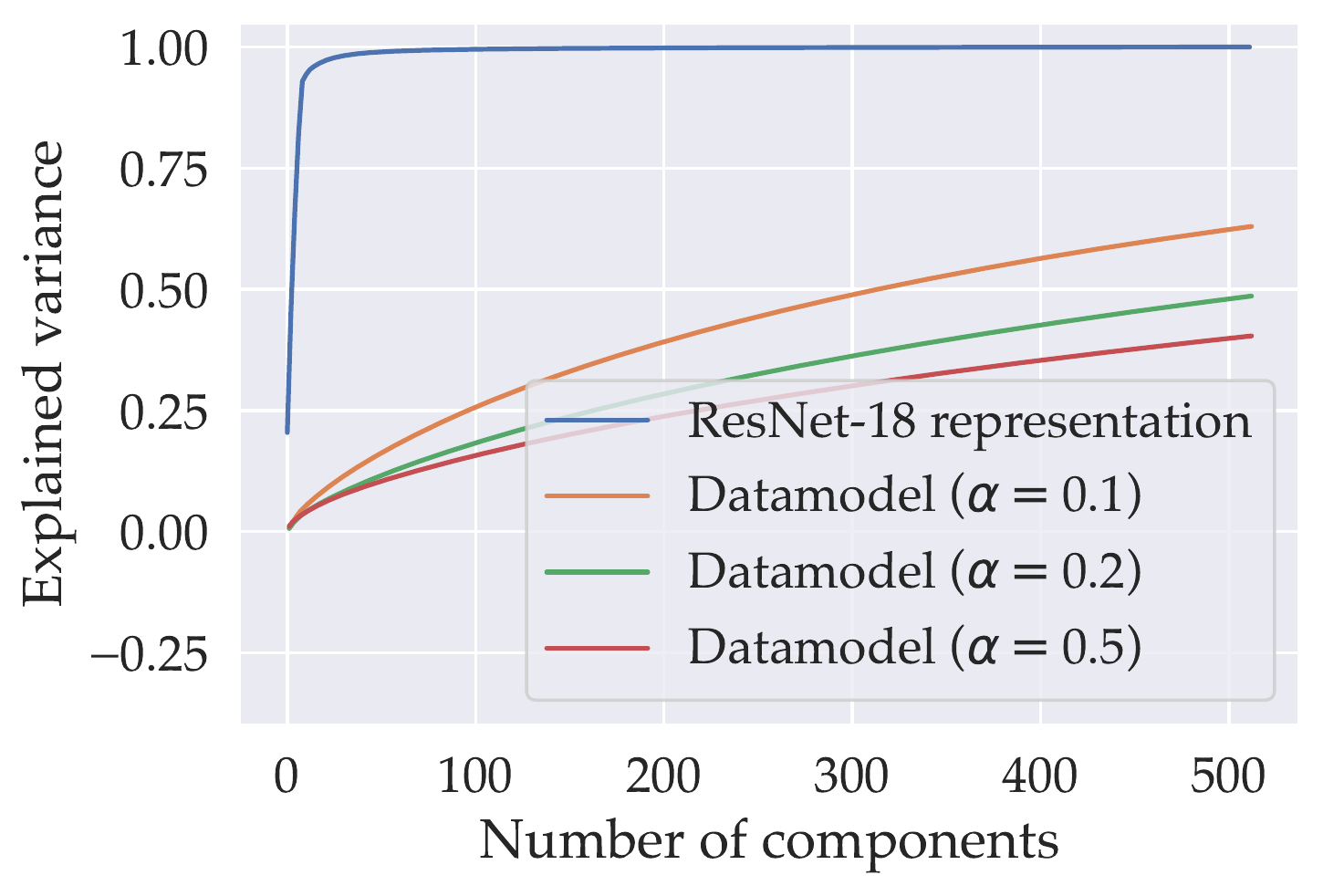}
    \caption{{\bf Datamodel embeddings have a higher effective dimension than
    deep representations}. For different embeddings, we plot the cumulative
    fraction of variance explained by the top $k$ components while varying $k$.
    For a network layer based embedding, 95\% of the variation in embedding
    space is captured by the first 10 principal components; meanwhile, datamodel
    embeddings need up to 500 components to capture even half of the variance.
    Here, we use a ResNet-18 model instead of ResNet-9 as it has more features in the representation layer (512 vs. 128); the plot looks similar for ResNet-9.
    }
    \label{fig:explained_variance}
\end{figure}

\paragraph{Analyzing model-faithfulness.} To see whether PCA directions reflect model behavior, we look at how ``removing'' different principal components affect model predictions.
More precisely, we remove training examples corresponding to:
\begin{itemize}
    \item Top $k$ most positive coordinates of the principal component vector
    \item Top $k$ most negative coordinates of the principal component vector
\end{itemize}
Then, for each principal component direction considered, we measure their impact on three groups of held-out samples:
\begin{itemize}
    \item The top 100 examples by most positive projection on the principal component
    \item The bottom 100 examples by most negative projection on the principal component
    \item The full test set
\end{itemize}
For each of these groups, we measure the mean change in margin after removing different principal component directions.
Our results (Figure~\ref{appfig:remove_direction_causality}) show that:
\begin{itemize}
    \item Removing the most positive coordinates of the PC decreases margin
    on the test set examples with the most positive projections on the PC and
    increases margin on the examples with the most negative projections on the PC.
    \item Removing the most negative coordinates of the PC
    has the opposite effect, increasing margin on the positive projection examples
    and decreasing margin on the negative projection examples.
    \item Increasing the size of each removed set increases the effect magnitude.
    \item Removing PC's have negligible impact on the aggregate test set, indicating that the impact of different PC's are roughly ``orthogonal,'' as one would expect based on the orthogonality of the PCs.
    \item Lastly, \Cref{fig:pca_causality_scatter} shows that datamodels can accurately predict the counterfactual effect of the above removed groups, similarly as in \Cref{fig:cifar_causal}.
\end{itemize}

\begin{figure}
    \centering
    \includegraphics[width=\linewidth,trim={0 0 6cm 0},clip]{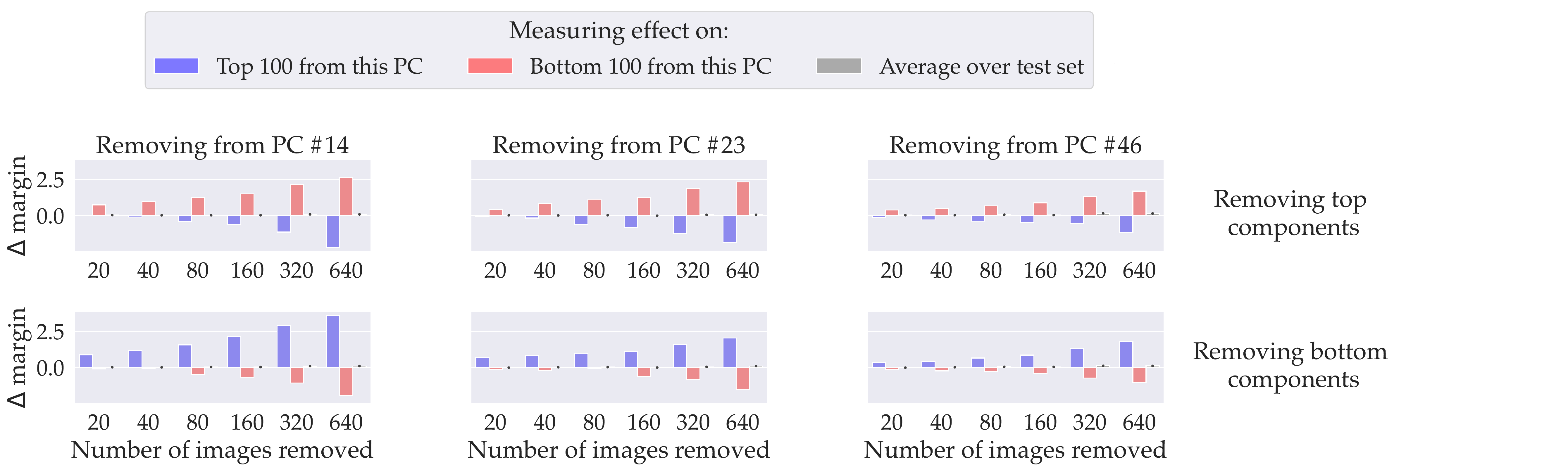}
    \caption{
    {\bf PCA directions generalize and capture ``orthogonal'' directions.}
    For each of the three principal components (PCs) above (randomly chosen from top 50), we consider the counterfactual of removing the training examples corresponding to the top or bottom $k$ coordinates in the PC, and measure its average effect on different groups: (red) test examples with the highest projections on the PC; (blue) test examples with the lowest projections on the PC; and (grey) the entire test set.
    The direction of the effect is consistent with the datamodel embeddings; removing top (resp. bottom) coordinates decrease (resp. increase) the average margin on test examples whose embeddings are most aligned with the PC. Moreover, the negligible impact over the test set in aggregate shows that the different PCs, which are orthogonal in the embedding space (by definition), are also approximately ``orthogonal'' in terms of their effect on model predictions.
    }
    \label{appfig:remove_direction_causality}
\end{figure}

\begin{figure}[!htp]
    \centering
    \includegraphics[width=0.45\linewidth]{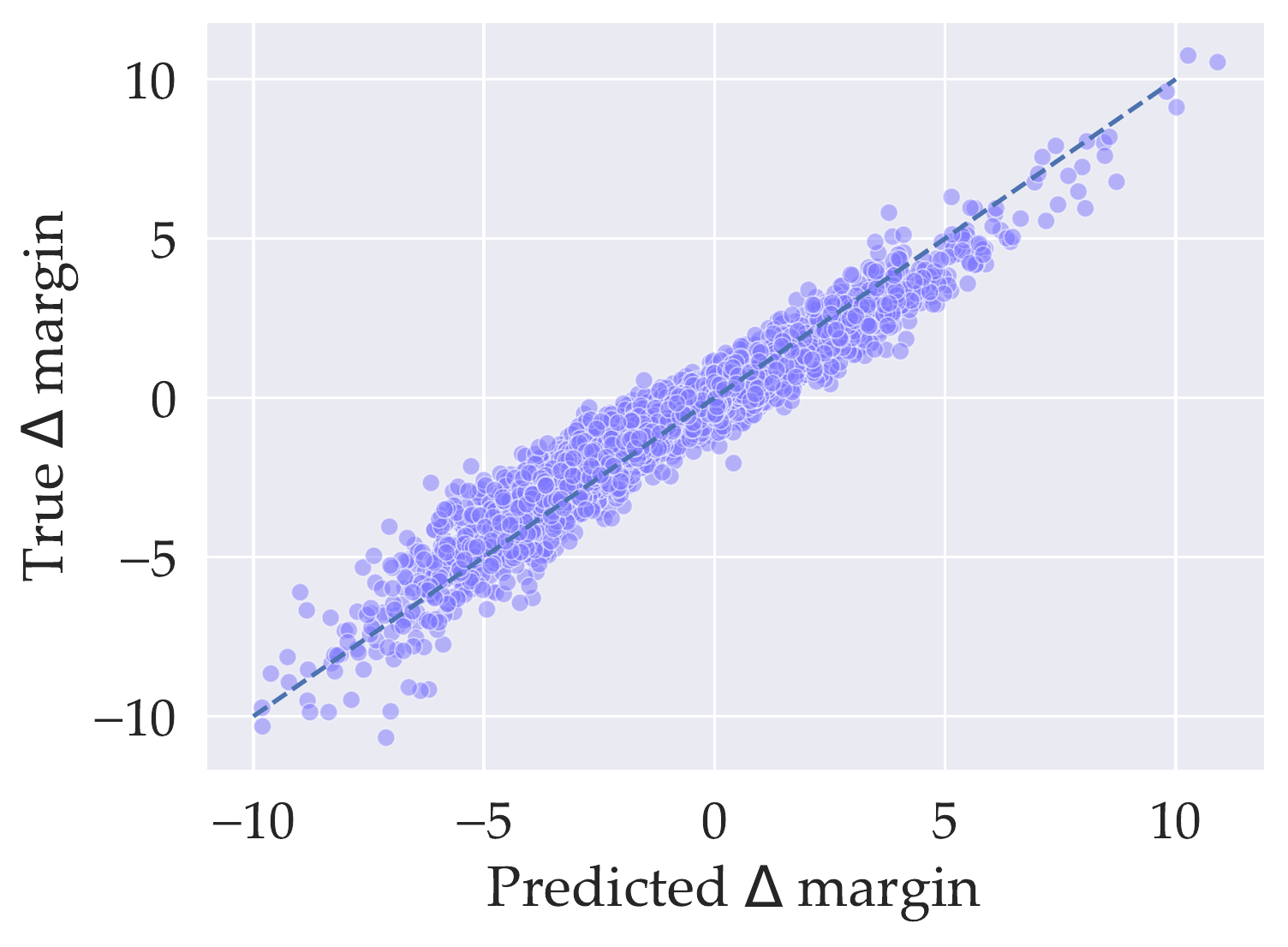}
    \caption{{\bf Datamodels predict the effect of ``removing'' principal components}.
    Each point corresponds to a
    \emph{PCA counterfactual}: removing training examples with the largest weights in the principal component (i.e., top-$k$ most positive or negative coordinates), and evaluating on test examples whose embeddings most align with the PC (e.g. smallest cosine distance).
    The $y$-coordinate of each point represents the {\em ground-truth}
    counterfactual effect (evaluated by retraining $T = 20$ times).
    The $x$-coordinate of each point represents the {\em
    datamodel-predicted} value of this quantity.
    }
    \label{fig:pca_causality_scatter}
\end{figure}

\paragraph{Representation baseline.} In \Cref{fig:pca_rep} we show the top principal components computed using a representation embedding.

\paragraph{Additional results.} In \Cref{fig:pca_extra} we show additional PCs from a datamodel PCA.

\begin{figure}[ht]
    \centering
    \includegraphics[width=\linewidth,trim={0 0 0 1cm}, clip]{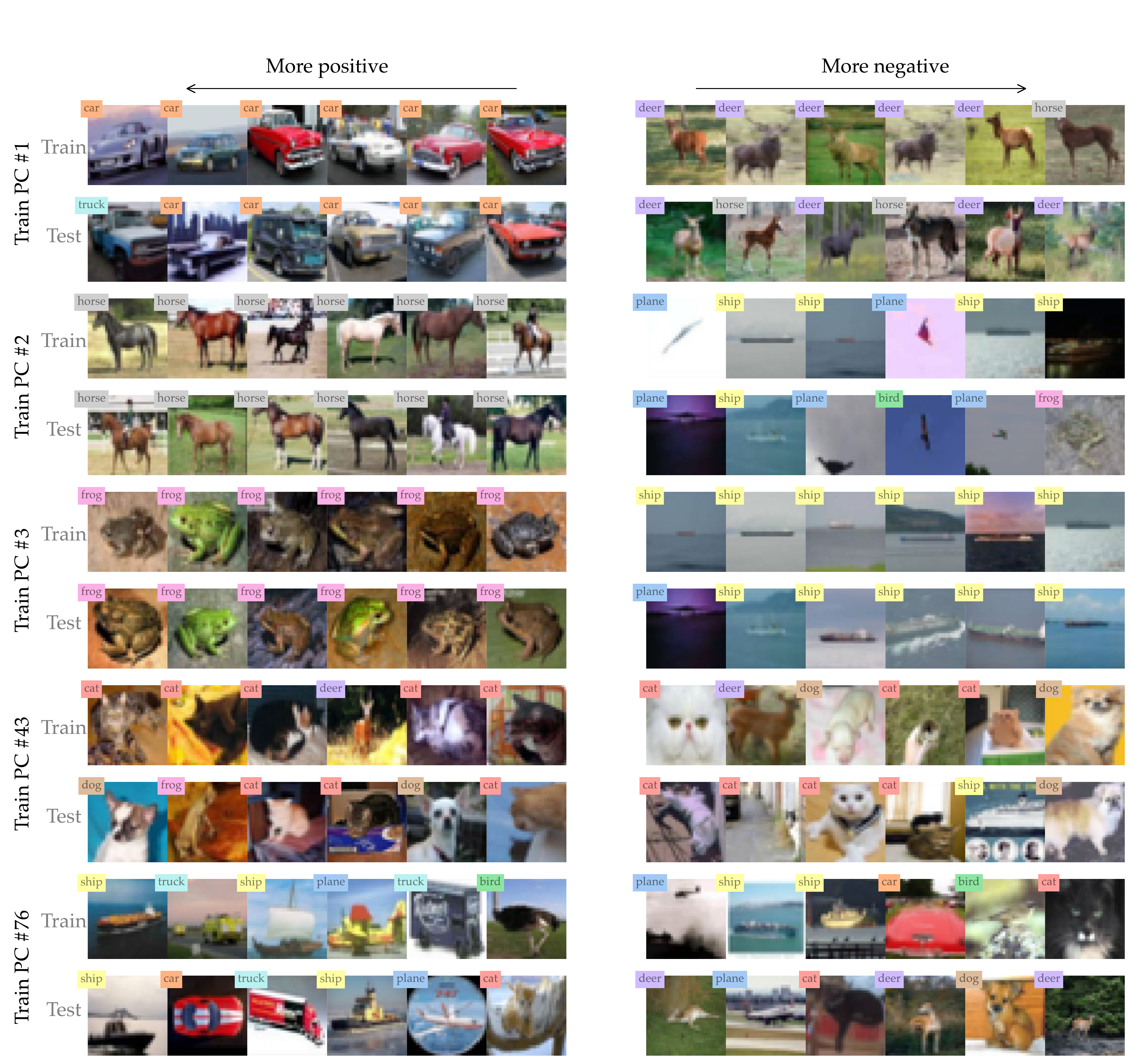}
    \caption{{\bf Representation-based baseline for PCA.} Visualization of highest magnitude images along top principal
    components of {\em representation} embeddings for CIFAR-10. In each row $i$, on the left
    we show the images with the highest normalized projections onto $v_i$, and
    on the right the images with the lowest projections. The PCs seem less coherent than those obtained from running PCA on datamodel embeddings.}
    \label{fig:pca_rep}
\end{figure}

\begin{figure}[ht]
    \centering
    \includegraphics[width=1.\linewidth]{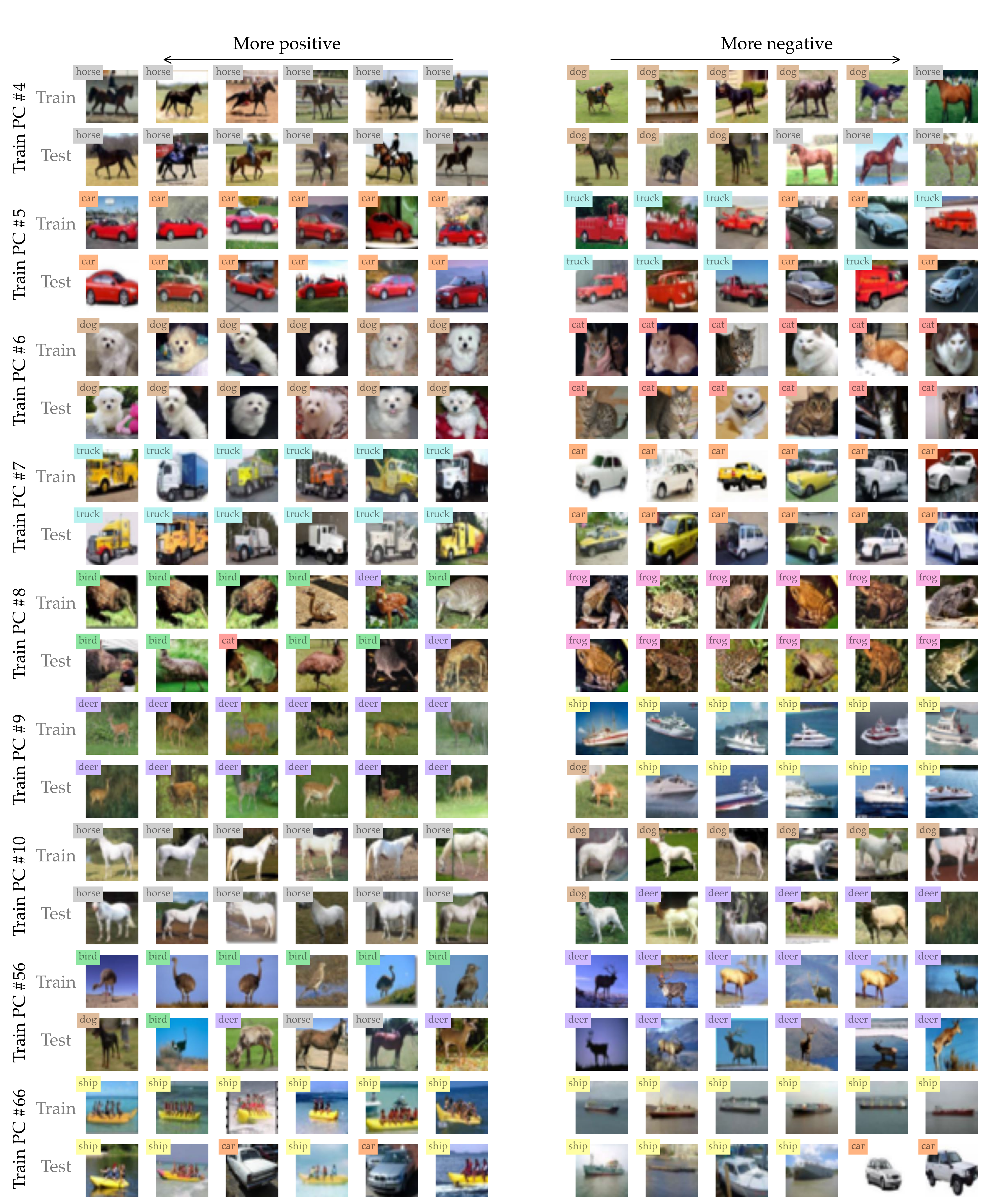}
    \caption{The remainder of the top 10 PCA directions and two selected directions.}
    \label{fig:pca_extra}
\end{figure}

\clearpage
\subsection{\fmow{}}
\begin{figure}[!ht]
    \centering
    \includegraphics[width=1.\linewidth]{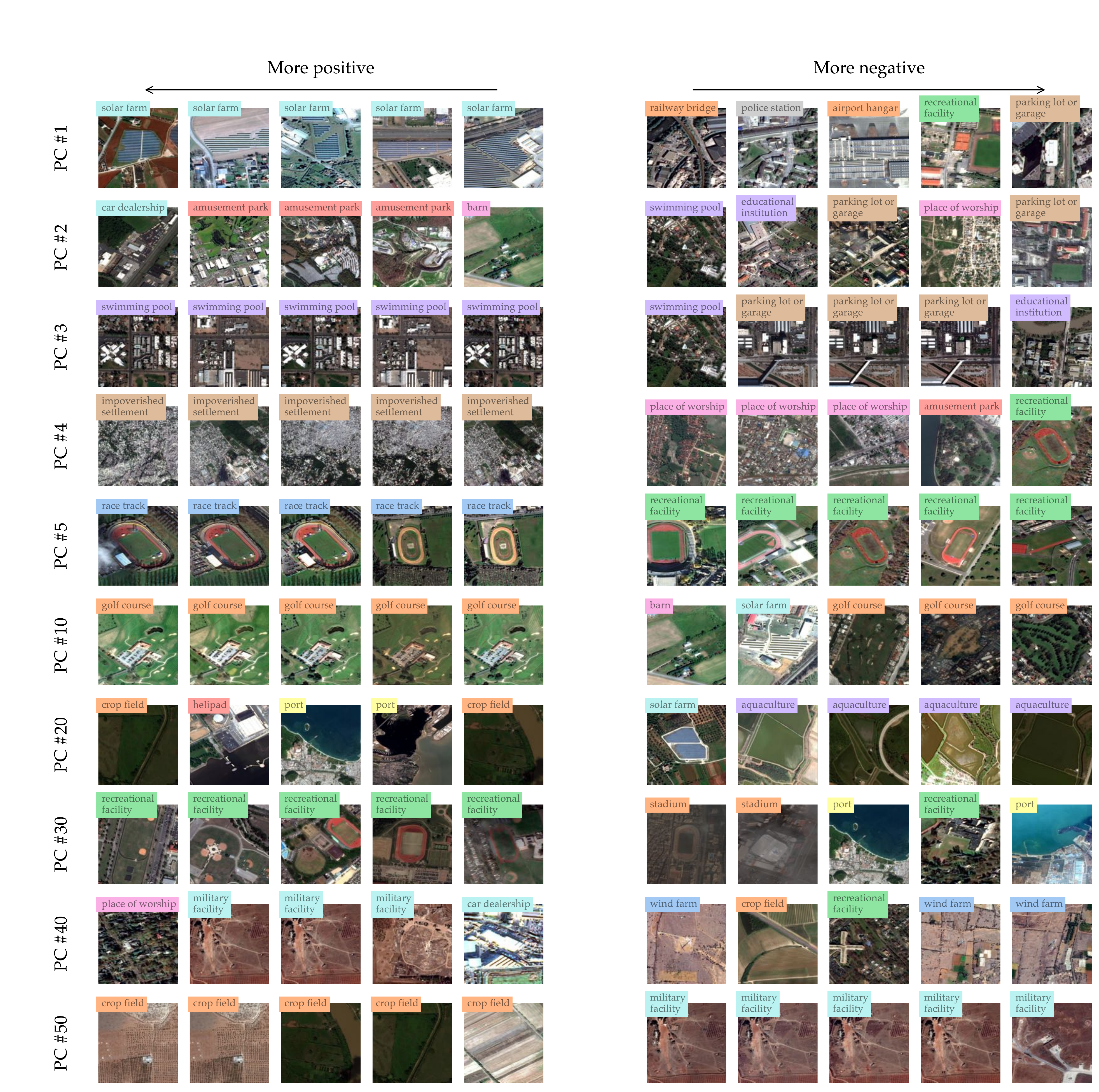}
    \caption{\fmow{} top PCA components, using $\alpha=20\%$. Top 5 and 5 selected from the top 50.}
    \label{fig:fmow_pca}
\end{figure}

%% file: appendices/influence_proof.tex
\section{Connection between Influence Estimation and Datamodels}
\label{app:ate_compare}
\subsection{Proof of Lemma \ref{lem:atelinear}}
\label{app:ate_proof}
\atelinear*
\begin{proof}
    For convenience, we introduce the $m \times n$ binary {\em mask matrix}
    $\bm{A}$ such that $\bm{A}_{ij}$ is an indicator for whether the $j$-th
    training image was included in $S_i$. Note that $\bm{A}$ is a random matrix
    with fixed row sum of $n/2$.
    Next, we define the {\em output vector} $\bm{y} \in \{0,1\}^m$ that
    indicates whether a model trained on $S_i$ was correct on $x$.
    Finally, we introduce the {\em count matrix}
    $\bm{C} = \text{diag}(\bm{1}^\top A)$, i.e., a diagonal matrix whose entries
    are the columns sums of $\bm{A}$, e.g. the number of times each example
    appears across $m$ different masks.

    We begin with $\bm{w}_{OLS}$. Consider the $n \times n$
    matrix $\bm{\Sigma} = \frac{1}{m}\bm{Z}^\top \bm{Z} = \frac{1}{m}(2\cdot \bm{A} - \bm{1}_{m\times n})^\top
    (2\cdot \bm{A} - \bm{1}_{m\times n})$. The diagonal entries of this matrix
    are $\bm{\Sigma}_{ii} = 1$ (due to $\bm{A}$ having constant row sum), while the off-diagonal is

    \[
        \bm{\Sigma}_{ab} = \frac{1}{m}\sum_{i=1}^m \begin{cases}
            +1 &\text{ if training image $x_a, x_b \in S_i$ or $x_a, x_b \not\in S_i$} \\
            -1 &\text{ otherwise.}
        \end{cases}
    \]
    Since $\bm{\Sigma}$ has bounded entries ($|\Sigma_{ab}| \leq 1$), we have
    that for fixed $n$, $\lim_{m \to \infty} \bm{\Sigma} =
    \mathbb{E}[\bm{\Sigma}]$, and in particular
    \begin{align*}
        \bm{\Sigma}_{ab} &\to \mathbb{P}(x_{a},x_{b} \in S_i \text{ or } x_{a},x_{b} \not\in S_i)
                              - (1 - \mathbb{P}(x_{a},x_{b} \in S_i \text{ or } x_{a},x_{b} \not\in S_i)) \\
        \mathbb{P}(x_{a},x_{b} \in S_i \text{ or } x_{a},x_{b} \not\in S_i) &=
                        2 \cdot \lr{\frac{\frac{n}{2}}{n}\cdot \frac{\frac{n}{2} - 1}{n} }
                        = \frac{1}{2} - \frac{1}{n} \\
        \text{Thus, }\bm{\Sigma}_{ab} &\to -\frac{1}{2n}.
    \end{align*}
    Now, using the Sherman-Morrison formula,
    \begin{align*}
        \bm{\Sigma}^{-1} = \frac{n}{n+2}\lr{\bm{I} + \frac{2}{n} \bm{1}_{n \times n}}
    \end{align*}
    By construction, the row sums of $\bm{Z} = 2\cdot \bm{A} - \bm{1}_{m \times n}$ are
    $0$, and so $\bm{1}_{n \times n} \cdot \bm{Z}^\top = 0$. Thus,
    \[
        \bm{w}_{OLS} = (\bm{Z}^\top \bm{Z})^{-1} \bm{Z}^\top \bm{y}
        = \frac{1}{m} \lr{\frac{1}{m}\bm{Z}^\top \bm{Z}}^{-1} \bm{Z}^\top \bm{y}
        = \frac{1}{m} \cdot \frac{n}{n+2} \bm{Z}^\top \bm{y}.
    \]

    We now shift our attention to the empirical influence estimator
    $\bm{w}_{infl}$. Using our notation, we can rewrite the
    (vectorized) empirical influence estimator
    \eqref{eq:sample_efficient_influences} as:
    \begin{align*}
        \bm{w}_{infl}
            &= \bm{C}^{-1} \bm{A}^\top \bm{y} - (m\cdot \bm{I}_n - \bm{C})^{-1} \lr{\bm{1}_{m\times n} - \bm{A}}^\top \bm{y} \\
            &= \lr{\bm{C}^{-1} - (m\cdot \bm{I}_n - \bm{C})^{-1}} \bm{A}^\top \bm{y} - (m\cdot \bm{I}_n - \bm{C})^{-1} \bm{1}_{m\times n}^\top \bm{y} \\
            &= m\cdot \bm{C}^{-1} \lr{m\cdot \bm{I}_n - \bm{C}}^{-1} \bm{A}^\top \bm{y} - (m\cdot \bm{I}_n - \bm{C})^{-1}\bm{1}_{m\times n}^\top \bm{y} \\
            &= \lr{m\cdot \bm{I}_n - \bm{C}}^{-1} \lr{
                m\cdot \bm{C}^{-1}\bm{A}^\top - \bm{1}_{m\times n}^\top
            } \bm{y}.
    \end{align*}
    Now, as $m \to \infty$ for fixed $n$, the random variable $m\bm{C}^{-1}$
    converges to $2 \cdot \bm{I}$ with probability $1$. Thus,
    \begin{align*}
        m\cdot \bm{A} \bm{C}^{-1} - \bm{1}_{m \times n}
        \to 2\cdot \bm{A} - \bm{1}_{m\times n},
    \end{align*}
    and the empirical influence estimator $\bm{w}_{infl} \to \frac{2}{m}
    \bm{Z}^\top \bm{y}$, which completes the proof.
\end{proof}

\subsection{Evaluating influence estimates as datamodels}
Lemma \ref{lem:atelinear} suggests that we can re-cast empirical influence
estimates as (rescaled) datamodels fit with least-squares loss. Under this view,
(i.e., ignoring the difference in conceptual goal), we can differentiate between explicit datamodels and those arising from empirical
influences along three axes:
\begin{itemize}
    \item {\bf Estimation algorithm}: Most importantly, datamodels
    {\em explicitly} minimize the squared error between true and predicted model
    outputs. Furthermore, datamodels as instantiated here use (a) a sparsity
    prior and (b) a bias term which may help generalization.
    \item {\bf Scale}: Driven by their intended applications (where one typically
    only needs to estimate the highest-influence training points for a given
    test point), empirical
    influence estimates are typically computed with relatively few samples
    (i.e., $m < d$, in our setting) \citep{feldman2020neural}.
    In contrast, we find that for datamodel loss to plateau, one needs to
    estimate parameters using a much larger set of models.
    \item {\bf Output type}: Finally, datamodels do not restrict to prediction
    of a binary correctness variable---in this paper, for example, for deep
    classification models we find that {\em correct-class margin} was best both
    heuristically and in practice.
\end{itemize}

In this section, we thus ask: how well do the rescaled datamodels that arise
from empirical influence estimates predict model outputs? We address this question in the context of the three axes of variation described
above. In order to make results comparable across different outputs types (e.g.,
correctness vs. correct-class margin), we measure correlation (in the sense
of \citet{spearman1904proof}) between the predicted and true model outputs, in addition
to MSE where appropriate.
To ensure a conservative comparison, we also measure
performance as a predictor of {\em correctness}. In particular, we treat
$\bm{w}^\top \mask{S_i}$ as a continuous predictor of the binary variable
$\bm{1}\{\text{model trained on $S_i$ is correct on $x$}\}$, and compute the AUC
of this predictor (intuitively, this should favor empirical influence estimates
since they are computed using correctnesses directly).

In Table \ref{tab:ate_compare} we show the difference between empirical
influence estimates (first row) and our final datamodel estimates (last row),
while disentangling the effect of the three axes above using the rows in
between. As expected, there is a vast difference in terms of correlation between
the original empirical influence estimates and explicit datamodels.
\begin{table}
    \begin{tabular}[h]{@{}ccccccc@{}}
        \toprule
        Algorithm & \# models ($m$) & Output type & Spearman $r$ & MSE & AUC & Difference \\
       \midrule
       Diff. of means & 25,000 & Correctness & 0.028 & N/A & 0.529 & \\
       Diff. of means & 100,000 & Correctness & 0.053 & N/A & 0.555 &
        Under $\to$ Over-determined \\
       Diff. of means & 100,000 & Margin & 0.213 & 2.052 & 0.653 &
       Output type
        \\
       LASSO & 100,000 & Margin & 0.320 & 1.382 & 0.724 &
       Explicit datamodel  \\
       \bottomrule
    \end{tabular}
    \caption{{\bf Disentangling the effect of different factors in datamodel
    performance.} Each row shows a different estimator for datamodels. We begin
    with the empirical influence (or difference of means) on correctness
    computed with 25,000 models, which is in the overparameterized regime (as
    there are $d=50,000$ variables). Then, we increase the number of models to
    an underparameterized regime. Next, we change the output type from
    correctness to margins. Lastly, we change the estimation algorithm from
    difference of means (which is approximately equivalent to OLS, as shown
    in~\Cref{app:ate_proof}) to LASSO. Each of these changes brings about
    significant gains in the signal captured by datamodels, as measured by
    Spearman rank correlation, MSE, or AUC.}
    \label{tab:ate_compare}
\end{table}
We further illustrate this point in Figure \ref{fig:ate_compare_line_graphs},
where we show how the correlation, MSE, and AUC vary with $m$ for both empirical
influence estimates and datamodels, as well as an intermediate estimator that
uses the estimation procedure of empirical influence estimates but replaces
correctness with margin.

\begin{figure}[h]
    \centering
    \includegraphics[width=0.8\textwidth]{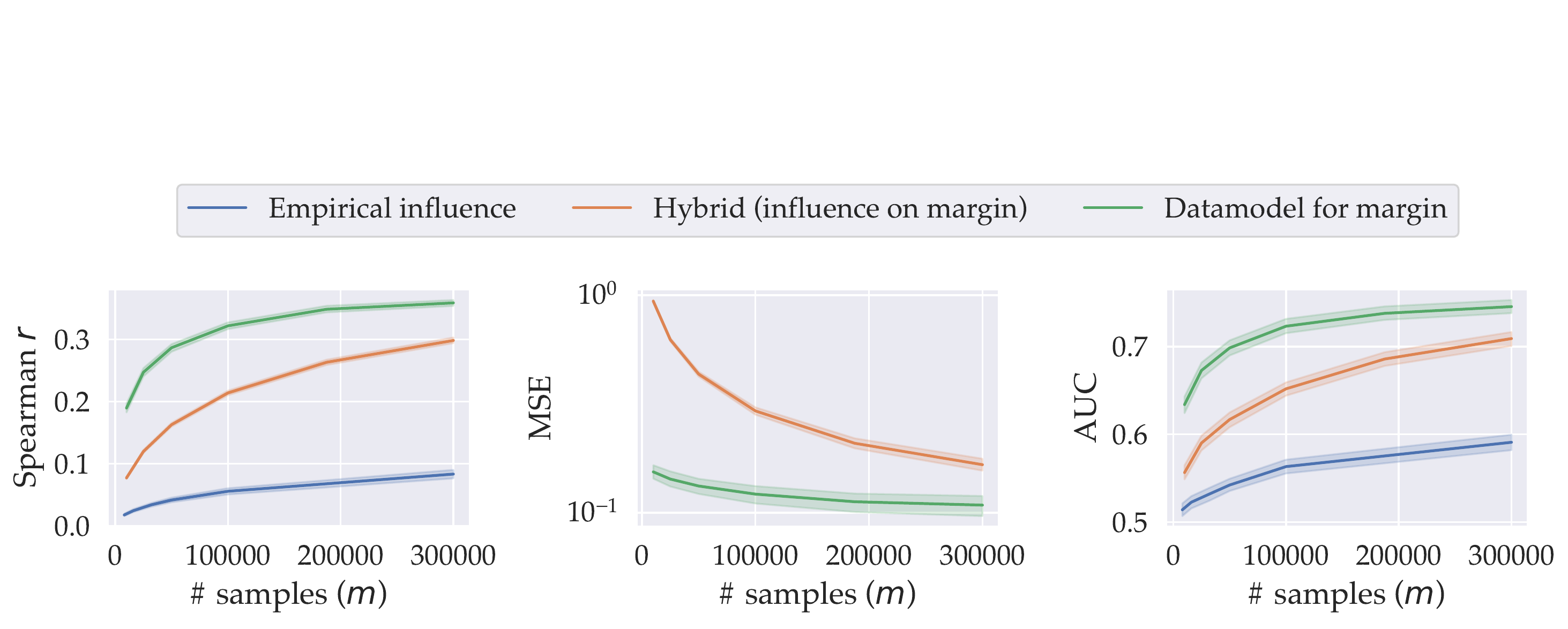}
    \caption{{\bf Datamodels have significantly better sample complexity than empirical influences.}
    We compare three estimators---empirical influence, empirical influence on margins, and $\ell_1$-regularized linear regression on margins (datamodels)---across a wide range of sample sizes on three different metrics. All metrics are averaged over the entire test set (i.e. over 10,000 datamodels).
    For MSE, we only show the estimators on margins as different output types are incomparable. Across all metrics, datamodels capture significantly more signal than empirical influences using the same number of samples. Conversely, datamodels need far fewer samples to reach the same level of performance.
    }
    \label{fig:ate_compare_line_graphs}
\end{figure}

\subsection{Testing Lemma \ref{lem:atelinear} empirically}
In this section, we visualize the performance of empirical influence estimates (\citep{feldman2020neural}) as datamodels. In Figures
\ref{fig:ate_residuals_tr} and \ref{fig:ate_residuals_val} we plot the distributions of
$\bm{w}_{infl}^\top \mask{S_i}|y_i$ for different CIFAR-10 test examples; Figure
\ref{fig:ate_residuals_tr} shows these ``conditional prediction distributions''
for subsets $S_i$ that were used to estimate the empirical influence, while
Figure \ref{fig:ate_residuals_val} shows the corresponding distributions on
held-out (unseen) subsets $S_i$.
The figures suggest that (i) indeed, empirical influences are
somewhat predictive of the correctness $y_i$, (ii) their predictiveness increases as number of samples $m \to \infty$ but is still rather low, and (c) a significant amount of the prediction error is generalization error, as the train
predictions in Figure \ref{fig:ate_residuals_tr} are significantly
better-separated than the heldout predictions in Figure
\ref{fig:ate_residuals_val}.
\begin{figure}[h]
    \centering
    \begin{subfigure}{\textwidth}
        \includegraphics[width=\textwidth,trim={0 0 2cm 1cm},clip]{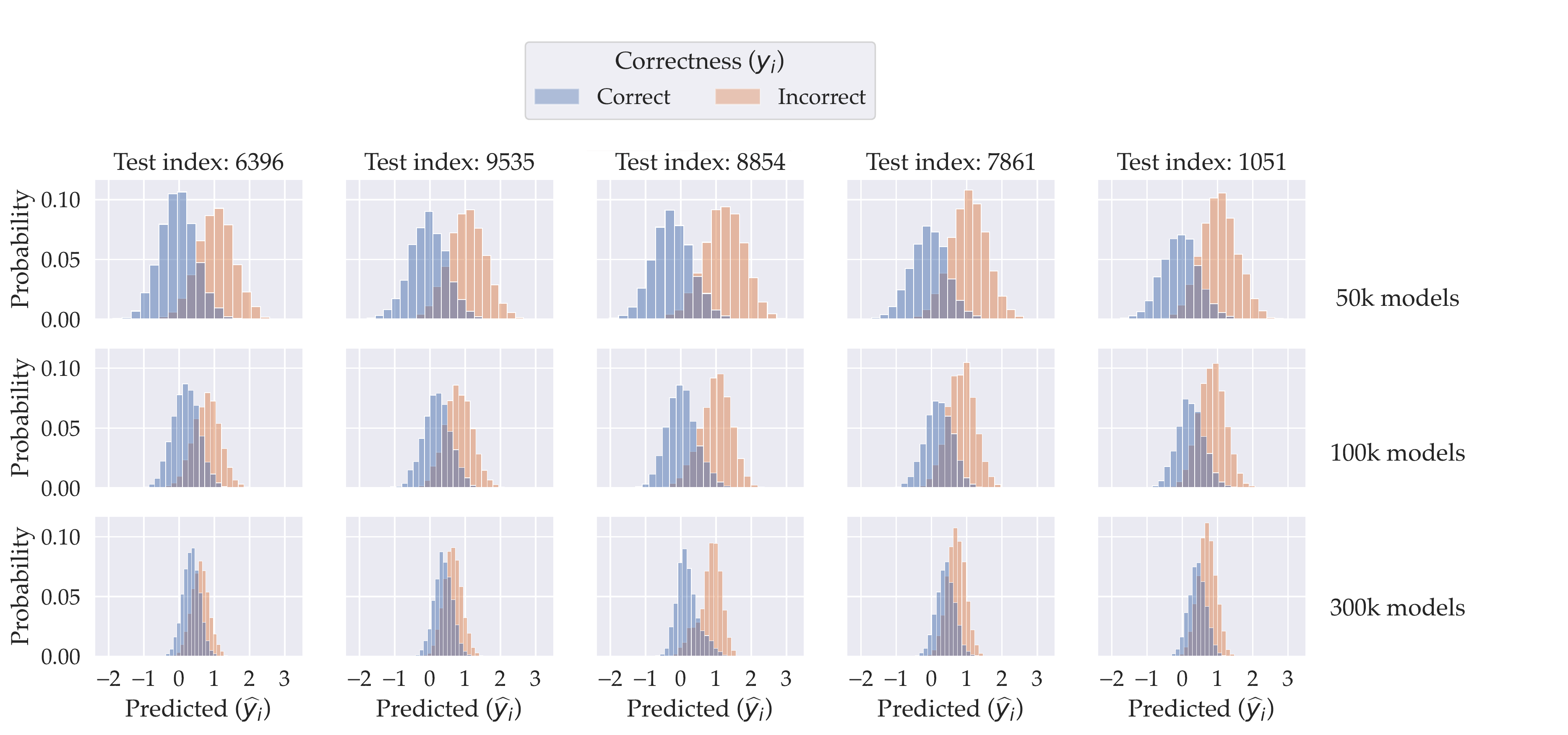}
        \caption{Train masks}
        \label{fig:ate_residuals_tr}
    \end{subfigure}
    \begin{subfigure}{\textwidth}
        \includegraphics[width=\textwidth,trim={0 0 2cm 3cm},clip]{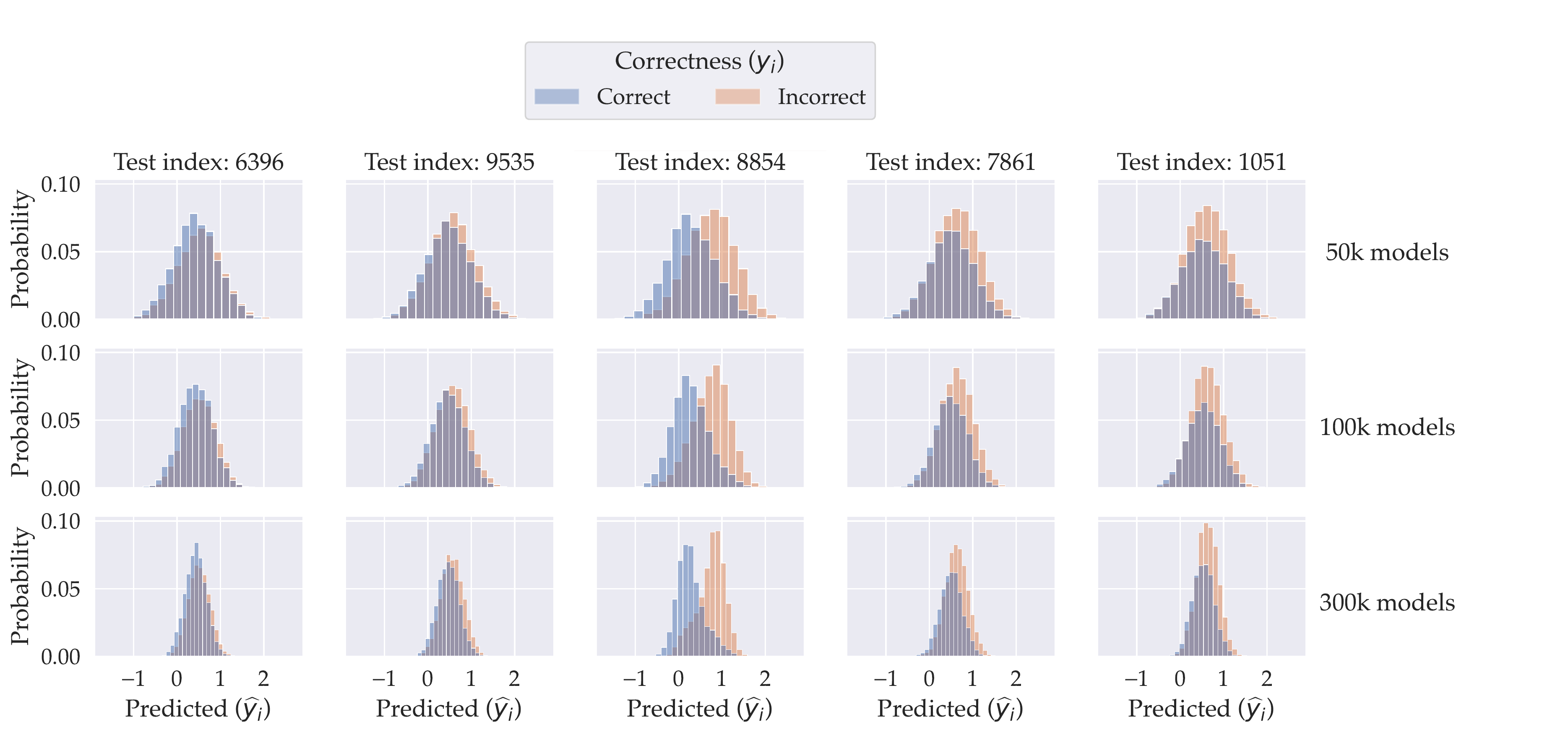}
        \caption{Heldout masks}
        \label{fig:ate_residuals_val}
    \end{subfigure}
    \caption{{\bf Empirical influence estimates are (weak) datamodels.}
    Each histogram illustrates the performance of empirical influences when the output of the corresponding datamodel is used as a statistic to distinguish between correct and incorrect predictions on the target example. Empirical influences can predict correctness on the ``train set'' of subsets (i.e. the masks used to estimate them), but suffer from significant generalization error when evaluated on a held-out set of subsets.}
\end{figure}

%% file: appendices/taylor.tex
\subsection{View of empirical influences as a Taylor approximation}
\Cref{lem:atelinear} shows that we can interpret empirical influences as (rescaled) estimates of the weights of a \emph{linear} datamodel. Here, we give an alternative intuition for why this is the case, even though the definition of empirical influence does not explicitly assume linearity anywhere: we show that the influences define a first-order
Taylor approximation of the multilinear extension $f$ of our target function $F$ of interest, where the influences (approximately) correspond to first-order derivatives of $f$.

Recall that we want to learn some output of interest $F: 2^T \rightarrow \mathbb{R}$, say the probability of correctness on a test example $z$, as a function of the examples $S \subset T$ included in the training set.
We first extend this function continuously so that we can take its derivatives.
The multilinear extension \citep{owen1972multilinear} of set function $F$ to the domain $[0,1]^n$ ($|T|=n$) is given by:
\begin{align}
    f(x) &= \sum\limits_{S \subseteq T} F(S) \prod\limits_{i \in S} x_i \prod\limits_{i \not\in S} (1-x_i)
\end{align}
$f(x)$ also has an intuitive interpretation: it is the expected value of $F(S)$ when $S$ is chosen by including each $x_i$ in the input with probability $x_i$.

Next, we take the derivative of $f$ w.r.t. to the input $x_i$:
\[
    \frac{\partial f}{\partial{x_i}} =
    \underbrace{\sum\limits_{S \subseteq T, i \in S} F(S) \prod\limits_{j \in S, j\ne i} x_j \prod\limits_{j \not\in S} (1 - x_j)}_{\mathop{\mathbb{E}}\limits_{s_j \sim Bern(x_j), s_i=1} F(S)}  - \underbrace{\sum\limits_{S \subseteq T, i \not\in S} F(S) \prod\limits_{j \in S} x_j \prod\limits_{j \not\in S, j \ne i} (1-x_j)}_{\mathop{\mathbb{E}}\limits_{s_j \sim Bern(x_j), s_i=0} F(S)}
\]

Note that because $f$ is multilinear, the derivative w.r.t. to $x_i$ is constant in $x_i$, but not w.r.t. to other $x_j$.
Now, observe that the above expression evaluated at $x_j=\alpha,\;\forall x_j$ corresponds approximately\footnote{There are two sources of approximation here. First, the $\alpha$-subsampling used in our datamodel definition is defined globally (e.g. $\alpha$ fraction of entire train set), which is different from the i.i.d. $Bern(\alpha)$ sampling that is considered here. Second, we only observe noisy versions of $F(S)$.} to $\alpha$-subsampled influence $\theta_i$, of $i$ on $F$:
the first term corresponds (using our earlier interpretation) to the expectation of $F(S)$ conditional on $S$ including $i$, and the second to that conditional on $S$ excluding $i$.

Finally, the first-order Taylor approximation of $f$ around an $x$ is given as:
\[
f(x) \approx F(\emptyset) + \sum\limits_{i} \frac{\partial f}{\partial{x_i}} \cdot x_i \approx F(\emptyset)  + \sum\limits_{i} \theta_i \cdot x_i
\]
where $\theta_i$ are the empirical influences.

\paragraph{The role of $\alpha$.}
The above perspective provides an alternative way to think the role of the sampling fraction $\alpha$.
The weights $\theta_i$ depend on the regime we are interested in; if we use $\alpha$-subsampled influences, then we are effectively taking a local linear approximation of $f$ in the regime around $\vec{x} = \alpha \cdot \vec{1}$.

\paragraph{Remark.}
Though we include the exposition above for completeness, this is a classical derivation that has appeared in similar form in prior works \citep{owen1972multilinear}.
Another connection is that {\em Shapley value} is equivalent to the integral of $f$ along the ``main diagonal'' of the hypercube; it is effectively empirical influences averaged uniformly over the choice of $\alpha$.